\documentclass[12pt]{article}

\usepackage{amsmath}
\usepackage{amsfonts,amssymb}
\usepackage{amsthm}
\usepackage{bm}
\usepackage{booktabs}
\usepackage{array}
\usepackage{graphicx}
\usepackage{float}
\usepackage{placeins}
\usepackage[table]{xcolor}
\usepackage{natbib}
\usepackage{url}
\usepackage[colorlinks=true,allcolors=blue]{hyperref}
\usepackage{microtype}

\newcommand{\blind}{1}

\newcommand{\R}{\mathbb{R}}
\newcommand{\Pp}{\mathbb{P}}
\newcommand{\E}{\mathbb{E}}

\newcommand{\KL}{\operatorname{KL}}
\newcommand{\AUC}{\operatorname{AUC}}
\newcommand{\tr}{\operatorname{tr}}

\newcommand{\norm}[1]{\left\lVert #1\right\rVert}
\newcommand{\ip}[2]{\left\langle #1,#2\right\rangle}

\newtheorem{theorem}{Theorem}
\newtheorem{lemma}{Lemma}
\newtheorem{proposition}{Proposition}

\theoremstyle{definition}

\newtheorem{assumption}{Assumption}

\begin{document}

\def\spacingset#1{\renewcommand{\baselinestretch}{#1}\small\normalsize}
\spacingset{1}

\if1\blind {
\title{\bf Token-Level Likelihood-Array Regression for Membership Inference and AI-Generated Text Detection}
  \author{Jiajun Sun$^1$ and Zhanrui Cai$^2$\\
  $^{1}$ Department of Statistics and Data Science, School of Economics,\\ Xiamen University\\
  $^{2}$ Faculty of Business and Economics, University of Hong Kong}
  
  \date{}
  \maketitle
} \fi

\if0\blind {
  \bigskip
  \bigskip
  \begin{center}
{\LARGE\bf Token-Level Likelihood-Array Regression for Membership Inference and AI-Generated Text Detection}
  \end{center}
  \medskip
} \fi

\bigskip

\begin{abstract}

Membership inference asks whether a text was used to train a language model, whereas AI-generated text detection asks whether it was generated by a language model rather than written by a human. Existing likelihood-based methods typically compress token-level probabilities into a few prespecified scores, most often using only probabilities conditioned on the full preceding context. We propose likelihood-array regression (LAR), which evaluates each target token under nested left-context windows and organizes the resulting likelihood-derived features into a structured array. After aligning arrays across texts of different lengths, LAR learns how detection information varies with context scale, token position, and likelihood features. LAR-1 aggregates learned contributions from individual aligned cells, while LAR-2 adds second-order features formed from pairs of evaluations of the same target token across context lengths. For within-path quadratic model, we establish matching minimax lower and upper bounds, characterize errors from finite-dimensional approximation and random squared projections, and derive conditions under which an oracle spectral sieve attains the minimax rate. Across multiple scoring language models, LAR substantially improves membership inference and AI-generated text detection over likelihood-based baselines. The analyses further show that shorter-context likelihoods contain information beyond conventional full-context probabilities, while second-order features provide additional gains for membership inference.

\end{abstract}


\newpage
\spacingset{1.8}

\section{Introduction}
\label{sec:introduction}

We consider two related problems in auditing language models. The first is membership inference, which asks whether a given text was included in the training data of a particular language model. This question is central to privacy and copyright audits. For example, a publisher may seek to determine whether its proprietary articles were used to train a model, as highlighted by The New York Times' ongoing lawsuit against OpenAI and Microsoft \citep{nyt2023complaint}. The second is AI-generated text detection, which asks whether a text was produced by a language model rather than written by a human. This question arises in authorship verification, academic publishing, and the identification of model-generated passages in mixed-content documents. Although the two problems concern different relationships between a text and a language model, both can be formulated as binary classification tasks based on conditional token probabilities. In membership inference, the label indicates whether the text was used to train the model; in AI-generated text detection, it indicates whether the text was written by a human or generated by AI.


Existing likelihood-based approaches to these two detection tasks generally reduce token-level probabilities to one or a few scalar summaries. For membership inference, commonly used summaries rely on calibrated losses, lower-tail token likelihoods, or adjustments for token rarity; for AI-generated text detection, they rely on token ranks, likelihood curvature, comparisons across models, or adaptive witness functions \citep{carlini2021extracting,shi2024detecting,mitchell2023detectgpt,hans2024binoculars,zhou2025adadetectgpt}. Most of these methods are built from the standard next-token log-probability sequence $\ell_{\mathrm{full}}(t)=\log p_\theta(x_t\mid x_1,\ldots,x_{t-1})$, 
where each token is evaluated conditional on the entire preceding sequence. More recent methods incorporate specific forms of context dependence by introducing external prefixes, computing perplexities over subsequences, or contrasting predictions with and without context \citep{xie2024recall,chang2025contextaware,sun2026clr}. Nevertheless, they ultimately aggregate the resulting likelihood information into a small number of prespecified scores. An alternative approach is to fine-tune a language model as a detector using a large labeled corpus \citep{zhou2026performanceguarantees}. Adapting such detectors to a new audited model or application domain generally requires additional labeled data and extra model fine-tuning. Existing approaches therefore leave open a broader statistical question: how should likelihood information that varies jointly with context length and token position be represented and modeled for these two detection tasks?

Addressing this question involves three major statistical challenges. First, evaluating the tokens in each document under multiple lengths of preceding context produces a likelihood representation whose dimensions vary with document length. 
Such a representation must place documents of different lengths on a common domain while preserving information about both token position and context length.
Second, token probabilities contain not only signals relevant to the detection label but also substantial nuisance variation arising from token rarity, baseline predictability, topic, writing style, document length, and model calibration. Third, the context lengths, document regions, and likelihood patterns that are informative for detection are not known in advance. Learning these relationships without imposing structural assumptions would require estimating a high-dimensional function from a relatively small labeled sample.

In this paper, we propose token-level likelihood-array regression (LAR), a regression framework for both applications. LAR represents the conditional token probabilities obtained from a fixed scoring language model as a structured, variable-length array and models the conditional log odds of the binary label through functionals of this array. To construct the array, the scoring model evaluates each target token under a sequence of nested left-context windows. Each resulting cell is indexed by the amount of preceding context and the token’s relative position in the document, and is augmented with several features derived from the model’s predictive distribution. We align these cells on a common domain so that documents of different lengths can be analyzed within a single model. LAR-1 learns how the location and likelihood-derived features of each cell contribute to the class log odds and averages these contributions across context lengths and target tokens. LAR-2 adds within-path second-order features formed from pairs of evaluations of the same target token; these include both same-context quadratic terms and products across different context lengths. Basis expansions and random squared projections provide finite-dimensional approximations of the two models, which are then estimated by ridge logistic regression.

We provide both empirical and theoretical support for the proposed framework. Most notably, across five scoring models in each application, LAR attains mean held-out AUCs of 0.914–0.983 for membership inference and 0.985–0.996 for AI-generated text detection, substantially outperforming existing likelihood-based procedures. Restricted-array analyses further show that likelihoods computed from shorter preceding contexts contain information beyond conventional full-context likelihoods, while the within-path second-order augmentation yields additional gains primarily for membership inference. On the theoretical side, for the proposed within-path quadratic model, we establish matching minimax lower and upper prediction bounds at the rate $\{{\log(en)/n}\}^{\nu/(\nu+1)}$ under product-form covariance decay. 
We further decompose the prediction error of finite-dimensional estimators into approximation, regularization, estimation, and squared-projection components, identify resolution conditions under which an oracle spectral sieve attains the minimax rate, and connect prediction error to AUC convergence.

The remainder of the paper is organized as follows. Section 2 reviews likelihood-based detection methods and regression with structured covariates. Section 3 constructs the likelihood-array representation and introduces the LAR models and estimators. Section 4 establishes the minimax prediction bounds, finite-resolution guarantees, and implications for AUC. Section 5 presents the numerical studies, and Section 6 discusses the scope and limitations of the proposed framework.

\section{Related Work}
\label{sec:related}

\subsection{Likelihood-based detection}

Existing membership-inference methods differ mainly in how they calibrate and aggregate token likelihoods. Proposed scores calibrate sequence loss using compression or a reference model, compare the target loss with those of nearby perturbations, focus on low-likelihood tokens, or standardize token likelihoods using the model's predictive vocabulary distribution \citep{carlini2021extracting,mattern2023membership,shi2024detecting,zhang2025minkplusplus}. Other procedures vary the conditioning context through external prefixes or collections of subsequences \citep{xie2024recall,chang2025contextaware}, or fine-tune the audited model using known nonmembers before recomputing likelihood-based scores \citep{zhang2025finetuning}. These methods indicate that membership evidence can depend on token difficulty, the predictive vocabulary distribution, and the context used to evaluate each token. As for likelihood-based AI-generated text detectors, existing methods construct scores from token probabilities and ranks, likelihood-to-rank ratios, likelihood curvature, comparisons between scoring models, regenerated text, or adaptive witness functions \citep{gehrmann2019gltr,su2023detectllm,mitchell2023detectgpt,bao2024fastdetectgpt,hans2024binoculars,yang2024dnagpt,verma2024ghostbuster,zhou2025adadetectgpt}. Despite their different constructions, these methods generally reduce the available token-level probability information to one or a few scalar scores.

Another class of methods learns detectors from labeled corpora of human and AI-generated text, using conventional classifiers, fine-tuned language models, or combinations of the two \citep{solaiman2019release,ippolito2020automatic,guo2023hc3,abburi2023generative}. More recently, \citet{zhou2026performanceguarantees} fine-tuned a language model on a large labeled corpus and calibrated its output to provide statistical error control. Those detectors can achieve strong classification performance, but adapting them to new generators or application domains may require substantial labeled data and further model fine-tuning. 

\subsection{Regression with structured covariates}

Scalar-on-function regression relates a scalar response to a functional covariate, with extensions to generalized response models \citep{ramsay2005functional,morris2015functional}. Basis expansions provide finite-dimensional representations of coefficient functions and facilitate regularized estimation. Quadratic functional regression additionally allows the response to depend on products of features at pairs of locations along a trajectory \citep{yaomuller2010quadratic}. LAR-1 and LAR-2 share this distinction between first- and second-order effects, but their covariate structure differs from that of standard functional regression. Rather than observing a single trajectory for each sampling unit, each text supplies a collection of target-token paths indexed by context length and located at different positions in the document. Moreover, the second-order products in LAR-2 are restricted to evaluations of the same target token.

Matrix and tensor regression preserve multidimensional structure when relating an array-valued covariate to a scalar response \citep{zhou2014regularized,luo2023low}. These methods commonly impose low-rank or sparse coefficient structures and generally assume that the observed arrays are of fixed dimension and a common indexing scheme. In contrast, a raw likelihood array has a triangular structure whose dimension and resolution vary with document length. Its cells must be aligned across texts while retaining their grouping into target-token paths. Thus, a likelihood array is neither a single functional trajectory nor a conventional fixed-dimensional array. This structure motivates both the proposed regression framework and the within-path quadratic model analyzed next.

\section{Likelihood-Array Regression}
\label{sec:method}

For $i=1,\ldots,n$, let $X_i=(x_{i1},\ldots,x_{iT_i})$ denote a tokenized text and let $Y_i\in\{0,1\}$ denote its class label. In membership inference, $Y_i=1$ indicates that $X_i$ belongs to the training data of the audited model. In AI-generated text detection, $Y_i=1$ indicates that $X_i$ is model-generated. The texts are the independent sampling units; arbitrary dependence among tokens and likelihood evaluations within a text is allowed. Let $p_\theta$ denote a fixed autoregressive language model from which conditional token probabilities are obtained. We refer to $p_\theta$ as the scoring model. In membership inference, the scoring model is typically the model being audited. In AI-generated text detection, it need not coincide with the model that generated the text. The scoring model remains fixed throughout the analysis: the labeled texts are used to fit the regression model but not to update $p_\theta$.

The construction proceeds from token-level likelihood evaluations to a document-level regression model. We first evaluate each target token under nested left-context windows, producing a variable-size likelihood array. We then align these arrays across documents and attach several likelihood-derived features to each aligned cell. Finally, LAR-1 models the contributions of individual cells to the class log odds, while LAR-2 augments these contributions with second-order features formed within each target-token path.

\subsection{Likelihood arrays and alignment}
\label{subsec:array}

For notational simplicity, we first consider a generic text $X=(x_1,\ldots,x_T)$ and suppress the document index $i$. Define the token-level log-likelihood
\[
\ell(s,t) = \log p_\theta(x_t \mid x_s,\ldots,x_{t-1}).
\]
We call $(s,t)$ a likelihood-array cell. It records the log probability assigned to the observed target token $x_t$ when the available left context begins at $x_s$.

For a fixed target position $t$, the sequence $\ell(1,t), \ell(2,t), \ldots, \ell(t-1,t)$ 
forms the target-token path for $x_t$. Moving from $s=1$ to $s=t-1$ progressively shortens the conditioning context, from all preceding tokens $x_1,\ldots,x_{t-1}$ to the single token $x_{t-1}$. For example, the path for target token $x_5$ consists of
\[
\begin{aligned}
\ell(1,5) &= \log p_\theta(x_5 \mid x_1,x_2,x_3,x_4), \\
\ell(2,5) &= \log p_\theta(x_5 \mid x_2,x_3,x_4), \\
\ell(3,5) &= \log p_\theta(x_5 \mid x_3,x_4), \\
\ell(4,5) &= \log p_\theta(x_5 \mid x_4).
\end{aligned}
\]
Thus, a target-token path records how the scoring model's assessment of the same observed token changes as the amount of preceding context varies. Collecting all cells yields the likelihood array $\mathcal{L}(X) =
\left\{
(s,t,\ell(s,t)) :
1 \leq s < t \leq T
\right\}$. The array is triangular because the path for target token $x_t$ contains $t-1$ cells, giving $T(T-1)/2$ cells in total. The conventional next-token log-probability sequence is the subset obtained by using all preceding tokens: $\ell_{\mathrm{full}}(t)=\ell(1,t)$ for $t=2,\ldots,T$. 


We now define the normalization of context and position. Let
$$
u(s,t)=\frac{\log(1+t-s)}{\log t},
\qquad
v(t;T)=\frac{t-2}{T-2}.
$$
The coordinate $u\in(0,1]$ measures the amount of available context on a normalized logarithmic scale. In particular, $u=1$ corresponds to using all preceding tokens, while smaller values correspond to shorter recent contexts. The logarithmic transformation provides finer resolution among short context windows, for which a fixed increase in length represents a larger proportional change. The coordinate $v\in[0,1]$ records the relative target position, with $v=0$ at the first predicted token and $v=1$ at the final token.

Together, $(u,v)$ maps the cells of variable-size likelihood arrays to a common context-scale and target-position domain. The observed context scales nevertheless remain path dependent: for a fixed target position $t$, both the number and locations of the values $\{u(s,t): s=1,\ldots,t-1\}$
depend on $t$. To place every target-token path on the same context-scale grid, let $u_g=\frac{g-1}{G-1}$, $g=1,\ldots,G,$ where $G$ controls the numerical resolution. For each target token, we linearly interpolate the observed pairs
$\{(u(s,t),\ell(s,t)): s=1,\ldots,t-1\}$
onto this grid and denote the resulting values by $\ell_t(u_g)$. At grid points below the smallest observed context scale, we use the value from the shortest available context.

\subsection{Likelihood-derived cell features}
\label{subsec:cell_features}

Alignment ensures that cell locations are comparable across texts, but the observed-token log probability alone can contain substantial variation unrelated to the class label. In particular, it depends on the intrinsic difficulty of the target token, the dispersion of the LLM's predictive distribution, and the amount of information supplied by the context. We therefore attach several likelihood-derived features to each aligned cell. All features are determined by the text and the fixed scoring model and are constructed without using the class labels.

The first feature is the observed-token log probability $\ell(s,t)$ introduced above. The second is the context contrast
\[
\delta(s,t)
=
\log\frac{p_\theta(x_t\mid x_s,\ldots,x_{t-1})}
{p_\theta(x_t\mid\mathrm{BOS})},
\]
where BOS denotes the beginning-of-sequence token. By comparing the probability of the same observed token with and without the preceding context, $\delta(s,t)$ removes part of the token-specific baseline predictability and isolates the incremental contribution of the available context. In this sense, the contrast directly helps address an important component of the second statistical challenge.
The next two features standardize these quantities against the scoring model's predictive distribution at the same cell. Let $A_{st}$ denote a vocabulary token drawn from $p_\theta(\cdot\mid x_s,\ldots,x_{t-1})$, and define
\begin{equation*}
\begin{aligned}
\mu_{st}
&=\mathbb E\{\log p_\theta(A_{st}\mid x_s,\ldots,x_{t-1})\},
&
\sigma_{st}^2
&=\operatorname{Var}\{\log p_\theta(A_{st}\mid x_s,\ldots,x_{t-1})\},\\
\mu_{st}^{\delta}
&=\mathbb E\!\left[
\log\frac{p_\theta(A_{st}\mid x_s,\ldots,x_{t-1})}
{p_\theta(A_{st}\mid\mathrm{BOS})}
\right],
&
(\sigma_{st}^{\delta})^2
&=\operatorname{Var}\!\left[
\log\frac{p_\theta(A_{st}\mid x_s,\ldots,x_{t-1})}
{p_\theta(A_{st}\mid\mathrm{BOS})}
\right],
\end{aligned}
\end{equation*}
where all expectations and variances are taken over the stated distribution of $A_{st}$.
The standardized log probability and standardized context contrast are
\begin{equation*}
z(s,t)=\frac{\ell(s,t)-\mu_{st}}{\sigma_{st}},
\qquad
z_{\delta}(s,t)=\frac{\delta(s,t)-\mu_{st}^{\delta}}{\sigma_{st}^{\delta}}.
\end{equation*}
These quantities express the observed-token log probability and context contrast in standard-deviation units relative to the model's predictive vocabulary distribution. We align $\delta(s,t)$, $z(s,t)$, and $z_{\delta}(s,t)$ using the same interpolation rule as for $\ell(s,t)$, and denote the resulting paths by $\delta_t(u_g)$, $z_t(u_g)$, and $z_{\delta,t}(u_g)$. The fifth feature describes how standardized token predictability changes as the available context increases. Let $\Delta_u=1/(G-1)$ and define \(\dot z_t(u_g)=\{z_t(u_{g+1})-z_t(u_{g-1})\}/(2\Delta_u)\) at the interior grid points, with one-sided differences at the endpoints. Thus, $\dot z_t(u_g)$ records the local response of standardized token predictability to an increase in context scale. The channel vector is defined as
\begin{equation*}
\boldsymbol{h}_t(u_g)=
\bigl(
\ell_t(u_g),
\delta_t(u_g),
z_t(u_g),
z_{\delta,t}(u_g),
\dot z_t(u_g)
\bigr)^\top
\in\mathbb{R}^5.
\end{equation*}
The five entries respectively describe observed-token predictability, the contribution of context, their standardized counterparts, and the local change in predictability as context increases. Although this paper uses the five channels above, the LAR framework permits other fixed, label-independent channel maps.

For text $i$, write $v_{it}=v(t;T_i)$ and $\boldsymbol{h}_{itg}=\boldsymbol{h}_{it}(u_g)$. The aligned multichannel likelihood array is
\begin{equation*}
\mathcal{A}_i=
\left\{
\bigl(
u_g,
v_{it},
\boldsymbol{h}_{itg}
\bigr):
t=2,\ldots,T_i;
g=1,\ldots,G
\right\}.
\end{equation*}
Every target-token path is now represented at the same $G$ context scales, and every path is located by the relative position of its target token. The resulting arrays remain variable in the number of target-token paths, but their cells share a common coordinate system and a common set of likelihood-derived features. The regression models in the next subsection determine how these aligned cell features and their locations contribute to the class label.

\subsection{First- and second-order likelihood-array models}
\label{subsec:models}

Define the conditional log odds for text $i$ by
\begin{equation*}
\eta_i
=
\log
\frac{\Pr(Y_i=1\mid\mathcal{A}_i)}
     {\Pr(Y_i=0\mid\mathcal{A}_i)}
\end{equation*}
We propose two nested models for $\eta_i$. LAR-1 aggregates contributions from individual aligned cells, whereas LAR-2 augments this cellwise model with second-order features constructed within each target-token path.

For an aligned cell, let $b(u,v,\boldsymbol{h})$ denote a shared cellwise contribution function. It maps the context scale $u$, relative target position $v$, and channel vector $\boldsymbol{h}$ to a contribution on the log-odds scale. Because $b(u,v,\boldsymbol{h})$ depends jointly on all three arguments, the contribution of a likelihood pattern can vary with both context scale and document position, and its dependence on the channel vector need not be linear.

The LAR-1 log odds are
\begin{equation}
\eta_i^{(1)}
=
\alpha
+
\frac{1}{T_i-1}
\sum_{t=2}^{T_i}
\left[
\frac{1}{G}
\sum_{g=1}^{G}
b\!\left(
u_g,
v_{it},
\boldsymbol{h}_{itg}
\right)
\right].
\label{eq:lar1}
\end{equation}
The inner average combines cellwise contributions across the aligned context scales for target token $t$. The outer average then combines the target-token paths across the document. The normalization assigns each target token equal total weight and prevents longer texts from contributing mechanically larger scores. Although the cells are averaged uniformly, their effective contributions are learned through the dependence of $b(u,v,\boldsymbol{h})$ on context scale and target position.

The term \emph{first-order} refers to the cellwise additive structure of LAR-1: each array cell enters through its own value of $b(u,v,\boldsymbol{h})$. This model can learn, for example, that a low standardized likelihood carries different information near the beginning and end of a document or under short and long preceding contexts. It does not, however, directly describe how two evaluations of the same target token behave jointly across context lengths.

Such joint behavior may contain additional information. For example, a target token may receive a low standardized likelihood under a short context. Its interpretation may differ depending on whether the likelihood remains low or rises substantially when more preceding context is supplied. To capture this behavior, LAR-2 introduces second-order features within each target-token path. For each relative target position $v$, let
\begin{equation*}
\kappa_v:
\bigl([0,1]\times\mathbb{R}^5\bigr)^2
\longrightarrow
\mathbb{R}
\end{equation*}
be a symmetric function of two aligned cells. For target token $t$ in text $i$, define the second-order path contribution
\begin{equation}
Q_{it}
=
\frac{1}{G^2}
\sum_{g=1}^{G}
\sum_{g'=1}^{G}
\kappa_{v_{it}}\!\left(
\bigl(u_g,\boldsymbol{h}_{itg}\bigr),
\bigl(u_{g'},\boldsymbol{h}_{itg'}\bigr)
\right).
\label{eq:path-pair-contribution}
\end{equation}
The LAR-2 log odds are defined as
\begin{equation}
\eta_i^{(2)}
=
\eta_i^{(1)}
+
\frac{1}{T_i-1}
\sum_{t=2}^{T_i}
Q_{it}.
\label{eq:lar2}
\end{equation}
For each target token, the double sum in \eqref{eq:path-pair-contribution} combines pairs of cells from the same target-token path. Terms with $g\neq g'$ describe associations across different context scales, while terms with $g=g'$ provide same-scale quadratic features. The resulting path-level contributions are then averaged across target tokens.
The restriction to cells from the same target-token path is deliberate. It captures how the scoring model's assessment of one observed token changes as additional context is supplied. Pairing cells from different target tokens would instead mix this context-response information with arbitrary associations between different tokens and would substantially enlarge the model.

We use the second-order term as a predictive augmentation rather than as a separately identified functional-ANOVA component. Because the cellwise function $b(u,v,\boldsymbol{h})$ is flexible, some same-scale or additive components of $\kappa_v$ could also be represented through $b(u,v,\boldsymbol{h})$. Our target in LAR-2 is therefore the total log odds $\eta_i^{(2)}$, rather than a unique decomposition into $b(u,v,\boldsymbol{h})$ and $\kappa_v$. In the finite-dimensional model below, ridge regularization produces a unique fitted score. Accordingly, comparisons between LAR-1 and LAR-2 measure the predictive value of adding within-path second-order features; they should not be interpreted as uniquely identifying an underlying interaction function.

Conventional likelihood scores can be viewed as prespecified reductions of the same token-level information. Average loss uses only the full-context cells where $u=1$, selects the log-probability channel, and averages it uniformly across target tokens. Min-K\% and Min-K\%++ also use the full-context cells but retain only a document-specific lower tail of the raw or standardized likelihoods. In contrast, LAR learns how likelihood information should be combined across context scales, target positions, and channels. Because the Min-K\% procedures use document-specific empirical quantiles, they are not literal special cases of the shared cellwise model; rather, they exemplify the prespecified aggregation rules that LAR is designed to avoid. The Supplement shows an empirical-measure representation of LAR-1 and LAR-2 and expresses these conventional full-context scores as functionals of the corresponding full-context empirical distributions. The main text retains the explicit cell- and path-based formulations because they are easier to follow. 

\subsection{Finite-dimensional estimation}
\label{subsec:estimation}

The models above involve the unknown functions $b(u,v,\boldsymbol{h})$ and $\kappa_v$. We approximate them using finite basis expansions, after which both LAR-1 and LAR-2 reduce to ridge logistic regression on document-level features. For the cellwise contribution function, let $\varphi_{u,1},\ldots,\varphi_{u,d_u}$, $\varphi_{v,1},\ldots,\varphi_{v,d_v}$, and $\varphi_{h,1},\ldots,\varphi_{h,d_h}$ be basis functions over context scale, relative target position, and the five-dimensional channel space, respectively. We approximate the function $b(u,v,\boldsymbol{h})$ by the tensor-product expansion
\begin{equation}
b(u,v,\boldsymbol{h})
=
\sum_{a=1}^{d_u}
\sum_{j=1}^{d_v}
\sum_{k=1}^{d_h}
\omega_{ajk}
\varphi_{u,a}(u)\varphi_{v,j}(v)\varphi_{h,k}(\boldsymbol{h}).
\label{eq:b-sieve}
\end{equation}
The channel basis may contain both linear and nonlinear functions of the five channels. Substituting \eqref{eq:b-sieve} into \eqref{eq:lar1} yields the document-level coordinates
\begin{equation}
Z^{(1)}_{i,ajk}
=
\frac{1}{(T_i-1)G}
\sum_{t=2}^{T_i}
\sum_{g=1}^{G}
\varphi_{u,a}(u_g)
\varphi_{v,j}(v_{it})
\varphi_{h,k}(\boldsymbol{h}_{itg}).
\label{eq:linear-coordinates}
\end{equation}
The LAR-1 log odds can therefore be written as
\begin{equation*}
\eta_i^{(1)}
=
\alpha
+
\sum_{a=1}^{d_u}
\sum_{j=1}^{d_v}
\sum_{k=1}^{d_h}
\omega_{ajk}Z^{(1)}_{i,ajk}.
\end{equation*}
Thus, the functional averaging in LAR-1 is performed when the document-level coordinates are constructed. Estimation subsequently involves an ordinary linear predictor with $d_ud_vd_h$ features.

To construct the second-order features, combine the context and channel bases into the path feature map $\boldsymbol{\psi}(u,\boldsymbol{h})
=
\left(
\varphi_{u,a}(u)\varphi_{h,k}(\boldsymbol{h})
\right)\in\mathbb{R}^{d}$, where $d=d_ud_h$, and $a=1,\ldots,d_u; k=1,\ldots,d_h$. For target token $t$ in text $i$, define its finite-dimensional path representation by
\begin{equation}
\boldsymbol{w}_{it}
=
\frac{1}{G}
\sum_{g=1}^{G}
\boldsymbol{\psi}\!\left(
u_g,
\boldsymbol{h}_{itg}
\right)
\in\mathbb{R}^{d}.
\label{eq:path-vector}
\end{equation}
This vector summarizes the joint variation of context scale and channel values along one target-token path. Within the finite model, we represent the second-order kernel as
\begin{equation}
\kappa_v\!\left(
(u,\boldsymbol{h}),
(u',\boldsymbol{h}')
\right)
=
\sum_{j=1}^{d_v}
\varphi_{v,j}(v)
\boldsymbol{\psi}(u,\boldsymbol{h})^{\mathsf T}
\boldsymbol{\Theta}_j
\boldsymbol{\psi}(u',\boldsymbol{h}'),
\label{eq:kernel-sieve}
\end{equation}
where each $\boldsymbol{\Theta}_j$ is a symmetric $d\times d$ coefficient matrix. Substituting \eqref{eq:kernel-sieve} into the path-specific double average in \eqref{eq:path-pair-contribution} yields the quadratic contribution
\begin{equation}
\sum_{j=1}^{d_v}
\varphi_{v,j}(v_{it})
\boldsymbol{w}_{it}^{\mathsf T}
\boldsymbol{\Theta}_j
\boldsymbol{w}_{it}.
\label{eq:path-quadratic-form}
\end{equation}
Forming \eqref{eq:path-quadratic-form} separately for each target token preserves the within-path restriction of LAR-2.

Directly estimating $\boldsymbol{\Theta}_1,\ldots,\boldsymbol{\Theta}_{d_v}$ would require $d_vd(d+1)/2$ quadratic coefficients. We reduce this dimension using random projection directions $\boldsymbol{a}_1,\ldots,\boldsymbol{a}_R
\in\mathbb{R}^{d},$
generated independently of the labeled sample. In the numerical studies, each direction is obtained by normalizing an independent Gaussian vector to have Euclidean norm $\sqrt d$. Within the finite model, each coefficient matrix is restricted to their rank-one span:
\begin{equation*}
\boldsymbol{\Theta}_j
=
\frac{1}{\sqrt R}
\sum_{r=1}^{R}
\gamma_{jr}
\boldsymbol{a}_r
\boldsymbol{a}_r^{\mathsf T}.
\end{equation*}
It follows that
\begin{equation*}
\boldsymbol{w}_{it}^{\mathsf T}
\boldsymbol{\Theta}_j
\boldsymbol{w}_{it}
=
\frac{1}{\sqrt R}
\sum_{r=1}^{R}
\gamma_{jr}
\left(
\boldsymbol{a}_r^{\mathsf T}
\boldsymbol{w}_{it}
\right)^2.
\end{equation*}
Each squared projection combines all pairs of context scales within the target-token path. The resulting document-level quadratic coordinates are
\begin{equation}
Z^{(2)}_{i,jr}
=
\frac{1}{\sqrt{R}(T_i-1)}
\sum_{t=2}^{T_i}
\varphi_{v,j}(v_{it})
\left(
\boldsymbol{a}_r^{\mathsf T}
\boldsymbol{w}_{it}
\right)^2.
\label{eq:quadratic-coordinates}
\end{equation}
The factor $R^{-1/2}$ is the normalization used in the theoretical projection map. Because the document-level coordinates are subsequently standardized within the training sample, this common rescaling does not change the fitted numerical regression.

Let $\boldsymbol{Z}_i^{(1)}$ collect the first-order coordinates in \eqref{eq:linear-coordinates}, and let $\boldsymbol{Z}_i^{(2)}$ collect the quadratic coordinates in \eqref{eq:quadratic-coordinates}. LAR-1 uses $\boldsymbol{Z}_i
=\boldsymbol{Z}_i^{(1)},$
whereas LAR-2 uses the concatenated feature vector $\boldsymbol{Z}_i
=
\left(
(\boldsymbol{Z}_i^{(1)})^{\mathsf T},
(\boldsymbol{Z}_i^{(2)})^{\mathsf T}
\right)^{\mathsf T}.$
Their respective dimensions are $d_ud_vd_h$ and $d_ud_vd_h+d_vR$.

All feature transformations are estimated from the training sample. This includes centering and scaling the channel values, path features, and document-level coordinates, as well as removing coordinates with zero empirical variance. The same transformations are then applied unchanged to validation or test texts. The channel basis and projection directions are also fixed before the regression coefficients are estimated.

Next, we fit the ridge logistic regression
\begin{align*}
(\widehat{\alpha},\widehat{\boldsymbol{\beta}})
=
\operatorname*{arg\,min}_{\alpha,\boldsymbol{\beta}}
\left\{
\frac{1}{n}
\sum_{i=1}^{n}
\left[
\log\left\{
1+\exp\left(
\alpha+\boldsymbol{\beta}^{\mathsf T}\boldsymbol{Z}_i
\right)
\right\}-
Y_i\left(
\alpha+\boldsymbol{\beta}^{\mathsf T}\boldsymbol{Z}_i
\right)
\right]
+
\frac{\lambda}{2}
\lVert\boldsymbol{\beta}\rVert_2^2
\right\}.
\end{align*}
The intercept is not penalized. The parameter $\lambda$ controls the tradeoff between fit and regularization and may be prespecified or selected using only the training data. The fitted linear predictor is $\widehat{\eta}_i
=\widehat{\alpha}+\widehat{\boldsymbol{\beta}}^{\mathsf T}
\boldsymbol{Z}_i,$ with larger values indicating stronger evidence for $Y_i=1$.


\section{Theoretical Analysis}
\label{sec:theory}

The preceding section introduced finite-dimensional estimators for the first- and second-order likelihood-array models. We now study an idealized population formulation in which the aligned feature map is fixed. The analysis concerns prediction in this functional model and the errors introduced by finite path features and squared projections; it does not cover likelihood computation, interpolation, training-sample standardization, or selection of the numerical basis. The main theoretical difficulty arises from the second-order term, which combines pairs of features from the same target-token path. Its prediction rate is therefore governed by the spectrum of pairwise products of path-feature covariance eigenvalues.


\subsection{Population formulation and assumptions}
\label{subsec:population_theory}

Let $\mathcal Z$ denote the channel space; for the five-channel construction in Section~\ref{subsec:cell_features}, $\mathcal Z=\mathbb R^5$. Let $\mathcal H$ be a
separable Hilbert space, and let $\psi:[0,1]\times\mathcal Z\longrightarrow\mathcal H$ be a fixed feature map. For target token $t$ in text $i$, define the Hilbert-space
path feature
\begin{equation*}
W_{it}
=
\frac{1}{G}
\sum_{g=1}^{G}
\psi(u_g,\boldsymbol h_{itg}).
\end{equation*}
This feature combines context scale and channel value along a single target-token path.
For example, setting $\mathcal H=\mathbb R^{d_ud_h}$ and
$\psi(u,\boldsymbol h)
=
\bigl(
\varphi_{u,a}(u)\varphi_{h,k}(\boldsymbol h)
\bigr)_{a,k}$ recovers the finite path vector in \eqref{eq:path-vector}. The general Hilbert-space
formulation allows the path representation to be infinite dimensional.

Let $\varphi_{v,1},\ldots,\varphi_{v,d_v}$ be basis functions over relative target
position, and write $\boldsymbol\varphi_v(v)
=\bigl(
\varphi_{v,1}(v),\ldots,\varphi_{v,d_v}(v)
\bigr)^{\mathsf T}.$
We assume that their span contains the constant function, as does the piecewise-linear basis used in the numerical studies. The symbol $\otimes$ denotes the tensor product, the infinite-dimensional analogue of the Kronecker product. In particular, $\boldsymbol\varphi_v(v)\otimes W$ is simply the block feature whose $j$th component is $\varphi_{v,j}(v)W$. To represent second-order effects, we use the symmetric tensor product $W\otimes_s W$. This is the Hilbert-space extension of an ordinary outer product: if $W$ is represented by a finite vector $\boldsymbol w\in\mathbb R^d$, then $W\otimes_s W$ corresponds to the symmetric matrix $\boldsymbol w\boldsymbol w^{\mathsf T}$, and a quadratic coefficient matrix $\Theta$ contributes $\boldsymbol w^{\mathsf T}\Theta\boldsymbol w$. Thus, the tensor notation records all pairwise products of the components of the same path feature.

The uncentered first- and second-order features of text $i$ are
\begin{equation}
\begin{aligned}
\overline\Phi_i^{(1)}
&=
\frac{1}{T_i-1}
\sum_{t=2}^{T_i}
\boldsymbol\varphi_v(v_{it})\otimes W_{it},
\\
\overline\Phi_i^{(2)}
&=
\frac{1}{T_i-1}
\sum_{t=2}^{T_i}
\boldsymbol\varphi_v(v_{it})
\otimes
\bigl(W_{it}\otimes_s W_{it}\bigr).
\end{aligned}
\label{eq:population_path_features}
\end{equation}
The first component averages position-indexed path features across target tokens. The second forms a quadratic feature separately within each target-token path before applying
the same position-sensitive averaging. It therefore preserves the same-token restriction of LAR-2 in \eqref{eq:lar2}.

For the covariance analysis, define the centered features
$\Phi_i^{(r)}
=
\overline\Phi_i^{(r)}
-
\mathbb E\bigl(\overline\Phi_i^{(r)}\bigr), r\in\{1,2\},$
and write $\Phi_i=(\Phi_i^{(1)},\Phi_i^{(2)})$. We equip this pair with the natural product Hilbert-space inner product. Centering only
changes the intercept and has no effect on the model's slopes.
We analyze the population logistic model
\begin{equation}
\mathbb P(Y_i=1\mid\Phi_i)
=
\left[
1+
\exp\left\{
-\alpha_0
-
\left\langle\vartheta_0^{(1)},\Phi_i^{(1)}\right\rangle
-
\left\langle\vartheta_0^{(2)},\Phi_i^{(2)}\right\rangle
\right\}
\right]^{-1},
\label{eq:theory_logistic_model}
\end{equation}
where
$\vartheta_0=(\vartheta_0^{(1)},\vartheta_0^{(2)})$ contains the first- and second-order coefficient functions. Here $\langle\cdot,\cdot\rangle$ denotes the Hilbert-space inner product; for finite vectors, it reduces to the ordinary transpose product. Accordingly,
the conditional log odds based on $\Phi_i$ are
\begin{equation*}
\eta_0(\Phi_i)
=
\alpha_0
+
\left\langle\vartheta_0^{(1)},\Phi_i^{(1)}\right\rangle
+
\left\langle\vartheta_0^{(2)},\Phi_i^{(2)}\right\rangle.
\end{equation*}
The components involving $\Phi_i^{(1)}$ and $\Phi_i^{(2)}$ represent,
respectively, first-order path effects and position-varying quadratic effects within
target-token paths. Thus, \eqref{eq:theory_logistic_model} is the population counterpart
of LAR-2 in Section~\ref{subsec:models}. 

\begin{assumption}[Bounded feature maps]
\label{ass:bounded_features}
There are finite constants $K_\psi$ and $K_v$ such that
\begin{equation*}
\sup_{(u,\boldsymbol h)\in[0,1]\times\mathcal Z}
\left\|\psi(u,\boldsymbol h)\right\|
\leq K_\psi,
\qquad
\sup_{v\in[0,1]}
\left\|\boldsymbol\varphi_v(v)\right\|
\leq K_v.
\end{equation*}
\end{assumption}

Because each $W_{it}$ averages bounded feature-map values,
Assumption~\ref{ass:bounded_features} implies that $\Phi_i^{(1)}$ and
$\Phi_i^{(2)}$ are uniformly bounded by constants that do not increase with text
length. This condition is used to control the empirical logistic risk and to relate excess
logistic risk to squared prediction loss.

\begin{assumption}[Spectral decay of path features]
\label{ass:path_capacity}
There exist a countable orthonormal basis $\{e_j:j\geq1\}$ of $\mathcal H$ and a
nonincreasing sequence $\{\kappa_j:j\geq1\}$ of nonnegative numbers satisfying $\kappa_j \leq C_\kappa j^{-\nu}$, $j\geq1$ for some $\nu>1$ and $C_\kappa<\infty$. Let $\mathcal S:\mathcal H\to\mathcal H$ be the operator defined by $\mathcal S e_j=\kappa_j e_j$. For some
$C_\Gamma<\infty$,
\begin{equation}
\operatorname{Cov}(\Phi^{(1)})
\preceq
C_\Gamma
(I_{d_v}\otimes\mathcal S),
\qquad
\operatorname{Cov}(\Phi^{(2)})
\preceq
C_\Gamma
\left\{
I_{d_v}\otimes
(\mathcal S\otimes_s\mathcal S)
\right\}.
\label{eq:path_capacity}
\end{equation}
Here $I_{d_v}$ is the identity operator on $\mathbb R^{d_v}$, and
$A\preceq B$ means that $B-A$ is positive semidefinite. The operator
$I_{d_v}\otimes\mathcal S$ applies $\mathcal S$ separately to each target-position
block, while $\mathcal S\otimes_s\mathcal S$ assigns covariance order
$\kappa_j\kappa_k$ to the pair of path directions $(e_j,e_k)$.
\end{assumption}

Assumption~\ref{ass:path_capacity} controls the variation of the text-level features
along increasingly complex path directions. The condition $\nu>1$ ensures that
$\sum_{j\geq1}\kappa_j<\infty$. The second covariance bound reflects the
within-path quadratic construction: variation in the symmetric tensor direction indexed
by $(e_j,e_k)$ is at most of order $\kappa_j\kappa_k$. Importantly, this assumption
does not require different target-token paths within a text to be independent. It is
sufficient for analogous bounds to hold for the contribution of a target-token path
selected uniformly from a text; the averaging in \eqref{eq:population_path_features}
then implies \eqref{eq:path_capacity}. A formal statement is provided in the Supplementary
Material.

\subsection{Minimax prediction rate for the within-path quadratic model}
\label{subsec:minimax_theory}

For fixed $\nu>1$ and $M>0$, define
\begin{equation*}
\mathcal P_\nu(M)
=
\left\{
P:
\begin{array}{l}
P\text{ satisfies Assumptions~\ref{ass:bounded_features}--\ref{ass:path_capacity}}
\\
\text{and the logistic model \eqref{eq:theory_logistic_model}},
\quad
|\alpha_0|\leq M,
\quad
\|\vartheta_0\|\leq M
\end{array}
\right\}.
\end{equation*}
The bounds in Assumption~\ref{ass:bounded_features} and the constants
$C_\kappa$ and $C_\Gamma$ in Assumption~\ref{ass:path_capacity} are understood to
hold uniformly over $P\in\mathcal P_\nu(M)$. The class is defined in terms of the
induced feature-label pair $(\Phi,Y)$: a fixed scoring model, channel construction, and
feature map determine one such joint distribution.

Let $P_\Phi$ denote the marginal distribution of $\Phi$. For any candidate score $\eta$,
define the squared prediction loss
\begin{equation*}
\|\eta-\eta_0\|_{P_\Phi}^2
=
\mathbb E\left[
\{\eta(\Phi)-\eta_0(\Phi)\}^2
\right].
\end{equation*}
Under the boundedness and radius restrictions above, the candidate and population scores
are uniformly bounded. Consequently, squared prediction loss is equivalent, up to fixed
constants, to excess population logistic risk.

The quadratic feature produces a product covariance spectrum. To see its effect, consider
the least favorable submodel used for the lower bound, in which
$\kappa_j\asymp j^{-\nu}$. The covariance eigenvalues associated with the symmetric
tensor directions $(e_j,e_k)$ are then of order
\begin{equation*}
\kappa_j\kappa_k
\asymp
(jk)^{-\nu},
\qquad 1\leq j\leq k.
\end{equation*}
The number of pairs $(j,k)$ satisfying $jk\leq m$ is of order $m\log(em)$. Hence,
if $\rho_\ell$ denotes the $\ell$th largest eigenvalue of the quadratic covariance
envelope, then
\begin{equation}
\rho_\ell
\asymp
\left\{
\frac{\ell}{\log(e\ell)}
\right\}^{-\nu}.
\label{eq:ordered_product_spectrum}
\end{equation}
In contrast to functional linear regression, which is governed by a single covariance
sequence, the quadratic model is governed by the ordered pairwise products of that
sequence. The logarithmic factor in \eqref{eq:ordered_product_spectrum} is therefore a
direct consequence of the within-path quadratic construction.

\begin{theorem}[Minimax lower bound]
\label{thm:quadratic_lower}
For fixed $\nu>1$ and $M>0$, there exists a constant $c>0$ such that, for all
sufficiently large $n$,
\begin{equation}
\inf_{\widetilde{\eta}}
\sup_{P\in\mathcal P_\nu(M)}
\mathbb E_P
\|\widetilde{\eta}-\eta_0\|_{P_\Phi}^2
\geq
c
\left\{
\frac{\log(en)}{n}
\right\}^{\nu/(\nu+1)}.
\label{eq:quadratic_lower}
\end{equation}
The infimum is over all estimators based on
$\{(Y_i,\Phi_i):i=1,\ldots,n\}$.
\end{theorem}

The proof restricts $\mathcal P_\nu(M)$ to a quadratic submodel with fixed text length,
no first-order effect, and path coefficients constructed along the pair-indexed directions
in \eqref{eq:ordered_product_spectrum}. Assouad's lemma then yields
\eqref{eq:quadratic_lower}. Thus, no estimator can attain a uniformly smaller expected prediction error over $\mathcal P_\nu(M)$. We next show that this lower bound is attained by ridge estimation using the complete
Hilbert-space features. Define
\begin{equation*}
\begin{aligned}
(\widehat\alpha_\lambda,\widehat\vartheta_\lambda)
\in
\operatorname*{arg\,min}_{|\alpha|\leq M,\,\|\vartheta\|\leq M}
\Bigg\{
&
\frac{1}{n}
\sum_{i=1}^n
\left[
\log\left\{
1+
\exp\left(
\alpha+\langle\vartheta,\Phi_i\rangle
\right)
\right\}
-
Y_i
\left(
\alpha+\langle\vartheta,\Phi_i\rangle
\right)
\right]
+
\frac{\lambda}{2}
\|\vartheta\|^2
\Bigg\}.
\end{aligned}
\end{equation*}
The fitted score is $\widehat{\eta}_\lambda(\Phi)=\widehat\alpha_\lambda+\left\langle\widehat\vartheta_\lambda,\Phi\right\rangle$.

The radius constraint is used to control the functional estimator uniformly in the
theoretical analysis. It is not imposed in the numerical implementation; the two
formulations agree whenever the unconstrained solution lies in the specified ball.

Let $\Gamma=\operatorname{Cov}(\Phi)$ denote the covariance operator of the combined
feature. Assumption~\ref{ass:path_capacity} and the ordered product spectrum imply the
effective-dimension bound
\begin{equation*}
\operatorname{tr}
\left\{
\Gamma(\Gamma+\lambda I)^{-1}
\right\}
\lesssim
\lambda^{-1/\nu}
\log(e/\lambda),
\qquad 0<\lambda\leq1.
\end{equation*}
This bound is the source of the estimation-error term in the following result.

\begin{theorem}[Upper bound for the within-path ridge estimator]
\label{thm:quadratic_upper}
There exists a constant $C<\infty$ such that, for every $0<\lambda\leq1$,
\begin{equation}
\sup_{P\in\mathcal P_\nu(M)}
\mathbb E_P
\|\widehat{\eta}_\lambda-\eta_0\|_{P_\Phi}^2
\leq
C
\left\{
\lambda
+
\frac{
\lambda^{-1/\nu}\log(e/\lambda)
}{n}
\right\}.
\label{eq:quadratic_upper}
\end{equation}
\end{theorem}

The two terms in \eqref{eq:quadratic_upper} are the regularization bias and the
estimation error, respectively. Choosing $\lambda
\asymp
\left\{
\frac{\log(en)}{n}
\right\}^{\nu/(\nu+1)}$ balances them and implies
\begin{equation*}
\sup_{P\in\mathcal P_\nu(M)}
\mathbb E_P
\|\widehat{\eta}_\lambda-\eta_0\|_{P_\Phi}^2
\leq
C
\left\{
\frac{\log(en)}{n}
\right\}^{\nu/(\nu+1)}.
\end{equation*}
Together, Theorems~\ref{thm:quadratic_lower} and~\ref{thm:quadratic_upper} establish
the minimax prediction rate over $\mathcal P_\nu(M)$. The matching logarithmic factor is induced by the pair-indexed quadratic spectrum and cannot be removed uniformly over this class. 

\subsection{Finite path sieves and squared projections}
\label{subsec:finite_sieve_theory}

The preceding rate is derived for the complete path-feature space, whereas the estimator
in Section~\ref{subsec:estimation} uses a $d$-dimensional path sieve and $R$ squared
projections. Let $\eta_{0,d}$ be the population score after projection onto the finite
path sieve, and write $\epsilon_d^2=\|\eta_{0,d}-\eta_0\|_{P_\Phi}^2$ for its approximation
error. An unrestricted quadratic model in $d$ path coordinates would contain
$d_vd(d+1)/2$ coefficients. The squared-projection representation uses $d_vR$
coefficients and preserves the within-path construction described in
Section~\ref{subsec:estimation}.

\begin{theorem}[Finite-dimensional approximation]
\label{thm:finite_projection}
Suppose Assumptions~\ref{ass:bounded_features}--\ref{ass:path_capacity} hold, with
$d_v$ fixed. Assume
also that the linear predictors in the radius-constrained finite-sieve classes are
uniformly bounded by a constant independent of $d$ and $R$. Fix
$M_1\geq\sqrt{3}M$, and let
$\widehat{\eta}_{d,R,\lambda}$ be the ridge estimator over
$|\alpha|\leq M_1$ and $\|\boldsymbol{\beta}\|\leq M_1$, based on the $d$-dimensional path
sieve and $R$ normalized Gaussian squared projections generated independently of the
labeled sample. For $0<\lambda\leq1$,
\begin{equation*}
\mathbb E_{\mathcal D,\mathcal R}
\|\widehat{\eta}_{d,R,\lambda}-\eta_0\|_{P_\Phi}^2
\leq
C
\left\{
\epsilon_d^2+\lambda+
\frac{\lambda^{-1/\nu}\log(e/\lambda)}{n}+\frac{1}{R}
\right\}.
\end{equation*}
\end{theorem}

Here $\mathbb E_{\mathcal D,\mathcal R}$ is over the labeled training sample and the independently generated projection directions. The four terms are, respectively, the path-sieve approximation error, regularization bias, estimation error, and squared-projection approximation error. The last term concerns the scalar quadratic score in prediction loss; it does not require approximation of an unrestricted coefficient matrix in Frobenius norm. The more general effective-dimension bound and the random-projection argument are presented in the Supplement.

For the oracle spectral sieve spanned by the population covariance directions $e_1,\ldots,e_d$, the approximation error satisfies $\epsilon_d^2\leq Cd^{-\nu}$. Consequently, this finite estimator retains the minimax rate when
\begin{equation}
d
\gtrsim
\left\{
\frac{n}{\log(en)}
\right\}^{1/(\nu+1)},
\qquad
R
\gtrsim
\left\{
\frac{n}{\log(en)}
\right\}^{\nu/(\nu+1)},
\qquad
\lambda
\asymp
\left\{
\frac{\log(en)}{n}
\right\}^{\nu/(\nu+1)}.
\label{eq:finite_resolution_orders}
\end{equation}
Under these choices, the expected prediction error is of order $\{\log(en)/n\}^{\nu/(\nu+1)}$. The spline and random channel bases used in the numerical studies are not assumed to coincide with these population covariance directions. For those and other basis families, Theorem~\ref{thm:finite_projection} applies with their approximation error $\epsilon_d^2$ left explicit.

\subsection{Implications for AUC}
\label{subsec:auc_theory}

The previous results concern prediction of the conditional log odds, while detection is
evaluated by ranking. Let $q(\Phi)=\mathbb P(Y=1\mid\Phi)$ and let $\Phi'$ be an
independent copy of $\Phi$. Since $\eta_0(\Phi)=\operatorname{logit}\{q(\Phi)\}$,
$\eta_0$ induces the optimal population ranking among scores based on $\Phi$. The
following condition controls the probability of nearly tied pairs.

\begin{assumption}[Pairwise margin condition]
\label{ass:auc_margin}
There exist constants $C_m<\infty$ and $\tau\geq0$ such that
\begin{equation*}
\mathbb P
\left\{
0<
|q(\Phi)-q(\Phi')|
\leq t
\right\}
\leq
C_m t^\tau,
\qquad
t>0.
\end{equation*}
In addition, there is a constant \(c_Y>0\) such that
\(\mathbb P(Y=y)\geq c_Y\) for \(y\in\{0,1\}\).
\end{assumption}

The condition permits class overlap; it holds with $\tau=0$ without further restrictions
and with $\tau=1$ when $q(\Phi)$ has a bounded density. Such conditions are standard in
bipartite ranking \citep{clemencon2008ranking}.

\begin{proposition}[AUC convergence]
\label{prop:auc_convergence}
Suppose Assumption~\ref{ass:auc_margin} holds. Then any score estimator
$\widehat{\eta}$ based on the training sample satisfies
\begin{equation*}
\operatorname{AUC}(\eta_0)
-
\mathbb E
\left\{
\operatorname{AUC}(\widehat{\eta})
\right\}
\leq
C
\left[
\mathbb E
\|\widehat{\eta}-\eta_0\|_{P_\Phi}^2
\right]^{(\tau+1)/(\tau+2)}.
\end{equation*}
\end{proposition}

Together with Theorem~\ref{thm:finite_projection}, the proposition implies that the
expected AUC converges to $\operatorname{AUC}(\eta_0)$ at rate
$\{\log(en)/n\}^{\nu(\tau+1)/\{(\nu+1)(\tau+2)\}}$ under
\eqref{eq:finite_resolution_orders}. The next section provides detailed numerical results on AUC.

\section{Numerical Studies}
\label{sec:numerical}

The numerical studies investigate three aspects of the proposed method: (1) its detection performance across scoring models and applications, (2) the contributions of shorter-context information and the within-path second-order augmentation, and (3) its sensitivity to labeled sample size and numerical resolution. We assess overall performance through comparisons with existing procedures and study the sources of detection information through restricted fits. The main paper also examines the effect of labeled sample size. Simulation, same-supervision comparisons, transfer results, and sensitivity to the numerical specification are reported in the Supplement.

Unless otherwise stated, all numerical studies use \(G=24\) equally spaced context-scale points. We use piecewise-linear spline bases with \(d_u=12\) functions over context scale and \(d_v=6\) functions over relative target position. The channel basis consists of a constant function, the five likelihood-derived channels, and 96 random trigonometric functions associated with Gaussian radial basis kernels at bandwidths \(0.5\), \(1\), and \(2\), giving \(d_h=102\). LAR-2 uses \(R=1024\) fixed projection directions. Both estimators are fitted with ridge penalty \(\lambda=0.01\). These values define a common working specification across applications and scoring models rather than being selected separately for each reported comparison. The individual choices are varied in the sensitivity analysis reported in the Supplementary Material.

The primary comparisons in Sections~\ref{sec:wikimia} and~\ref{sec:aidetect} use AUC on held-out texts. We use 20 repeated 80/20 train--test splits, refitting LAR within each split and evaluating the stored comparison scores on the same test texts. WikiMIA splits are stratified jointly by label and text-length group. In the generated-text data, each human passage and its model-generated counterpart remain in the same split, with domains balanced across splits. Entries in Tables~\ref{tab:wikimia} and~\ref{tab:aidetect} report AUCs with descriptive uncertainty summaries in parentheses. The subsequent diagnostic analyses state their evaluation protocols separately. Partial AUC and true-positive rates at selected false-positive rates are reported in the Supplement.

\subsection{WikiMIA membership inference}
\label{sec:wikimia}

WikiMIA is a widely used benchmark for pretraining-data membership inference \citep{shi2024detecting,zhang2025minkplusplus,xie2024recall,zhang2025finetuning}. It assigns proxy membership labels to English Wikipedia event articles based on their creation dates. Articles created well before corpus collection are plausible members, whereas those created after the relevant data cutoff are necessarily absent. WikiMIA approximates this distinction by labeling articles created before 2017 as members and those created after January 1, 2023, as nonmembers. The LLM models were selected to be compatible with this temporal split. Pythia was trained on the Pile, assembled in 2020, and released in April 2023; Falcon-7B was trained in March 2023 and released in May 2023; and LLaMA 2 was released in July 2023 with a reported pretraining-data cutoff of September 2022 \citep{gao2020pile,biderman2023pythia,tii2023falcon7b,touvron2023llama2}. 

The dataset contains segments of 32, 64, 128, and 256 words. These four length groups contain 776, 542, 250, and 82 samples, respectively. Pooling them results in the 1,650 samples analyzed here, including 861 members and 789 nonmembers. 
We compare the proposed method with Loss, Zlib, Min-K\%, Min-K\%++, Ref, ReCaLL, Lowercase, PAC, and DC-PDD \citep{carlini2021extracting,shi2024detecting,zhang2025minkplusplus,xie2024recall,ye2024pac}. Lowercase, PAC, and DC-PDD require additional evaluations or perturbations of the target text.

\begin{table}[ht!]
\centering
\caption{Membership inference on 1,650 WikiMIA texts. Entries are mean test AUCs over 20 label- and length-stratified 80/20 splits, with standard deviations across splits in parentheses. LAR is refitted within each split; the comparison scores are evaluated on the same test texts. The largest mean AUC in each column is shown in bold.}
\label{tab:wikimia}
\footnotesize
\setlength{\tabcolsep}{3pt}
\renewcommand{\arraystretch}{0.8}
\begin{tabular}{lccccc}
\toprule
Method & Pythia-1B & Pythia-6.9B & Pythia-12B & LLaMA-2-7B & Falcon-7B \\
\midrule
\multicolumn{6}{l}{\textit{Direct likelihood scores}} \\
Loss & 0.588 (0.022) & 0.634 (0.024) & 0.647 (0.025) & 0.531 (0.028) & 0.578 (0.030) \\
Zlib & 0.616 (0.019) & 0.659 (0.019) & 0.669 (0.022) & 0.546 (0.026) & 0.595 (0.026) \\
Min-K\% & 0.596 (0.017) & 0.667 (0.022) & 0.688 (0.029) & 0.528 (0.028) & 0.579 (0.029) \\
Min-K\%++ & 0.625 (0.021) & 0.707 (0.021) & 0.725 (0.029) & 0.614 (0.031) & 0.780 (0.021) \\
\midrule
\multicolumn{6}{l}{\textit{Calibration- and perturbation-assisted scores}} \\
Ref & 0.581 (0.034) & 0.636 (0.037) & 0.641 (0.031) & 0.554 (0.034) & 0.556 (0.035) \\
ReCaLL & 0.765 (0.023) & 0.841 (0.018) & 0.871 (0.018) & 0.769 (0.022) & 0.836 (0.024) \\
Lowercase & 0.553 (0.025) & 0.591 (0.028) & 0.616 (0.024) & 0.497 (0.020) & 0.551 (0.028) \\
PAC & 0.632 (0.015) & 0.705 (0.019) & 0.705 (0.020) & 0.544 (0.029) & 0.629 (0.025) \\
DC-PDD & 0.565 (0.020) & 0.583 (0.021) & 0.595 (0.025) & 0.559 (0.029) & 0.586 (0.032) \\
\midrule
\multicolumn{6}{l}{\textit{Likelihood-array estimators}} \\
LAR-1 & 0.914 (0.012) & 0.926 (0.013) & 0.934 (0.013) & 0.951 (0.013) & 0.960 (0.009) \\
LAR-2 & \textbf{0.953} (0.009) & \textbf{0.958} (0.010) & \textbf{0.966} (0.007) & \textbf{0.976} (0.008) & \textbf{0.983} (0.007) \\
\bottomrule
\end{tabular}
\end{table}

Among the comparison procedures, ReCaLL records the largest mean AUC for each of the five scoring models, with values between \(0.765\) and \(0.871\). Both LAR estimators obtain higher mean AUCs than every comparison procedure across all five models. LAR-2 improves on LAR-1 by \(0.023\)--\(0.039\). The consistent increase indicates that the first-order functional does not capture all of the benchmark separation available in the likelihood array; the restricted analyses below examine the contribution of the within-path second-order augmentation.

As a comparison with another supervised procedure, Fine-tuned Score Deviation reports a maximum AUC of \(0.90\) for Pythia-6.9B, compared with \(0.958\) for LAR-2 in Table~\ref{tab:wikimia} \citep{zhang2025finetuning}. The computational requirements also differ: Fine-tuned Score Deviation uses known nonmembers to perform LoRA fine-tuning of the audited language model, whereas the supervised stage of LAR fits a ridge regression once the likelihood arrays are available.

\subsection{AI-generated text detection}
\label{sec:aidetect}

The generated-text study uses the GPT-3-to-4 benchmark assembled for the evaluation of AdaDetectGPT \citep{zhou2025adadetectgpt}. We consider four domains that differ substantially in genre: fiction from WritingPrompts, news articles from XSum, high-school and university essays, and consumer reviews from Yelp. From each domain, we use 150 human passages and 150 passages generated by GPT-4o-2024-08-06, giving 600 matched pairs and 1,200 texts. The detection label distinguishes a GPT-4o passage from its human counterpart. The Supplement repeats the evaluation with Claude-3.5-Haiku and Gemini-2.5-Flash as two additional generators.

Each generated passage is constructed from the same source as the human passage with which it is paired. GPT-4o receives the first 120 tokens of the human text and a domain-specific instruction to continue it, with a requested length of approximately 150 words for fiction, news, and reviews and 200 words for essays. Generation uses temperature $0.8$ and at most 200 output tokens. The two passages in each pair are then truncated to the shorter word count. They consequently share their domain, opening context, and length, limiting the extent to which the detection task can be resolved from topic or document length alone.

The scoring models are DeepSeek-R1-Distill-Qwen-7B (DeepSeek-7B), OLMo-7B, Qwen3-4B, Qwen3-8B, and Qwen3-14B. All five are open-weight models and are distinct from the GPT-4o generator. The analysis therefore evaluates whether the likelihood-array representation remains informative when its probabilities are obtained from an independently chosen scoring model rather than from the generator itself.

\begin{table}[ht!]
\centering
\caption{AI-generated text detection on 1,200 texts from four domains. Entries are AUCs, with descriptive uncertainty summaries in parentheses. The repeated-split results use 20 domain-stratified, source-pair-grouped 80/20 splits. The largest AUC in each column is shown in bold. RADAR and StatDetectLLM each span the five scoring-model columns.}
\label{tab:aidetect}
\footnotesize
\setlength{\tabcolsep}{3pt}
\renewcommand{\arraystretch}{0.8}
\def\dashunit{\rule[0.55ex]{4pt}{0.3pt}\hspace{2pt}}
\def\dashsegment{\hbox{\dashunit\dashunit\dashunit\dashunit\dashunit\dashunit\dashunit\dashunit\dashunit\dashunit\dashunit\dashunit\dashunit}}
\begin{tabular}{lccccc}
\toprule
Method & DeepSeek-7B & OLMo-7B & Qwen3-4B & Qwen3-8B & Qwen3-14B \\
\midrule
\multicolumn{6}{l}{\textit{Scoring-model detection scores}} \\
Likelihood & 0.777 (0.020) & 0.644 (0.022) & 0.809 (0.020) & 0.821 (0.019) & 0.824 (0.020) \\
Entropy & 0.664 (0.016) & 0.531 (0.022) & 0.667 (0.016) & 0.689 (0.018) & 0.685 (0.019) \\
Rank & 0.627 (0.015) & 0.607 (0.009) & 0.628 (0.012) & 0.628 (0.014) & 0.613 (0.010) \\
Log-rank & 0.752 (0.020) & 0.644 (0.022) & 0.793 (0.019) & 0.801 (0.018) & 0.802 (0.019) \\
LRR & 0.608 (0.016) & 0.617 (0.018) & 0.626 (0.015) & 0.657 (0.016) & 0.651 (0.016) \\
Fast-DetectGPT & 0.730 (0.019) & 0.773 (0.019) & 0.664 (0.009) & 0.666 (0.012) & 0.706 (0.011) \\
\midrule
\multicolumn{6}{l}{\textit{Supervised detectors}} \\
AdaDetectGPT & 0.711 (0.015) & 0.773 (0.013) & 0.660 (0.016) & 0.662 (0.016) & 0.702 (0.015) \\
RADAR & \multicolumn{2}{c}{\strut\dashsegment} & 0.839 (0.019) & \multicolumn{2}{c}{\strut\dashsegment} \\
StatDetectLLM & \multicolumn{2}{c}{\strut\dashsegment} & 0.980 (0.005) & \multicolumn{2}{c}{\strut\dashsegment} \\
\midrule
\multicolumn{6}{l}{\textit{Likelihood-array regression}} \\
LAR-1 & 0.985 (0.004) & 0.994 (0.002) & \textbf{0.995} (0.003) & \textbf{0.995} (0.003) & \textbf{0.996} (0.002) \\
LAR-2 & \textbf{0.987} (0.004) & \textbf{0.995} (0.003) & 0.994 (0.004) & 0.994 (0.003) & 0.991 (0.006) \\
\bottomrule
\end{tabular}
\end{table}

The comparison procedures include likelihood, entropy, rank, log-rank, LRR, Fast-DetectGPT, and AdaDetectGPT \citep{gehrmann2019gltr,su2023detectllm,mitchell2023detectgpt,bao2024fastdetectgpt,zhou2025adadetectgpt}. We also include two text-only detectors trained on external labeled data: RADAR and StatDetectLLM \citep{hu2023radar,zhou2026performanceguarantees}. StatDetectLLM is trained on a large external corpus and uses LoRA to fine-tune the one-billion-parameter Gemma 3 model. We evaluate it through the authors' public implementation, applying its General, News, Academia, and User Review settings to the writing, XSum, essay, and Yelp domains, respectively. The implementation returns a p-value for each text, and we use its negative as the ranking score for computing AUC. In contrast, LAR does not fine-tune a language model: its supervised stage fits only the regularized regression after the likelihood arrays have been constructed, although array construction requires access to token-level logits under arbitrary left-context windows.

Among the externally trained detectors, RADAR records a mean AUC of \(0.839\), while StatDetectLLM records \(0.980\). LAR-1 and LAR-2 obtain higher mean AUCs than every comparison procedure for each scoring model. LAR-1 AUCs range from \(0.985\) to \(0.996\), while the corresponding LAR-2 AUCs differ by at most $0.005$. Although the first-order model is nested within the second-order model at the population level, the second-order estimator uses the larger product-spectrum feature class studied in Section~\ref{sec:theory} and therefore incurs additional finite-sample estimation and projection error. When the within-path second-order signal is weak, this additional cost can offset its small approximation gain. The results are therefore consistent with generated-text separation being captured primarily by the first-order functional, with little additional information from the second-order augmentation.

\subsection{Sources of detection information}
\label{sec:mechanism_results}

We examine two possible sources of detection information: the likelihood-derived channels and the location of information within the array. The analysis uses WikiMIA with Pythia-1B and AI-generated text detection with DeepSeek-7B. We first add the five channels cumulatively in four groups to assess the contribution of different likelihood quantities. We then fit both LAR-1 and LAR-2 to the full-context cells, the shorter-context cells, and the complete array. These restrictions assess whether shorter-context cells contribute beyond full-context likelihood and whether this pattern persists after the within-path second-order augmentation is added. These diagnostics use five-fold out-of-fold predictions, with the same basis and ridge specification across fits.

\begin{figure}[ht!]
\centering
\begingroup
\small
\renewcommand{\baselinestretch}{1}\selectfont
\includegraphics[width=0.90\textwidth]{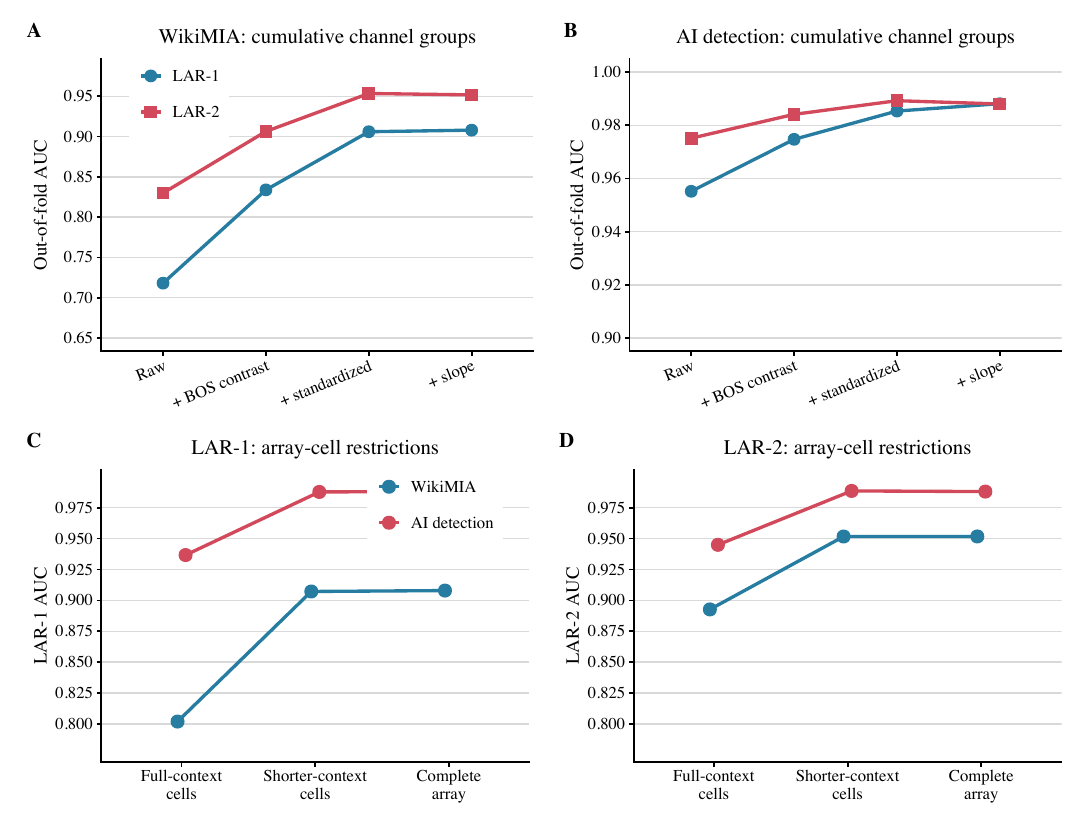}
\caption{Sources of detection information in the likelihood array. Panels A and B report five-fold out-of-fold AUC as the likelihood-derived channels enter LAR-1 and LAR-2 cumulatively. Panels C and D compare LAR-1 and LAR-2, respectively, when fitted to the full-context cells, the shorter-context cells, and the complete array. The two settings are WikiMIA with Pythia-1B and AI-generated text detection with DeepSeek-7B.}
\label{fig:sieve_mechanism}
\endgroup
\end{figure}

In the cumulative sequence, the largest increase occurs after adding the BOS contrast for WikiMIA and after adding the standardized channels for generated-text detection. Adding the local context-response channel changes AUC little in either setting. For LAR-1, the shorter-context cells record AUCs of \(0.907\) in WikiMIA and \(0.988\) in generated-text detection, compared with \(0.802\) and \(0.937\) for the full-context cells. For LAR-2, the corresponding AUCs are \(0.952\) and \(0.989\), compared with \(0.893\) and \(0.945\). In both models, the shorter-context results are close to those from the complete array. These comparisons provide evidence that shorter-context information contributes in both applications. The difference between LAR-1 and LAR-2 remains pronounced for WikiMIA but is small for generated-text detection, consistent with the task-dependent contribution of the within-path second-order augmentation in Tables~\ref{tab:wikimia} and~\ref{tab:aidetect}.

To localize the first-order contributions over context scale and text position, Figure~\ref{fig:functional_contrasts} displays coefficient contrasts for three representative channels.

\begin{figure}[ht!]
\centering
\begingroup
\small
\renewcommand{\baselinestretch}{1}\selectfont
\includegraphics[width=0.96\textwidth]{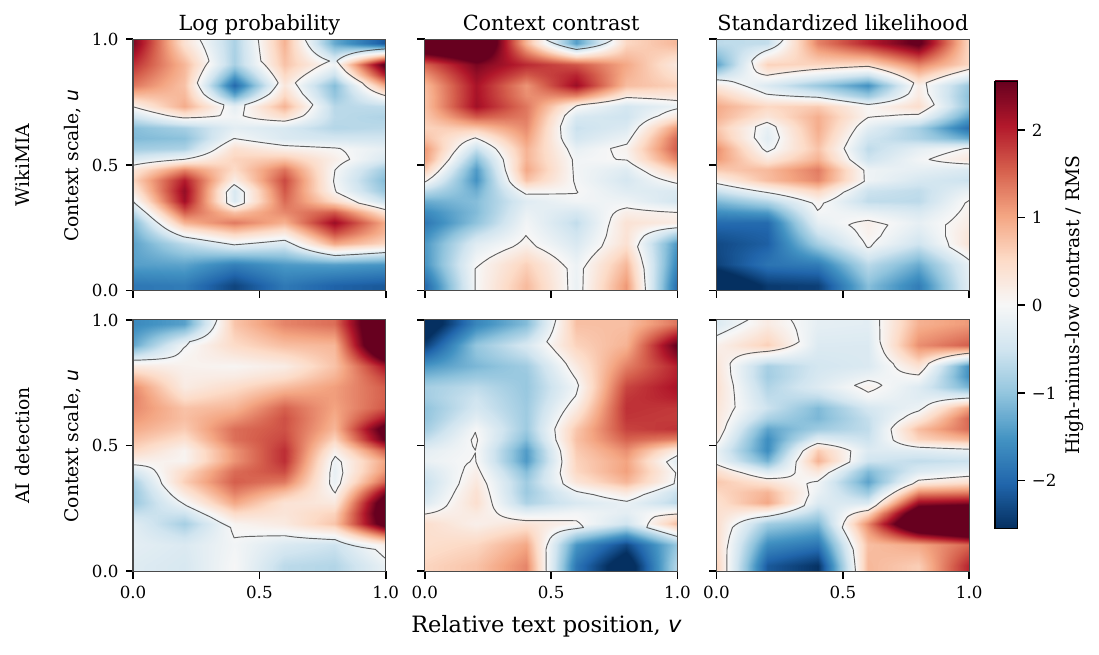}
\caption{Estimated LAR-1 coefficient contrasts for three representative channels. Each panel shows the change in the fitted coefficient function when the indicated channel is moved from its 20th to its 80th percentile, with the remaining channels fixed at their medians. The surfaces are averaged over the five fold-specific fits and divided by their root mean square within each panel; color therefore shows location and sign, not relative importance across channels. Positive values favor membership in the WikiMIA analysis (top) and AI-generated text in the generated-text analysis (bottom). The scoring models are Pythia-1B and DeepSeek-7B, respectively, and the gray contours mark zero.}
\label{fig:functional_contrasts}
\endgroup
\end{figure}

The contrasts vary over both context scale and text position and change sign across the array. Thus, positive and negative regional contributions can partly cancel when likelihood information is reduced to a global summary. The spatial patterns also differ between the two detection tasks. In both settings, substantial variation occurs away from the full-context cells \(u=1\), consistent with the restriction results in Figure~\ref{fig:sieve_mechanism}. The complete set of five channel contrasts is reported in the Supplement.

\subsection{Effect of sample size}
\label{sec:sample_size_results}

We examine how detection performance changes with the number of labeled texts. For each application, we hold out a fixed test set and fit the full-context-only model, LAR-1, and LAR-2 to repeated subsamples of the remaining labeled texts. All fits use the same basis and ridge specification. At each sample size, Figure~\ref{fig:sieve_sample_size} reports the median test AUC over 20 subsamples, together with the 10th and 90th percentiles.

\begin{figure}[H]
\centering
\begingroup
\small
\renewcommand{\baselinestretch}{1}\selectfont
\includegraphics[width=0.90\textwidth]{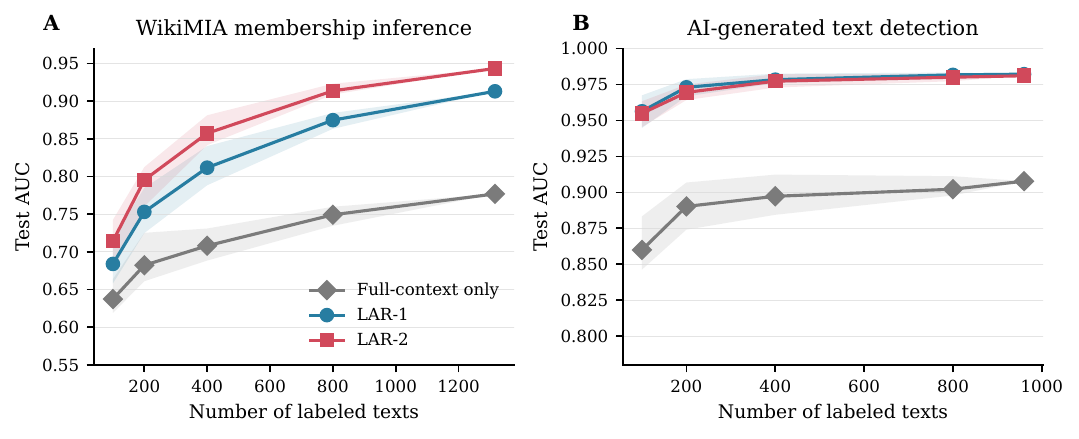}
\caption{Effect of labeled sample size on detection performance. The full-context-only model, LAR-1, and LAR-2 are evaluated on the same held-out test set and use the same basis and ridge specification. Curves show the median test AUC over 20 training subsamples, and shaded bands show the 10th and 90th percentiles.}
\label{fig:sieve_sample_size}
\endgroup
\end{figure}

With 100 labeled texts, the median WikiMIA AUCs are \(0.638\), \(0.684\), and \(0.714\) for the full-context-only model, LAR-1, and LAR-2, respectively. The corresponding values for generated-text detection are \(0.860\), \(0.956\), and \(0.955\). Thus, the shorter-context likelihood-array information improves ranking even at the smallest labeled sample size considered. LAR-2 provides an additional improvement over LAR-1 in WikiMIA, and the difference generally becomes more pronounced as the labeled sample size increases. In generated-text detection, the two LAR estimators remain nearly indistinguishable across sample sizes, consistent with the full-sample results.

Sensitivity to the numerical specification and random-feature seed is reported in the Supplementary Material.

\section{Discussion}
\label{sec:discussion}

This paper develops likelihood-array regression for membership inference and AI-generated text detection. By evaluating each target token under nested context windows, LAR preserves variation across context scale, target position, and likelihood-derived channels that is discarded by conventional full-context summaries. In the numerical studies, shorter-context cells provide additional ranking information in both applications, while the within-path second-order augmentation substantially improves membership inference. Under product-form covariance decay, the theoretical analysis establishes matching minimax prediction bounds for the within-path quadratic model and characterizes the additional approximation costs of finite path sieves and squared random projections.

Practical use raises several directions for further work. Threshold-based decisions require calibration beyond ranking performance, and constructing the likelihood array is more computationally demanding than full-context scoring. Adaptive selection or subsampling of context windows may reduce this cost while preserving the relevant likelihood information. The framework could also be extended to combine multiple scoring models and to localize membership or generated-text evidence within a document.

\section*{Data availability statement}

The data that support the findings of this study are publicly available.
The WikiMIA benchmark is available from its
\href{https://huggingface.co/datasets/swj0419/WikiMIA}{Hugging Face repository},
and the data used for AI-generated text detection are available from the
\href{https://github.com/Mamba413/AdaDetectGPT}{AdaDetectGPT repository}.

\bibliographystyle{apalike}
\begingroup
\small
\renewcommand{\baselinestretch}{1.2}\selectfont
\bibliography{Bibliography-MM-MC}
\endgroup

\clearpage
\spacingset{1}

\setcounter{section}{0}
\setcounter{equation}{0}
\setcounter{figure}{0}
\setcounter{table}{0}
\setcounter{lemma}{0}
\renewcommand{\thesection}{S\arabic{section}}
\renewcommand{\theequation}{S\arabic{equation}}
\renewcommand{\thefigure}{S\arabic{figure}}
\renewcommand{\thetable}{S\arabic{table}}
\renewcommand{\thelemma}{S\arabic{lemma}}
\renewcommand{\theHsection}{supp.\arabic{section}}
\renewcommand{\theHequation}{supp.\arabic{equation}}
\renewcommand{\theHfigure}{supp.\arabic{figure}}
\renewcommand{\theHtable}{supp.\arabic{table}}

\begin{center}
{\LARGE\bf Supplementary Material for\\ Token-Level Likelihood-Array Regression for Membership Inference and AI-Generated Text Detection}
\end{center}
\medskip

\section{Proofs}
\label{sec:supp_setup}

\subsection{Notation and preliminary covariance bound}

We first restate the objects used throughout the proofs.  Let \(\mathcal X_i\) collect the complete likelihood-array data for text \(i\), including all of its context paths. For target token \(t\), let \(W_{it}\in\mathcal H\) be its Hilbert-space path feature and set \(v_{it}=(t-2)/(T_i-2)\). The vector \(\boldsymbol{\varphi}_v(v_{it})\in\R^{d_v}\) contains the target-position basis functions, where \(d_v\) is fixed in the rate analysis. Define the centered text-level features
\begin{equation}
\begin{split}
  \Phi_i^{(1)}
  &=
  \frac{1}{T_i-1}\sum_{t=2}^{T_i}
  \boldsymbol{\varphi}_v(v_{it})\otimes W_{it}
  -
  \E\left\{
  \frac{1}{T_i-1}\sum_{t=2}^{T_i}
  \boldsymbol{\varphi}_v(v_{it})\otimes W_{it}
  \right\},\\
  \Phi_i^{(2)}
  &=
  \frac{1}{T_i-1}\sum_{t=2}^{T_i}
  \boldsymbol{\varphi}_v(v_{it})\otimes(W_{it}\otimes_s W_{it})
  -
  \E\left\{
  \frac{1}{T_i-1}\sum_{t=2}^{T_i}
  \boldsymbol{\varphi}_v(v_{it})\otimes(W_{it}\otimes_s W_{it})
  \right\}.
\end{split}
\label{eq:supp_text_features}
\end{equation}
Here \(\otimes_s\) denotes the symmetric tensor product. Put \(\Phi_i=(\Phi_i^{(1)},\Phi_i^{(2)})\), which lies in \((\R^{d_v}\otimes\mathcal H)\oplus\{\R^{d_v}\otimes\mathcal H_s^{\otimes2}\}\). The conditional response model is
\begin{equation*}
  \Pp(Y_i=1\mid\Phi_i)
  =
  \sigma\{\alpha_0+\ip{\vartheta_0}{\Phi_i}\},
  \qquad
  \sigma(x)=\{1+\exp(-x)\}^{-1},
\end{equation*}
where \(|\alpha_0|\leq M\) and \(\norm{\vartheta_0}\leq M\). Let \(P_\Phi\) denote the marginal distribution of \(\Phi\), and write \(\norm{g}_{P_\Phi}^2=\E\{g(\Phi)^2\}\). Because \(\Phi_i\) is centered, the squared prediction norm of a perturbation \((a,h)\) is \(a^2+\E\langle h,\Phi_i\rangle^2\). We retain the intercept in the proof of Theorem~2 and show explicitly where its parametric \(n^{-1}\) contribution enters.

Assumption~1 of the main paper gives \(\norm{\psi(u,\boldsymbol h)}\leq K_\psi\) and \(\norm{\boldsymbol{\varphi}_v(v)}\leq K_v\) over their respective domains. Because \(W_{it}\) is an average of the values \(\psi(u_g,\boldsymbol h_{itg})\), it follows that \(\sup_{i,t}\norm{W_{it}}\leq K_\psi\). Because each feature in \eqref{eq:supp_text_features} is in turn an average over target tokens, these bounds imply that there is a deterministic constant \(K_\Phi<\infty\) such that \(\|\Phi\|\leq K_\Phi\) almost surely, independently of text length. No independence among the paths within
\(\mathcal X_i\) is used.

Let \(\mathcal S\) be the positive trace-class operator from Assumption~2, with eigenvalues
\begin{equation}
  \kappa_j\leq C_\kappa j^{-\nu},
  \qquad \nu>1.
\label{eq:supp_eigen_upper}
\end{equation}
The two blockwise covariance bounds in Assumption~2 are
\begin{equation}
\begin{split}
  \operatorname{cov}(\Phi^{(1)})
  &\preceq
  C_\Gamma(I_{d_v}\otimes \mathcal S),\\
\operatorname{cov}(\Phi^{(2)})
  &\preceq
  C_\Gamma\{I_{d_v}\otimes(\mathcal S\otimes_s \mathcal S)\}.
\end{split}
\label{eq:supp_block_capacity}
\end{equation}
Here \(A\preceq B\) means that \(B-A\) is positive semidefinite. The condition \(\nu>1\) implies \(\sum_j\kappa_j<\infty\), so \(\mathcal S\) has finite trace.

Let \(\Gamma=\operatorname{cov}(\Phi)\) be the covariance operator of the complete first- and second-order feature. For \(h=(h_1,h_2)\), the inequality \((a+b)^2\leq2a^2+2b^2\) and \eqref{eq:supp_block_capacity} give
\begin{equation}
\Gamma
\preceq
2C_\Gamma
\left[
(I_{d_v}\otimes\mathcal S)
\oplus
\{I_{d_v}\otimes(\mathcal S\otimes_s\mathcal S)\}
\right].
\label{eq:supp_text_capacity}
\end{equation}
This bound includes any cross-covariance between the first- and second-order blocks.

Throughout the proofs, \(\mathcal P_\nu(M)\) denotes the collection of joint laws of the induced path features and labels that satisfy the bounded-feature condition, the conditional logistic model above, and \eqref{eq:supp_eigen_upper}--\eqref{eq:supp_block_capacity}, with \(|\alpha_0|\leq M\), \(\norm{\vartheta_0}\leq M\), and all fixed constants bounded uniformly over the collection. A fixed scoring model, channel map, and feature map select one member of this class. This is the model class defined in the main paper.

The upper bounds use the one-sided decay in \eqref{eq:supp_eigen_upper}. For the proof of Theorem~1, we construct a submodel on which
\begin{equation}
  c_\kappa j^{-\nu}
  \leq
  \kappa_j
  \leq
  C_\kappa j^{-\nu}.
\label{eq:supp_eigen_two_sided}
\end{equation}

The following lemma verifies the sufficient path-level condition stated after Assumption~2 in the main paper. Conditional on \(\mathcal X_i\), draw \(I\) uniformly from \(\{2,\ldots,T_i\}\), and define \(\widetilde{\Phi}^{(1)}=\boldsymbol{\varphi}_v(v_{iI})\otimes W_{iI}\) and \(\widetilde{\Phi}^{(2)}=\boldsymbol{\varphi}_v(v_{iI})\otimes(W_{iI}\otimes_s W_{iI})\).

\begin{lemma}[From path-level to text-level covariance]
\label{lem:supp_contraction}
If
\begin{equation*}
\operatorname{cov}(\widetilde{\Phi}^{(1)})\preceq C_\Gamma(I_{d_v}\otimes\mathcal S),
\qquad
\operatorname{cov}(\widetilde{\Phi}^{(2)})\preceq C_\Gamma\{I_{d_v}\otimes(\mathcal S\otimes_s\mathcal S)\},
\end{equation*}
then the two bounds in \eqref{eq:supp_block_capacity} hold with the same constant \(C_\Gamma\). No independence among paths within a text is required.
\end{lemma}

\begin{proof}
The uniform selection of \(I\) gives
\[
\E(\widetilde{\Phi}^{(k)}\mid\mathcal X_i)=\frac{1}{T_i-1}\sum_{t=2}^{T_i}\widetilde{\Phi}_{it}^{(k)},
\qquad k=1,2,
\]
where \(\widetilde{\Phi}_{it}^{(1)}=\boldsymbol{\varphi}_v(v_{it})\otimes W_{it}\) and \(\widetilde{\Phi}_{it}^{(2)}=\boldsymbol{\varphi}_v(v_{it})\otimes(W_{it}\otimes_s W_{it})\). Hence, for \(k\in\{1,2\}\) and any \(h_k\) in the corresponding feature space, \eqref{eq:supp_text_features} implies
\[
\ip{h_k}{\Phi_i^{(k)}}
=
\E\left\{
\ip{h_k}{\widetilde{\Phi}^{(k)}-\E\widetilde{\Phi}^{(k)}}
\,\middle|\,\mathcal X_i
\right\}.
\]
Conditional Jensen's inequality and then the path-level covariance bound give
\begin{align*}
\operatorname{var}\{\ip{h_k}{\Phi_i^{(k)}}\}
&\leq
\operatorname{var}\{\ip{h_k}{\widetilde{\Phi}^{(k)}}\}\\
&\leq
C_\Gamma\ip{h_k}{A_kh_k},
\end{align*}
where \(A_1=I_{d_v}\otimes\mathcal S\) and \(A_2=I_{d_v}\otimes(\mathcal S\otimes_s\mathcal S)\). Since this holds for every \(h_k\), it proves both inequalities in \eqref{eq:supp_block_capacity}. Conditioning on the complete likelihood-array data permits arbitrary dependence among the constituent paths.
\end{proof}

\subsection{Spectrum and effective dimension}
\label{sec:supp_spectrum}

For a positive trace-class operator \(A\), define its effective dimension at regularization level \(\lambda\) by \(\mathcal N_A(\lambda)=\tr\{A(A+\lambda I)^{-1}\}\).

\begin{lemma}[Ordering the product spectrum]
\label{lem:supp_products}
Arrange \(\{\kappa_j\kappa_k:1\leq j<k\}\) in nonincreasing order as \(\rho_1\geq\rho_2\geq\cdots\). Under \eqref{eq:supp_eigen_upper},
\[
\rho_\ell
\leq
C\left\{\frac{\ell}{\log(e\ell)}\right\}^{-\nu},
\qquad \ell\geq1,
\]
and
\[
\sum_{\ell\geq1}\frac{\rho_\ell}{\rho_\ell+\lambda}
\leq
C\lambda^{-1/\nu}\log(e/\lambda),
\qquad 0<\lambda\leq1.
\]
If the two-sided condition \eqref{eq:supp_eigen_two_sided} holds, both bounds are sharp:
\begin{equation}
  \rho_\ell
  \asymp
  \left\{
  \frac{\ell}{\log(e\ell)}
  \right\}^{-\nu},
\label{eq:supp_ordered_products}
\end{equation}
and
\begin{equation}
  \sum_{\ell\geq1}
  \frac{\rho_\ell}{\rho_\ell+\lambda}
  \asymp
  \lambda^{-1/\nu}\log(e/\lambda),
  \qquad 0<\lambda\leq1.
\label{eq:supp_product_dimension}
\end{equation}
\end{lemma}

\begin{proof}
The proof has three parts.

\emph{Step 1: count the pairs.} For \(x\geq4\), let
\[
D(x)=\#\{(j,k):1\leq j<k,\ jk\leq x\}.
\]
The restriction \(j<k\) and \(jk\leq x\) implies \(j<\sqrt x\). Hence
\begin{align*}
D(x)
&\leq
\sum_{1\leq j<\sqrt x}\left\lfloor\frac{x}{j}\right\rfloor
\leq
x\sum_{1\leq j<\sqrt x}\frac1j
\leq
C x\log x.
\end{align*}
For the reverse inequality, restrict the count to \(1\leq j\leq\sqrt x/4\) and \(2j\leq k\leq x/j\). The restriction on \(j\) gives
\[
\frac{2j}{x/j}=\frac{2j^2}{x}\leq\frac18,
\qquad\text{and hence}\qquad
\frac{x}{j}-2j\geq\frac{7x}{8j}.
\]
After allowing for the two integer endpoints, the interval \([2j,x/j]\) therefore contains at least \(x/(2j)\) integers for all sufficiently large \(x\). Every such integer satisfies \(j<k\) and \(jk\leq x\). Consequently,
\[
D(x)
\geq
\sum_{1\leq j\leq\sqrt x/4}\frac{x}{2j}
\geq
c x\log x.
\]
Thus, \(D(x)\) is bounded above and below by fixed multiples of \(x\log(ex)\).

\emph{Step 2: order the products.} For \(z>0\), let
\[
N(z)=\#\{(j,k):j<k,\ \kappa_j\kappa_k\geq z\}
\]
be the number of product eigenvalues at least \(z\). Under \eqref{eq:supp_eigen_upper}, \(\kappa_j\kappa_k\leq C(jk)^{-\nu}\). Thus, \(\kappa_j\kappa_k\geq z\) can hold only if \(jk\leq Cz^{-1/\nu}\), and for \(0<z\leq1\),
\begin{align*}
N(z)
&\leq
D\{Cz^{-1/\nu}\}\\
&\leq
Cz^{-1/\nu}\log(e/z).
\end{align*}
Under the two-sided condition \eqref{eq:supp_eigen_two_sided}, every pair satisfying \(jk\leq cz^{-1/\nu}\) has \(\kappa_j\kappa_k\geq z\). The lower bound for \(D\) then gives
\[
N(z)\geq cz^{-1/\nu}\log(e/z).
\]
We next invert these counting bounds. By definition of the decreasing rearrangement, \(\rho_\ell\geq z\) if and only if \(N(z)\geq\ell\). Put \(x=z^{-1/\nu}\). Apart from fixed constants, the counting relation is
\[
\ell\asymp x\log(ex).
\]
Its inverse has order \(x_\ell=\ell/\log(e\ell)\). Indeed,
\[
\log(e x_\ell)
=
\log\left\{\frac{e\ell}{\log(e\ell)}\right\}
\asymp
\log(e\ell),
\]
and hence \(x_\ell\log(e x_\ell)\asymp\ell\). Since \(z=x^{-\nu}\), the level associated with rank \(\ell\) is therefore
\[
z_\ell
\asymp
x_\ell^{-\nu}
=
\left\{\frac{\ell}{\log(e\ell)}\right\}^{-\nu}.
\]
Using only the upper counting bound yields the stated upper bound for \(\rho_\ell\); using both counting bounds yields \eqref{eq:supp_ordered_products}.

\emph{Step 3: sum the regularized eigenvalues.} Put
\[
L_\lambda
=
\left\lceil C_0\lambda^{-1/\nu}\log(e/\lambda)\right\rceil,
\]
where \(C_0\) is a sufficiently large fixed constant. To verify the purpose of this choice, observe that
\[
\log(eL_\lambda)\asymp\log(e/\lambda)
\quad\text{and}\quad
\frac{L_\lambda}{\log(eL_\lambda)}
\asymp
C_0\lambda^{-1/\nu}.
\]
The upper bound for the ordered products then gives
\[
\rho_{L_\lambda}
\leq
C\left\{\frac{L_\lambda}{\log(eL_\lambda)}\right\}^{-\nu}
\leq
CC_0^{-\nu}\lambda.
\]
Choosing \(C_0\) large enough makes the last expression at most \(\lambda\). Since \(\rho_\ell\) is nonincreasing, \(\rho_\ell\leq\lambda\) for every \(\ell>L_\lambda\). We may therefore use \(x/(x+\lambda)\leq\min(1,x/\lambda)\) and split the sum at \(L_\lambda\):
\begin{align*}
\sum_{\ell\geq1}\frac{\rho_\ell}{\rho_\ell+\lambda}
&\leq
L_\lambda
+
\lambda^{-1}\sum_{\ell>L_\lambda}\rho_\ell\\
&\leq
L_\lambda
+
\frac{C}{\lambda}
\sum_{\ell>L_\lambda}
\left\{\frac{\log(e\ell)}{\ell}\right\}^{\nu}.
\end{align*}
Because \(\nu>1\), comparison with the integral of \(x^{-\nu}\{\log(ex)\}^\nu\) gives
\[
\sum_{\ell>L}\left\{\frac{\log(e\ell)}{\ell}\right\}^{\nu}
\leq
C L^{1-\nu}\{\log(eL)\}^\nu.
\]
For clarity, write \(q_\lambda=\log(e/\lambda)\). Since \(L_\lambda\asymp\lambda^{-1/\nu}q_\lambda\) and \(\log(eL_\lambda)\asymp q_\lambda\), the tail term has order
\begin{align*}
\lambda^{-1}L_\lambda^{1-\nu}\{\log(eL_\lambda)\}^\nu
&\asymp
\lambda^{-1}
\left(\lambda^{-1/\nu}q_\lambda\right)^{1-\nu}
q_\lambda^\nu\\
&=
\lambda^{-1/\nu}q_\lambda
\asymp L_\lambda.
\end{align*}
The first \(L_\lambda\) summands and the tail are therefore both bounded by a constant multiple of \(L_\lambda\), proving the upper bound in \eqref{eq:supp_product_dimension}.

Under \eqref{eq:supp_eigen_two_sided}, the lower bound in \eqref{eq:supp_ordered_products} implies that, for a sufficiently small fixed \(c_0>0\),
\[
\rho_\ell\geq c_1\lambda,
\qquad 1\leq\ell\leq c_0L_\lambda,
\]
where \(c_1>0\) does not depend on \(\lambda\). Each of these summands is at least \(c_1/(1+c_1)\). There are order \(L_\lambda\) such terms, which proves the matching lower bound in \eqref{eq:supp_product_dimension}.
\end{proof}

\begin{lemma}[Effective dimension]
\label{lem:supp_effective_dimension}
Under \eqref{eq:supp_text_capacity} and \eqref{eq:supp_eigen_upper},
\begin{equation}
  \mathcal N_\Gamma(\lambda)
  \leq
  C\lambda^{-1/\nu}\log(e/\lambda),
  \qquad 0<\lambda\leq1.
\label{eq:supp_effective_dimension}
\end{equation}
\end{lemma}

\begin{proof}
If \(A\) has eigenvalues \(a_1\geq a_2\geq\cdots\geq0\), then
\[
\mathcal N_A(\lambda)
=
\sum_{j\geq1}\frac{a_j}{a_j+\lambda}.
\]
The scalar function \(x\mapsto x/(x+\lambda)\) is increasing on \([0,\infty)\). The min--max characterization of eigenvalues therefore implies that \(0\preceq A\preceq CB\) gives
\[
\mathcal N_A(\lambda)
\leq
\mathcal N_{CB}(\lambda)
=
\mathcal N_B(\lambda/C).
\]
Apply this relation to \eqref{eq:supp_text_capacity}. The position basis produces only the fixed multiplicity \(d_v\), which is absorbed into the constant below.

For the first-order block, let \(J_\lambda=\lceil(C_\kappa/\lambda)^{1/\nu}\rceil\). Splitting the sum at \(J_\lambda\) gives
\begin{align*}
\sum_{j\geq1}\frac{\kappa_j}{\kappa_j+\lambda}
&\leq
J_\lambda
+
\lambda^{-1}\sum_{j>J_\lambda}\kappa_j\\
&\leq
C\lambda^{-1/\nu}
+
\frac{C}{\lambda}\int_{J_\lambda}^{\infty}x^{-\nu}\,dx
\\
&=
C\lambda^{-1/\nu}
+
\frac{C}{(\nu-1)\lambda}J_\lambda^{1-\nu}
\leq
C\lambda^{-1/\nu},
\end{align*}
where the last inequality uses \(J_\lambda\asymp\lambda^{-1/\nu}\), and hence \(\lambda^{-1}J_\lambda^{1-\nu}\asymp\lambda^{-1/\nu}\).

The symmetric tensor block has off-diagonal eigenvalues \(\kappa_j\kappa_k\), \(j<k\), and diagonal eigenvalues of order \(\kappa_j^2\). Lemma~\ref{lem:supp_products} bounds the off-diagonal contribution by \(C\lambda^{-1/\nu}\log(e/\lambda)\). For the diagonal contribution, \(\kappa_j^2\leq Cj^{-2\nu}\). Let \(K_\lambda=\lceil(C/\lambda)^{1/(2\nu)}\rceil\). Splitting the diagonal sum at \(K_\lambda\) gives
\begin{align*}
\sum_{j\geq1}\frac{\kappa_j^2}{\kappa_j^2+\lambda}
&\leq
K_\lambda
+
\lambda^{-1}\sum_{j>K_\lambda}\kappa_j^2\\
&\leq
C\lambda^{-1/(2\nu)}
+
\frac{C}{\lambda}
\int_{K_\lambda}^{\infty}x^{-2\nu}\,dx\\
&=
C\lambda^{-1/(2\nu)}
+
\frac{C}{(2\nu-1)\lambda}K_\lambda^{1-2\nu}\\
&\leq
C\lambda^{-1/(2\nu)}.
\end{align*}
The final line uses \(K_\lambda\asymp\lambda^{-1/(2\nu)}\), so \(\lambda^{-1}K_\lambda^{1-2\nu}\asymp\lambda^{-1/(2\nu)}\).
For \(0<\lambda\leq1\),
\[
\lambda^{-1/(2\nu)}
\leq
\lambda^{-1/\nu}
\leq
\lambda^{-1/\nu}\log(e/\lambda).
\]
Thus the diagonal and first-order contributions are no larger than the off-diagonal order. Adding the three contributions proves \eqref{eq:supp_effective_dimension}.
\end{proof}

\subsection{Proof of Theorem 1}
\label{sec:supp_lower}

\begin{proof}
The proof constructs a finite hypercube inside \(\mathcal P_\nu(M)\) and then applies Assouad's argument. We give each step explicitly.

\emph{Step 1: construct an admissible path distribution.} Let \(e_1,e_2,\ldots\) be an orthonormal basis of \(\mathcal H\), let \(\xi_1,\xi_2,\ldots\) be independent Rademacher variables, and choose \(\kappa_j\) satisfying \eqref{eq:supp_eigen_two_sided}. Define
\begin{equation*}
W=\sum_{j\geq1}\sqrt{\kappa_j}\xi_j e_j,
\qquad
Q=W\otimes_s W-\E(W\otimes_s W).
\end{equation*}
Since \(\nu>1\), \(\sum_j\kappa_j<\infty\). Moreover,
\[
\norm{W}^2=\sum_{j\geq1}\kappa_j
\]
for every realization of the Rademacher sequence. Thus \(W\) is uniformly bounded. The norm of a rank-one tensor satisfies \(\|W\otimes_s W\|=\|W\|^2\), so
\[
\norm{Q}
\leq
\norm{W\otimes_s W}
+
\norm{\E(W\otimes_s W)}
\leq
2\norm{W}^2.
\]
Hence \(Q\) is also uniformly bounded. Fix a text length \(T_0>2\), let every target-token path in a text share the same random feature \(W\), and write
\[
\overline{\boldsymbol{\varphi}}_v
=
\frac{1}{T_0-1}\sum_{t=2}^{T_0}\boldsymbol{\varphi}_v\{v(t;T_0)\}.
\]
The centered text-level features are then \(\Phi^{(1)}=\overline{\boldsymbol{\varphi}}_v\otimes W\) and \(\Phi^{(2)}=\overline{\boldsymbol{\varphi}}_v\otimes Q\). The position basis contains the constant function, so there is a fixed vector \(\boldsymbol a_0\in\R^{d_v}\) such that \(\boldsymbol a_0^\prime\boldsymbol{\varphi}_v(v)=1\) for every \(v\), and consequently \(\boldsymbol a_0^\prime\overline{\boldsymbol{\varphi}}_v=1\). A quadratic coefficient of the form \(\boldsymbol a_0\otimes\Theta\) therefore gives the score \(\langle\Theta,Q\rangle\). We set the first-order coefficient to zero. The fixed norms of \(\boldsymbol a_0\) and \(\overline{\boldsymbol{\varphi}}_v\) are absorbed into the constants below.

For \(j<k\), define the orthonormal symmetric tensor
\[
E_{jk}=\frac{e_j\otimes e_k+e_k\otimes e_j}{\sqrt2}.
\]
The diagonal terms of \(W\otimes_s W\) are deterministic because \(\xi_j^2=1\), and they disappear after centering. Expanding the off-diagonal terms gives
\begin{equation*}
Q
=
\sum_{j<k}\sqrt{2\kappa_j\kappa_k}\,\xi_j\xi_k E_{jk}.
\end{equation*}
For two distinct unordered pairs \(\{j,k\}\neq\{j',k'\}\), at least one index appears exactly once in the product \(\xi_j\xi_k\xi_{j'}\xi_{k'}\). Independence and \(\E\xi_j=0\) therefore imply
\[
\E(\xi_j\xi_k\xi_{j'}\xi_{k'})=0.
\]
For the same pair, \(\E(\xi_j\xi_k)^2=1\). Hence the coordinates \(\xi_j\xi_k\) are orthonormal in \(L_2\), and
\[
\E\ip{Q}{E_{jk}}^2
=
2\kappa_j\kappa_k.
\]
Thus \(E_{jk}\) is a covariance eigenvector with eigenvalue \(2\kappa_j\kappa_k\). It follows that
\[
\operatorname{cov}(\Phi^{(1)})
=
(\overline{\boldsymbol{\varphi}}_v\,\overline{\boldsymbol{\varphi}}_v^{\prime})\otimes\mathcal S
\preceq
\norm{\overline{\boldsymbol{\varphi}}_v}^2(I_{d_v}\otimes\mathcal S)
\]
and
\[
\operatorname{cov}(\Phi^{(2)})
\preceq
2\norm{\overline{\boldsymbol{\varphi}}_v}^2
\{I_{d_v}\otimes(\mathcal S\otimes_s\mathcal S)\}.
\]
The submodel therefore satisfies Assumption~2 with a fixed covariance constant. This verification is made at the text level and does not require the individual paths within a text to be independent.

Arrange the off-diagonal covariance eigenvalues in decreasing order as \(\rho_1\geq\rho_2\geq\cdots\), and denote the corresponding tensor directions by \(E_1,E_2,\ldots\). The factor 2 does not affect their order. Lemma~\ref{lem:supp_products} gives
\[
\rho_\ell
\asymp
\left\{\frac{\ell}{\log(e\ell)}\right\}^{-\nu}.
\]

\emph{Step 2: choose the active block.} Let
\[
b_n
=
\left\lfloor
c_b n^{1/(\nu+1)}\{\log(en)\}^{\nu/(\nu+1)}
\right\rfloor,
\]
where \(c_b>0\) is fixed and sufficiently small. Put \(L_n=\log(en)\). The integer part does not affect any of the following orders, so
\[
b_n\asymp n^{1/(\nu+1)}L_n^{\nu/(\nu+1)}.
\]
Moreover,
\[
\log b_n
=
\frac{1}{\nu+1}\log n
+
\frac{\nu}{\nu+1}\log L_n
+O(1).
\]
Because \(\log L_n=o(\log n)\), it follows that \(\log(eb_n)\asymp L_n\). Consequently,
\begin{align*}
\frac{b_n}{\log(eb_n)}
&\asymp
\frac{n^{1/(\nu+1)}L_n^{\nu/(\nu+1)}}{L_n}\\
&=
n^{1/(\nu+1)}L_n^{-1/(\nu+1)}
=
\left(\frac{n}{L_n}\right)^{1/(\nu+1)}.
\end{align*}
Substitution into the ordered-eigenvalue formula gives
\begin{align*}
\rho_{b_n}
&\asymp
\left\{\frac{b_n}{\log(eb_n)}\right\}^{-\nu}\\
&\asymp
\left(\frac{L_n}{n}\right)^{\nu/(\nu+1)}.
\end{align*}
The definition of \(b_n\) also gives
\[
\frac{b_n}{n}
\asymp
n^{-\nu/(\nu+1)}L_n^{\nu/(\nu+1)}
=
\left(\frac{L_n}{n}\right)^{\nu/(\nu+1)}.
\]
Thus \(n\rho_{b_n}\asymp b_n\). Replacing \(b_n\) by \(2b_n\) changes \(b_n/\log(eb_n)\) only by a fixed multiplicative factor, and hence \(\rho_{2b_n}\asymp\rho_{b_n}\). We have therefore established
\begin{equation}
\rho_{2b_n}\asymp\rho_{b_n},
\qquad
n\rho_{b_n}\asymp b_n,
\qquad
\frac{b_n}{n}
\asymp
\left\{\frac{\log(en)}{n}\right\}^{\nu/(\nu+1)}.
\label{eq:supp_lower_balance}
\end{equation}
\emph{Step 3: define the hypercube of logistic models.} For \(\omega=(\omega_1,\ldots,\omega_{b_n})\in\{-1,1\}^{b_n}\), set
\begin{equation*}
\Theta_\omega
=
h\sum_{\ell=1}^{b_n}\omega_\ell E_{b_n+\ell},
\qquad
h^2=\frac{\delta}{n\rho_{b_n}},
\end{equation*}
where \(\delta>0\) is a sufficiently small fixed constant. This choice balances separation and indistinguishability. Changing one sign produces squared prediction separation of order \(h^2\rho_{b_n}\) for one text; over \(n\) independent texts, the corresponding Kullback--Leibler divergence is of order \(nh^2\rho_{b_n}=\delta\). Give every model the same marginal distribution of \(W\), and define its conditional response law by
\begin{equation*}
\Pp_\omega(Y=1\mid W)
=
\sigma\{\eta_\omega(Q)\},
\qquad
\eta_\omega(Q)=\ip{\Theta_\omega}{Q}.
\end{equation*}
The squared coefficient norm is
\[
\norm{\Theta_\omega}^2=b_nh^2
=
\delta\frac{b_n}{n\rho_{b_n}}.
\]
By \eqref{eq:supp_lower_balance}, the ratio on the right is bounded above and below by fixed constants. Choosing \(\delta\) sufficiently small therefore ensures \(\norm{\boldsymbol a_0\otimes\Theta_\omega}\leq M\) for every \(\omega\). Moreover,
\[
|\eta_\omega(Q)|
=
|\ip{\Theta_\omega}{Q}|
\leq
\norm{\Theta_\omega}\norm{Q},
\]
and both factors on the right are uniformly bounded. Thus the log odds are uniformly bounded and every \(\Pp_\omega\) belongs to \(\mathcal P_\nu(M)\).

\emph{Step 4: control the distance between adjacent experiments.} Let \(\omega^{(\ell)}\) be obtained from \(\omega\) by changing only the sign of coordinate \(\ell\). Write \(B(x)=\log(1+e^x)\). The Kullback--Leibler divergence between two Bernoulli distributions in canonical logistic form is the Bregman divergence of \(B\):
\[
\KL\{\operatorname{Bernoulli}(\sigma(x)),
       \operatorname{Bernoulli}(\sigma(y))\}
=B(y)-B(x)-B'(x)(y-x).
\]
The integral remainder in Taylor's theorem gives
\[
B(y)-B(x)-B'(x)(y-x)
=
(y-x)^2
\int_0^1
(1-t)B''\{x+t(y-x)\}\,dt.
\]
Because \(B''(z)=\sigma(z)\{1-\sigma(z)\}\leq1/4\) and \(\int_0^1(1-t)dt=1/2\),
\begin{equation}
\KL\{\operatorname{Bernoulli}(\sigma(x)),
       \operatorname{Bernoulli}(\sigma(y))\}
\leq\frac{(x-y)^2}{8}.
\label{eq:supp_logistic_kl}
\end{equation}
The adjacent log-odds functions differ by
\[
\eta_\omega(Q)-\eta_{\omega^{(\ell)}}(Q)
=2h\omega_\ell\ip{Q}{E_{b_n+\ell}}.
\]
The marginal law of \(Q\) is common to all alternatives, so conditioning on \(Q\), applying \eqref{eq:supp_logistic_kl}, and then using independence across the \(n\) texts gives
\begin{align*}
\KL(\Pp_\omega^n,\Pp_{\omega^{(\ell)}}^n)
&=n\E_Q\left[
\KL\{\Pp_\omega(Y\mid Q),\Pp_{\omega^{(\ell)}}(Y\mid Q)\}
\right]\\
&\leq
\frac{n}{8}(2h)^2
\E\ip{Q}{E_{b_n+\ell}}^2\\
&=
\frac{nh^2}{2}\rho_{b_n+\ell}
\leq
\frac{nh^2}{2}\rho_{b_n}
=\frac{\delta}{2}.
\end{align*}
Pinsker's inequality therefore gives
\begin{equation}
\operatorname{TV}(\Pp_\omega^n,\Pp_{\omega^{(\ell)}}^n)
\leq
\sqrt{\frac12\KL(\Pp_\omega^n,\Pp_{\omega^{(\ell)}}^n)}
\leq
\frac{\sqrt\delta}{2}.
\label{eq:supp_adjacent_tv}
\end{equation}

\emph{Step 5: convert testing error into prediction error.} Fix any estimator \(\widetilde{\eta}\) based on the labeled sample. Conditional on that sample, regard \(\widetilde{\eta}\) as a function of the complete feature \(\Phi\) and define its coefficient along the normalized quadratic direction \(E_{b_n+\ell}\) by
\[
\widehat a_\ell
=
\E_\Phi\left[
\widetilde{\eta}(\Phi)
\frac{\ip{Q}{E_{b_n+\ell}}}{\sqrt{\rho_{b_n+\ell}}}
\right],
\qquad
\widehat\omega_\ell=\operatorname{sign}(\widehat a_\ell),
\]
with an arbitrary convention when \(\widehat a_\ell=0\). To verify the normalization used here, define
\[
\varphi_\ell(Q)
=
\frac{\ip{Q}{E_{b_n+\ell}}}{\sqrt{\rho_{b_n+\ell}}}.
\]
The covariance calculation in Step 1 gives
\[
\E\varphi_\ell(Q)=0,
\qquad
\E\{\varphi_\ell(Q)\varphi_k(Q)\}
=
\mathbf 1(\ell=k).
\]
Thus \(\varphi_1,\ldots,\varphi_{b_n}\) are orthonormal in \(L_2(P_\Phi)\), and \(\widehat a_\ell=\langle\widetilde{\eta},\varphi_\ell\rangle_{L_2(P_\Phi)}\). This definition permits \(\widetilde{\eta}\) to use both the first- and second-order components of \(\Phi\). Under model \(\omega\),
\[
\ip{\eta_\omega}{\varphi_\ell}_{L_2(P_\Phi)}
=
h\omega_\ell\sqrt{\rho_{b_n+\ell}}.
\]
If \(\widehat\omega_\ell\neq\omega_\ell\), these two coefficients have opposite signs or the estimated coefficient is zero. Their difference therefore has magnitude at least \(h\sqrt{\rho_{b_n+\ell}}\). Bessel's inequality states that the squared norm of \(\widetilde{\eta}-\eta_\omega\) is at least the sum of its squared coefficients along any finite orthonormal family. Applying it to \(\{\varphi_\ell\}_{\ell=1}^{b_n}\) yields
\begin{align}
\norm{\widetilde{\eta}-\eta_\omega}_{P_\Phi}^2
&\geq
h^2\sum_{\ell=1}^{b_n}
\rho_{b_n+\ell}
\mathbf 1(\widehat\omega_\ell\neq\omega_\ell)\notag\\
&\geq
h^2\rho_{2b_n}
\sum_{\ell=1}^{b_n}
\mathbf 1(\widehat\omega_\ell\neq\omega_\ell).
\label{eq:supp_prediction_hamming}
\end{align}

Average \eqref{eq:supp_prediction_hamming} over the \(2^{b_n}\) values of \(\omega\). For a fixed coordinate \(\ell\), pair every \(\omega\) with \(\omega_\ell=1\) with \(\omega^{(\ell)}\), whose \(\ell\)th sign is \(-1\). Then
\begin{align*}
&2^{-b_n}\sum_\omega
\Pp_\omega(\widehat\omega_\ell\neq\omega_\ell)\\
&\quad=
2^{-b_n}
\sum_{\omega:\,\omega_\ell=1}
\left\{
\Pp_\omega(\widehat\omega_\ell=-1)
+
\Pp_{\omega^{(\ell)}}(\widehat\omega_\ell=1)
\right\}.
\end{align*}
For any test between two distributions, the sum of its two error probabilities is at least one minus their total variation distance. There are \(2^{b_n-1}\) adjacent pairs in the last sum. Hence \eqref{eq:supp_adjacent_tv} implies
\[
2^{-b_n}\sum_{\omega}
\Pp_\omega(\widehat\omega_\ell\neq\omega_\ell)
\geq
\frac12\left(1-\frac{\sqrt\delta}{2}\right).
\]
Choose \(\delta\leq1\), so the right side is bounded below by a positive constant. Summing over \(\ell\) gives
\begin{align*}
\inf_{\widetilde{\eta}}\sup_{\omega}
\E_\omega\norm{\widetilde{\eta}-\eta_\omega}_{P_\Phi}^2
&\geq
c b_n h^2\rho_{2b_n}\\
&=
c\frac{b_n}{n}\frac{\rho_{2b_n}}{\rho_{b_n}}\\
&\geq
c'\frac{b_n}{n}.
\end{align*}
The equality substitutes \(h^2=\delta/(n\rho_{b_n})\), with the fixed \(\delta\) absorbed into \(c\); the final inequality uses \(\rho_{2b_n}/\rho_{b_n}\asymp1\). Finally, \eqref{eq:supp_lower_balance} shows that \(b_n/n\asymp\{\log(en)/n\}^{\nu/(\nu+1)}\). Since the hypercube is a subset of \(\mathcal P_\nu(M)\), the same lower bound holds for the supremum over \(\mathcal P_\nu(M)\), proving Theorem~1.
\end{proof}

\subsection{Proof of Theorem 2}
\label{sec:supp_upper}

Write \(\ell_y(s)=\log(1+e^s)-ys\). The estimator in Theorem~2 is
\begin{equation}
(\widehat\alpha_\lambda,\widehat\vartheta_\lambda)
\in
\arg\min_{|\alpha|\leq M,\,\norm{\vartheta}\leq M}
\left\{
P_n\ell_Y\bigl(\alpha+\ip{\vartheta}{\Phi}\bigr)
+
\frac{\lambda}{2}\norm{\vartheta}^2
\right\},
\label{eq:supp_ridge_estimator}
\end{equation}
where \(P_n\) denotes the empirical average over the \(n\) independent texts. Let \(L(\alpha,\vartheta)=P\ell_Y\{\alpha+\langle\vartheta,\Phi\rangle\}\) denote population logistic risk.

\begin{lemma}[Bounded-score curvature]
\label{lem:supp_curvature}
There are constants \(0<c<C<\infty\), depending only on the feature and coefficient bounds, such that, for every \((\alpha,\vartheta)\) in the radius-constrained class,
\[
  c\left[
  (\alpha-\alpha_0)^2
  +\norm{\vartheta-\vartheta_0}_\Gamma^2
  \right]
  \leq
  L(\alpha,\vartheta)-L(\alpha_0,\vartheta_0)
  \leq
  C\left[
  (\alpha-\alpha_0)^2
  +\norm{\vartheta-\vartheta_0}_\Gamma^2
  \right],
\]
where \(\norm{h}_\Gamma^2=\E\langle h,\Phi\rangle^2\).
\end{lemma}

\begin{proof}
Let \(s_0=\alpha_0+\langle\vartheta_0,\Phi\rangle\), \(s=\alpha+\langle\vartheta,\Phi\rangle\), and \(B(s)=\log(1+e^s)\). Under the correctly specified logistic model, \(\E(Y\mid\Phi)=B'(s_0)=\sigma(s_0)\). Conditional on \(\Phi\),
\begin{align*}
\E\{\ell_Y(s)-\ell_Y(s_0)\mid\Phi\}
&=B(s)-B(s_0)-B'(s_0)(s-s_0).
\end{align*}
By Taylor's theorem, the right side equals \(B''(s_0+t(s-s_0))(s-s_0)^2/2\) for some \(t\in(0,1)\). By Assumption~1, there is a deterministic constant \(K_\Phi<\infty\) such that \(\|\Phi\|\leq K_\Phi\) almost surely. The radius constraints imply
\[
|s|
\leq
|\alpha|+\norm{\vartheta}\norm{\Phi}
\leq
M+MK_\Phi
=:B,
\]
and the same bound holds for \(s_0\). Every point between \(s_0\) and \(s\) therefore lies in \([-B,B]\). On this interval,
\[
0<\min_{|x|\leq B}\sigma(x)\{1-\sigma(x)\}
\leq B''(x)\leq\frac14.
\]
Writing \(m_B=\min_{|x|\leq B}\sigma(x)\{1-\sigma(x)\}>0\), we obtain, conditionally on \(\Phi\),
\[
\frac{m_B}{2}(s-s_0)^2
\leq
\E\{\ell_Y(s)-\ell_Y(s_0)\mid\Phi\}
\leq
\frac18(s-s_0)^2.
\]
Taking expectations bounds the excess risk above and below by fixed multiples of \(\E(s-s_0)^2\). Finally, \(\E\Phi=0\), so the intercept and slope cross term vanishes:
\[
\E(s-s_0)^2
=(\alpha-\alpha_0)^2
+\E\langle\vartheta-\vartheta_0,\Phi\rangle^2.
\]
This proves the lemma.
\end{proof}

\begin{lemma}[Rademacher contraction with a deterministic offset]
\label{lem:supp_offset_contraction}
Let \(\varepsilon_1,\ldots,\varepsilon_n\) be independent Rademacher variables.
For an index set \(\mathcal T\), let \(x_i(t)\in\R\), let
\(\phi_i:\R\to\R\) be one-Lipschitz with \(\phi_i(0)=0\), and let
\(\Omega:\mathcal T\to\R\) be deterministic. Whenever the suprema below are
measurable,
\[
\E_\varepsilon\sup_{t\in\mathcal T}
\left\{
\frac1n\sum_{i=1}^n\varepsilon_i\phi_i\{x_i(t)\}-\Omega(t)
\right\}
\leq
\E_\varepsilon\sup_{t\in\mathcal T}
\left\{
\frac1n\sum_{i=1}^n\varepsilon_i x_i(t)-\Omega(t)
\right\}.
\]
\end{lemma}

\begin{proof}
First suppose that \(\mathcal T\) is finite. Conditional on
\(\varepsilon_1,\ldots,\varepsilon_{n-1}\), collect all terms not involving
\(\varepsilon_n\), including \(-\Omega(t)\), into \(R_t\). Then
\begin{align*}
&\E_{\varepsilon_n}\sup_{t\in\mathcal T}
\left\{R_t+\frac{\varepsilon_n}{n}\phi_n\{x_n(t)\}\right\}\\
&\quad=
\frac12\sup_{t,s\in\mathcal T}
\left[
R_t+R_s+\frac1n\left\{\phi_n(x_n(t))-\phi_n(x_n(s))\right\}
\right]\\
&\quad\leq
\frac12\sup_{t,s\in\mathcal T}
\left\{
R_t+R_s+\frac1n|x_n(t)-x_n(s)|
\right\}\\
&\quad=
\E_{\varepsilon_n}\sup_{t\in\mathcal T}
\left\{R_t+\frac{\varepsilon_n}{n}x_n(t)\right\}.
\end{align*}
The inequality uses the Lipschitz property. The last equality follows by writing
the absolute value as the maximum of its two signs; the two resulting suprema
are identical after interchanging \(t\) and \(s\). Repeating this replacement
for coordinates \(n-1,\ldots,1\) proves the result for finite \(\mathcal T\).
The general case follows by approximation with finite subclasses, under the
stated measurability condition.
\end{proof}

\begin{lemma}[Localized empirical-process bound]
\label{lem:supp_process}
Let \(Z\) be a centered Hilbert-space feature with covariance operator \(\Gamma_Z\),
and let
\(f^\circ(Z)=\alpha^\circ+\ip{\theta^\circ}{Z}\) be any fixed score. For
\(a,b>0\),
\begin{equation}
\begin{aligned}
\E\sup_{\delta,h}\Big[&(P-P_n)
\left\{
\ell_Y\bigl(f^\circ(Z)+\delta+\ip{h}{Z}\bigr)
-\ell_Y\bigl(f^\circ(Z)\bigr)
\right\}\\
&-a\{\delta^2+\norm{h}_{\Gamma_Z}^2\}-b\norm{h}^2\Big]
\leq
\frac{C}{an}\{1+\mathcal N_{\Gamma_Z}(b/a)\}.
\end{aligned}
\label{eq:supp_process_bound}
\end{equation}
The supremum may be restricted to any subset of \(\R\oplus\mathcal H\). The
constant \(C\) is universal and the result does not require \(f^\circ\) to be the
true conditional log odds.
\end{lemma}

\begin{proof}
Let
\[
g_{\delta,h}(Y,Z)
=
\ell_Y\bigl(f^\circ(Z)+\delta+\langle h,Z\rangle\bigr)
-
\ell_Y\bigl(f^\circ(Z)\bigr).
\]
The derivative of \(\ell_y(s)\) is \(\sigma(s)-y\), whose absolute value is at
most one. Hence \(g_{\delta,h}\) is one-Lipschitz as a function of the scalar
increment \(\delta+\langle h,Z\rangle\), and \(g_{0,0}=0\).

For completeness, let \((Y_i',Z_i')_{i=1}^n\) be an independent copy of the sample
and let \(P_n'\) be its empirical measure. Conditional Jensen's inequality gives
\[
\E\sup_{\delta,h}
\{(P-P_n)g_{\delta,h}-q(\delta,h)\}
\leq
\E\sup_{\delta,h}
\{(P_n'-P_n)g_{\delta,h}-q(\delta,h)\},
\]
where
\(q(\delta,h)=a\{\delta^2+\norm{h}_{\Gamma_Z}^2\}+b\norm{h}^2\).
Introduce independent Rademacher variables \(\varepsilon_1,\ldots,\varepsilon_n\).
The joint distribution of each difference
\(g_{\delta,h}(Y_i',Z_i')-g_{\delta,h}(Y_i,Z_i)\) is symmetric, so the right
side is unchanged when the \(i\)th difference is multiplied by \(\varepsilon_i\).
Splitting the ghost-sample and observed-sample terms, assigning half of \(q\) to
each, and using that the two samples have the same distribution give
\begin{align*}
&\E\sup_{\delta,h}
\{(P_n'-P_n)g_{\delta,h}-q(\delta,h)\}\\
&\qquad\leq
2\E\sup_{\delta,h}
\left\{
\frac1n\sum_{i=1}^n\varepsilon_i
g_{\delta,h}(Y_i,Z_i)
-\frac12q(\delta,h)
\right\}.
\end{align*}
Conditional on the observed data, define
\[
\phi_i(x)
=
\ell_{Y_i}\{f^\circ(Z_i)+x\}-\ell_{Y_i}\{f^\circ(Z_i)\}.
\]
Each \(\phi_i\) is one-Lipschitz and satisfies \(\phi_i(0)=0\). Apply
Lemma~\ref{lem:supp_offset_contraction} with index \(t=(\delta,h)\),
\(x_i(t)=\delta+\langle h,Z_i\rangle\), and
\(\Omega(t)=q(\delta,h)/2\). Combining the resulting inequality with the
preceding symmetrization bound gives
\begin{align*}
&\E\sup_{\delta,h}
\left[
(P-P_n)g_{\delta,h}
  -a\{\delta^2+\norm{h}_{\Gamma_Z}^2\}-b\norm{h}^2
\right]\\
&\qquad\leq
2\E\sup_{\delta,h}
\left[
  V_0\delta+\ip{V}{h}
-\frac a2\delta^2
  -\frac a2\norm{h}_{\Gamma_Z}^2
  -\frac b2\norm{h}^2
\right],
\end{align*}
where
\[
V_0=\frac1n\sum_{i=1}^n\varepsilon_i,
\qquad
V=\frac1n\sum_{i=1}^n\varepsilon_i Z_i.
\]
The scalar and Hilbert-space suprema separate. Completing the square gives
\[
\sup_\delta\left(V_0\delta-\frac a2\delta^2\right)
=\frac{V_0^2}{2a}
\]
and, since \(b>0\),
\[
\sup_h\left[
  \ip{V}{h}-\frac12\ip{h}{(a\Gamma_Z+bI)h}
\right]
=
\frac12\ip{V}{(a\Gamma_Z+bI)^{-1}V}.
\]
For the second identity, put \(A=a\Gamma_Z+bI\). The condition \(b>0\) makes
\(A\) invertible, and
\begin{align*}
\ip{V}{h}-\frac12\ip{h}{Ah}
&=
-\frac12
\norm{A^{1/2}h-A^{-1/2}V}^2
+
\frac12\ip{V}{A^{-1}V}.
\end{align*}
The squared norm is minimized at \(h=A^{-1}V\), proving the displayed supremum.
Now \(\E V_0^2=1/n\), and independence of the Rademacher variables gives
\[
\E(V\otimes V)
=
\frac1{n^2}\sum_{i=1}^n\E(Z_i\otimes Z_i)
=
\frac{\Gamma_Z}{n}.
\]
Therefore,
\begin{align*}
\E\ip{V}{(a\Gamma_Z+bI)^{-1}V}
&=
\frac1n\tr\{\Gamma_Z(a\Gamma_Z+bI)^{-1}\}\\
&=
\frac1{an}\mathcal N_{\Gamma_Z}(b/a).
\end{align*}
To see the last equality directly, let \(\gamma_1,\gamma_2,\ldots\) be the
eigenvalues of \(\Gamma_Z\). Then
\begin{align*}
\tr\{\Gamma_Z(a\Gamma_Z+bI)^{-1}\}
&=
\sum_{j\geq1}\frac{\gamma_j}{a\gamma_j+b}\\
&=
\frac1a
\sum_{j\geq1}\frac{\gamma_j}{\gamma_j+b/a}
=
\frac1a\mathcal N_{\Gamma_Z}(b/a).
\end{align*}
Combining the scalar and Hilbert-space terms proves \eqref{eq:supp_process_bound}.
\end{proof}

\begin{proof}
Set \(\widehat\delta=\widehat\alpha_\lambda-\alpha_0\) and \(\widehat h=\widehat\vartheta_\lambda-\vartheta_0\). We proceed from the defining inequality of the estimator to the stated rate.

\emph{Step 1: basic inequality.} Since \((\alpha_0,\vartheta_0)\) belongs to the constrained parameter set, \eqref{eq:supp_ridge_estimator} gives
\begin{align*}
P_n\{\ell_{\widehat\alpha_\lambda,\widehat\vartheta_\lambda}
-\ell_{\alpha_0,\vartheta_0}\}
+\frac\lambda2
\left(\norm{\widehat\vartheta_\lambda}^2-\norm{\vartheta_0}^2\right)
\leq0,
\end{align*}
where \(\ell_{\alpha,\vartheta}=\ell_Y\{\alpha+\langle\vartheta,\Phi\rangle\}\). Let
\[
\Delta\ell
=
\ell_{\alpha_0+\widehat\delta,\vartheta_0+\widehat h}
-
\ell_{\alpha_0,\vartheta_0}.
\]
Adding \(P\Delta\ell-P\Delta\ell\) to the empirical inequality and moving the empirical-process term to the right gives
\[
P\Delta\ell
+
\frac\lambda2
\left(\norm{\vartheta_0+\widehat h}^2-\norm{\vartheta_0}^2\right)
\leq
(P-P_n)\Delta\ell.
\]
Lemma~\ref{lem:supp_curvature} bounds \(P\Delta\ell\) from below by \(c_0(\widehat\delta^2+\|\widehat h\|_\Gamma^2)\), and therefore yields
\begin{align}
c_0\left(\widehat\delta^2+\norm{\widehat h}_\Gamma^2\right)
+\frac\lambda2
\left(\norm{\vartheta_0+\widehat h}^2-\norm{\vartheta_0}^2\right)
\leq
(P-P_n)
\{\ell_{\alpha_0+\widehat\delta,\vartheta_0+\widehat h}
-\ell_{\alpha_0,\vartheta_0}\}.
\label{eq:supp_ridge_basic}
\end{align}
The penalty difference satisfies
\begin{align*}
\norm{\vartheta_0+\widehat h}^2-\norm{\vartheta_0}^2
&=\norm{\widehat h}^2+2\ip{\vartheta_0}{\widehat h}\\
&\geq\frac12\norm{\widehat h}^2-2\norm{\vartheta_0}^2,
\end{align*}
where the last line uses \(2|\langle a,b\rangle|\leq\norm{b}^2/2+2\norm{a}^2\). Since \(\norm{\vartheta_0}\leq M\), \eqref{eq:supp_ridge_basic} implies
\begin{align}
c_0\left(\widehat\delta^2+\norm{\widehat h}_\Gamma^2\right)
+\frac\lambda4\norm{\widehat h}^2
\leq
(P-P_n)
\{\ell_{\alpha_0+\widehat\delta,\vartheta_0+\widehat h}
-\ell_{\alpha_0,\vartheta_0}\}
+\lambda M^2.
\label{eq:supp_ridge_basic_two}
\end{align}

\emph{Step 2: localize the empirical term.} Put
\[
\widehat A
=
\widehat\delta^2+\norm{\widehat h}_\Gamma^2.
\]
Subtract \(c_0\widehat A/2+\lambda\|\widehat h\|^2/8\) from both sides of \eqref{eq:supp_ridge_basic_two}. This gives
\begin{align*}
\frac{c_0}{2}\widehat A
+
\frac{\lambda}{8}\norm{\widehat h}^2
&\leq
\Big[
(P-P_n)
\{\ell_{\alpha_0+\widehat\delta,\vartheta_0+\widehat h}
-\ell_{\alpha_0,\vartheta_0}\}\\
&\qquad
-\frac{c_0}{2}\widehat A
-\frac{\lambda}{8}\norm{\widehat h}^2
\Big]
+\lambda M^2.
\end{align*}
The expression in square brackets, evaluated at \((\widehat\delta,\widehat h)\), is no larger than its supremum over all admissible \((\delta,h)\). Lemma~\ref{lem:supp_process} applies to that supremum with \(a=c_0/2\) and \(b=\lambda/8\). Since
\[
\frac{b}{a}
=
\frac{\lambda/8}{c_0/2}
=
\frac{\lambda}{4c_0},
\]
the effective-dimension argument in the lemma is a fixed multiple of \(\lambda\). Taking expectations, dropping the nonnegative term \(\lambda\E\|\widehat h\|^2/8\), and absorbing fixed constants gives
\begin{align}
\E\left(
\widehat\delta^2+\norm{\widehat h}_\Gamma^2
\right)
&\leq
C\left[
\lambda M^2
+\frac{1+\mathcal N_\Gamma(c\lambda)}{n}
\right].
\label{eq:supp_ideal_upper}
\end{align}
The left side equals \(\E\norm{\widehat{\eta}_\lambda-\eta_0}_{P_\Phi}^2\), because \(\Phi\) is centered.

\emph{Step 3: substitute the covariance bound.} Lemma~\ref{lem:supp_effective_dimension} gives
\[
\mathcal N_\Gamma(c\lambda)
\leq
C\lambda^{-1/\nu}\log(e/\lambda).
\]
For \(0<\lambda\leq1\), the right side is at least a positive constant, so the additional parametric term \(1/n\) in \eqref{eq:supp_ideal_upper} is absorbed into the same bound. Since \(M\) is fixed over \(\mathcal P_\nu(M)\), we obtain
\begin{equation*}
\sup_{P\in\mathcal P_\nu(M)}
\E_P\norm{\widehat{\eta}_\lambda-\eta_0}_{P_\Phi}^2
\leq
C\left\{
\lambda
+\frac{\lambda^{-1/\nu}\log(e/\lambda)}{n}
\right\},
\end{equation*}
which is the assertion of Theorem~2.

Finally, take \(\lambda\asymp\{\log(en)/n\}^{\nu/(\nu+1)}\) and again write \(L_n=\log(en)\). Then
\[
\lambda^{-1/\nu}
\asymp
\left(\frac{n}{L_n}\right)^{1/(\nu+1)},
\qquad
\log(e/\lambda)\asymp L_n.
\]
The estimation term is consequently
\begin{align*}
\frac{\lambda^{-1/\nu}\log(e/\lambda)}{n}
&\asymp
\frac1n
\left(\frac{n}{L_n}\right)^{1/(\nu+1)}L_n\\
&=
n^{-\nu/(\nu+1)}L_n^{\nu/(\nu+1)}
\asymp
\lambda.
\end{align*}
Thus the upper bound is of order \(\{\log(en)/n\}^{\nu/(\nu+1)}\), matching Theorem~1.
\end{proof}

\subsection{Proof of Theorem 3}
\label{sec:supp_raw_projection}

Theorem~3 compares the complete \(d\)-dimensional within-path quadratic feature with the finite collection of squared projections used by LAR-2. We begin by defining these two objects precisely. Let \(U\) be the centered first-order sieve feature, and let \(\boldsymbol w_{it,d}\in\R^d\) be the sieve vector for the path of target token \(t\). For the \(j\)th target-position basis function, define
\[
\overline Q_{ij}
=
\frac{1}{T_i-1}
\sum_{t=2}^{T_i}
\varphi_{v,j}(v_{it})\boldsymbol w_{it,d}\boldsymbol w_{it,d}^\prime,
\qquad
Q_{ij}=\overline Q_{ij}-\E(\overline Q_{ij}).
\]
Each outer product is formed within one target-token path before the matrices are averaged over target position. Write \(Q_i=(Q_{i1},\ldots,Q_{i,d_v})\), and suppress the text index when considering a generic observation. Thus, \((U,Q)\) is the complete finite-sieve feature before squared projection. Let \(\Gamma_d\) denote its covariance operator, and equip \(\mathbb Q_d=\bigoplus_{j=1}^{d_v}\R_{\rm sym}^{d\times d}\) with the direct-sum Frobenius inner product.

Let \(\boldsymbol\xi\sim N(0,I_d)\) and set
\(\boldsymbol a=\sqrt d\,\boldsymbol\xi/\norm{\boldsymbol\xi}\). Thus,
\(\boldsymbol a\) is uniform on the sphere of radius \(\sqrt d\), as in
the normalized directions used by the estimator. Put
\(a_d=d/(d+2)\) and define the linear operator
\begin{equation}
\mathcal G_d(C)
=
\E\{(\boldsymbol a^\prime C\boldsymbol a)\boldsymbol a\boldsymbol a^\prime\}
=
a_d\{2C+\tr(C)I_d\}.
\label{eq:supp_spherical_operator}
\end{equation}
To verify the second equality, the \((a,b)\) entry of the expectation is
\[
\sum_{c=1}^d\sum_{e=1}^d C_{ce}\E(a_a a_b a_c a_e).
\]
The fourth-moment identity for the uniform distribution on this sphere is
\[
\E(a_a a_b a_c a_e)
=
a_d\{
\mathbf 1(a=b)\mathbf 1(c=e)
+\mathbf 1(a=c)\mathbf 1(b=e)
+\mathbf 1(a=e)\mathbf 1(b=c)\}.
\]
Indeed, if \(\boldsymbol u\) is uniform on the unit sphere, then
\(\E(u_au_bu_cu_e)\) equals the same sum of indicator products divided
by \(d(d+2)\); multiplying \(\boldsymbol u\) by \(\sqrt d\) gives the
displayed factor \(a_d\). Substitution reduces the matrix entry to
\(a_d\{\tr(C)\mathbf 1(a=b)+2C_{ab}\}\), proving
\eqref{eq:supp_spherical_operator}. To invert the operator, decompose any symmetric matrix as
\[
C=C_0+\bar c I_d,
\qquad
\bar c=\frac{\tr(C)}{d},
\qquad
\tr(C_0)=0.
\]
Equation~\eqref{eq:supp_spherical_operator} gives
\[
\mathcal G_d(C_0)=2a_dC_0,
\qquad
\mathcal G_d(I_d)=dI_d.
\]
Thus \(\mathcal G_d\) multiplies every traceless matrix by \(2a_d\)
and every scalar multiple of \(I_d\) by \(d\). Applying the reciprocal
multiplier on these two orthogonal subspaces and substituting the
definitions of \(C_0\) and \(\bar c\) give
\begin{equation}
\mathcal G_d^{-1}(C)
=
\frac{d+2}{2d}
\left\{C-\frac{\tr(C)}{d+2}I_d\right\},
\qquad
\norm{\mathcal G_d^{-1}}_{\rm op}\leq\frac32.
\label{eq:supp_spherical_inverse}
\end{equation}

For independent copies \(\boldsymbol a_1,\ldots,\boldsymbol a_R\) of
\(\boldsymbol a\), define the squared-projection map by
\[
\{\mathcal R_R(Q)\}_{jr}
=
R^{-1/2}\ip{\boldsymbol a_r\boldsymbol a_r^\prime}{Q_j}_{\rm F},
\qquad
j=1,\ldots,d_v,\quad r=1,\ldots,R.
\]
For a complete quadratic coefficient \(\Theta=(\Theta_1,\ldots,\Theta_{d_v})\in\mathbb Q_d\), define the random coefficient vector
\begin{equation*}
\{\boldsymbol\beta_R(\Theta)\}_{jr}
=
R^{-1/2}
\ip{\boldsymbol a_r\boldsymbol a_r^\prime}
{\mathcal G_d^{-1}(\Theta_j)}_{\rm F}.
\end{equation*}
The next lemma concerns the scalar quadratic functional \(\langle\Theta,Q\rangle\), not approximation of the matrix tuple \(\Theta\) in Frobenius norm. Although every realization uses only \(R\) rank-one directions, the Monte Carlo average can represent an arbitrary norm-bounded \(\Theta\in\mathbb Q_d\) in squared prediction loss. The operator \(\mathcal G_d^{-1}\) corrects the spherical fourth moment, and \eqref{eq:supp_spherical_inverse} shows that this correction is bounded uniformly in \(d\).

\begin{lemma}[Approximation by squared projections]
\label{lem:supp_raw_approximation}
For every \(\Theta,Q\in\mathbb Q_d\),
\begin{align}
\E_{\mathcal R}
\ip{\boldsymbol\beta_R(\Theta)}{\mathcal R_R(Q)}
&=
\ip{\Theta}{Q},
\label{eq:supp_raw_unbiased}\\
\E_{\mathcal R}
\left|
\ip{\boldsymbol\beta_R(\Theta)}{\mathcal R_R(Q)}
-\ip{\Theta}{Q}
\right|^2
&\leq
\frac{C}{R}\norm{\Theta}^2\norm{Q}_{\mathcal G}^2,
\label{eq:supp_raw_mse}\\
\E_{\mathcal R}\norm{\boldsymbol\beta_R(\Theta)}^2
&=
\sum_{j=1}^{d_v}
\ip{\mathcal G_d^{-1}(\Theta_j)}{\Theta_j}_{\rm F}
\leq
\frac32\norm{\Theta}^2,
\label{eq:supp_raw_coefficient_norm}
\end{align}
where
\[
\norm{Q}_{\mathcal G}^2
=
\sum_{j=1}^{d_v}
\ip{Q_j}{\mathcal G_d(Q_j)}_{\rm F}.
\]
The expectation is over the random spherical directions.
\end{lemma}

\begin{proof}
Fix \(\Theta,Q\in\mathbb Q_d\), and put \(A_j=\mathcal G_d^{-1}(\Theta_j)\). For symmetric matrices \(B\) and \(C\), the fourth-moment identity used above gives
\begin{equation}
\E\{(\boldsymbol a^\prime B\boldsymbol a)
(\boldsymbol a^\prime C\boldsymbol a)\}
=
a_d\{2\tr(BC)+\tr(B)\tr(C)\}
=
\ip{\mathcal G_d(B)}{C}_{\rm F}.
\label{eq:supp_two_quadratic_forms}
\end{equation}
By the definitions of \(\boldsymbol\beta_R(\Theta)\) and \(\mathcal R_R(Q)\),
\begin{equation}
\ip{\boldsymbol\beta_R(\Theta)}{\mathcal R_R(Q)}
=
\frac1R\sum_{r=1}^R X(\boldsymbol a_r),
\qquad
X(\boldsymbol a)
=
\sum_{j=1}^{d_v}
(\boldsymbol a^\prime A_j\boldsymbol a)
(\boldsymbol a^\prime Q_j\boldsymbol a).
\label{eq:supp_projection_average}
\end{equation}
Taking expectation in \eqref{eq:supp_projection_average} and applying \eqref{eq:supp_two_quadratic_forms} yield
\begin{align*}
\E X(\boldsymbol a)
&=
\sum_{j=1}^{d_v}
\ip{\mathcal G_d(A_j)}{Q_j}_{\rm F}\\
&=
\sum_{j=1}^{d_v}\ip{\Theta_j}{Q_j}_{\rm F}
=\ip{\Theta}{Q}.
\end{align*}
The average in \eqref{eq:supp_projection_average} is therefore unbiased, proving \eqref{eq:supp_raw_unbiased}.

We next bound its variance. Independence of the \(R\) directions gives
\begin{equation}
\E_{\mathcal R}
\left|
\ip{\boldsymbol\beta_R(\Theta)}{\mathcal R_R(Q)}
-\ip{\Theta}{Q}
\right|^2
=
\frac1R\operatorname{var}\{X(\boldsymbol a)\}
\leq
\frac1R\E\{X(\boldsymbol a)^2\}.
\label{eq:supp_mc_variance}
\end{equation}
For each realization of \(\boldsymbol a\), Cauchy--Schwarz over the \(d_v\) position blocks gives
\[
|X(\boldsymbol a)|^2
\leq
\left\{\sum_{j=1}^{d_v}(\boldsymbol a^\prime A_j\boldsymbol a)^2\right\}
\left\{\sum_{j=1}^{d_v}(\boldsymbol a^\prime Q_j\boldsymbol a)^2\right\}.
\]
A second application of Cauchy--Schwarz, now to expectation, gives
\begin{align*}
\E\{X(\boldsymbol a)^2\}
&\leq
\left[
\E\left\{\sum_j(\boldsymbol a^\prime A_j\boldsymbol a)^2\right\}^2
\right]^{1/2}\\
&\qquad\times
\left[
\E\left\{\sum_j(\boldsymbol a^\prime Q_j\boldsymbol a)^2\right\}^2
\right]^{1/2}.
\end{align*}
For completeness, set \(Z_j=\boldsymbol a^\prime D_j\boldsymbol a\). We first verify the dimension-free fourth-moment bound
\begin{equation}
\E(Z_j^4)
\leq
C\{\E(Z_j^2)\}^2.
\label{eq:supp_spherical_fourth_bound}
\end{equation}
Let \(\boldsymbol z\sim N(0,I_d)\), and write \(\boldsymbol z=L\boldsymbol u\), where \(\boldsymbol u\) is uniform on the unit sphere and independent of \(L=\norm{\boldsymbol z}\). Since \(\boldsymbol a=\sqrt d\,\boldsymbol u\), independence gives
\[
\E(\boldsymbol a^\prime D_j\boldsymbol a)^4
=
\frac{d^4}{\E(L^8)}
\E(\boldsymbol z^\prime D_j\boldsymbol z)^4,
\qquad
\E(L^8)=d(d+2)(d+4)(d+6).
\]
For a symmetric matrix \(D\), the Gaussian quadratic-form identity is
\begin{align*}
\E(\boldsymbol z^\prime D\boldsymbol z)^4
={}&\tr(D)^4
+12\tr(D)^2\tr(D^2)
+12\tr(D^2)^2\\
&+32\tr(D)\tr(D^3)
+48\tr(D^4).
\end{align*}
The inequalities \(|\tr(D^3)|\leq\tr(D^2)^{3/2}\) and \(\tr(D^4)\leq\tr(D^2)^2\), followed by \(2ab\leq a^2+b^2\), show that this expression is at most
\[
C\{\tr(D)^2+\tr(D^2)\}^2.
\]
Moreover, \eqref{eq:supp_two_quadratic_forms} with \(B=C=D\) gives
\[
\E(\boldsymbol a^\prime D\boldsymbol a)^2
=
a_d\{\tr(D)^2+2\tr(D^2)\},
\qquad a_d=\frac{d}{d+2}\geq\frac13.
\]
Since \(d^4/\E(L^8)\leq1\), the last three displays prove \eqref{eq:supp_spherical_fourth_bound}, with a constant independent of \(d\). Consequently, \(\{\E(Z_j^4)\}^{1/2}\leq C\E(Z_j^2)\). Expanding the square and applying Cauchy--Schwarz term by term gives
\begin{align*}
\E\left(\sum_j Z_j^2\right)^2
&=
\sum_{j,k}\E(Z_j^2Z_k^2)\\
&\leq
\sum_{j,k}\{\E(Z_j^4)\E(Z_k^4)\}^{1/2}\\
&\leq
C\sum_{j,k}\E(Z_j^2)\E(Z_k^2)\\
&=
C\left\{\sum_j\E(Z_j^2)\right\}^2.
\end{align*}
Taking square roots and using \(\E(\boldsymbol a^\prime D_j\boldsymbol a)^2=\langle D_j,\mathcal G_d(D_j)\rangle_{\rm F}\), we obtain
\begin{align*}
\left[
\E\left\{\sum_j(\boldsymbol a^\prime D_j\boldsymbol a)^2\right\}^2
\right]^{1/2}
&\leq
C\sum_j\E(\boldsymbol a^\prime D_j\boldsymbol a)^2\\
&=
C\sum_j\ip{D_j}{\mathcal G_d(D_j)}_{\rm F}.
\end{align*}
Apply this inequality first with \(D_j=A_j\) and then with \(D_j=Q_j\). Since \(A_j=\mathcal G_d^{-1}(\Theta_j)\),
\begin{align*}
\E\{X(\boldsymbol a)^2\}
&\leq
C
\left\{
\sum_j\ip{\Theta_j}{\mathcal G_d^{-1}(\Theta_j)}_{\rm F}
\right\}
\norm{Q}_{\mathcal G}^2\\
&\leq
C\norm{\Theta}^2\norm{Q}_{\mathcal G}^2,
\end{align*}
where the last inequality follows from \eqref{eq:supp_spherical_inverse}. Combining this result with \eqref{eq:supp_mc_variance} proves \eqref{eq:supp_raw_mse}.

Finally, direct calculation gives
\begin{align*}
\E_{\mathcal R}\norm{\boldsymbol\beta_R(\Theta)}^2
&=
\frac1R\sum_{r=1}^R\sum_{j=1}^{d_v}
\E(\boldsymbol a_r^\prime A_j\boldsymbol a_r)^2\\
&=
\sum_{j=1}^{d_v}\ip{A_j}{\mathcal G_d(A_j)}_{\rm F}\\
&=
\sum_{j=1}^{d_v}
\ip{\mathcal G_d^{-1}(\Theta_j)}{\Theta_j}_{\rm F}
\leq
\frac32\norm{\Theta}^2.
\end{align*}
This proves \eqref{eq:supp_raw_coefficient_norm}.
\end{proof}

Thus, the approximation class in Lemma~\ref{lem:supp_raw_approximation} is the full direct-sum Frobenius ball \(\{\Theta\in\mathbb Q_d:\|\Theta\|\leq M\}\). The approximation norm used in Theorem~3 is obtained by integrating \eqref{eq:supp_raw_mse} over the distribution of a new text feature \(Q\). The remaining quantity \(\norm{Q}_{\mathcal G}\) is bounded uniformly in the sieve dimension, as follows. From \eqref{eq:supp_spherical_operator},
\[
\norm{Q}_{\mathcal G}^2
=
a_d
\sum_{j=1}^{d_v}
\left\{2\norm{Q_j}_{\rm F}^2+\tr(Q_j)^2\right\}.
\]
Before centering, each \(Q_j\) is an average of matrices of the form \(\varphi_{v,j}(v_{it})\boldsymbol w_{it,d}\boldsymbol w_{it,d}^\prime\). Assumption~1 gives \(|\varphi_{v,j}(v)|\leq K_v\) and \(\norm{\boldsymbol w_{it,d}}\leq\norm{W_{it}}\leq K_W\). For a rank-one matrix,
\[
\norm{\boldsymbol w\boldsymbol w^\prime}_{\rm F}
=
\norm{\boldsymbol w}^2,
\qquad
\tr(\boldsymbol w\boldsymbol w^\prime)
=
\norm{\boldsymbol w}^2.
\]
The triangle inequality therefore gives
\[
\norm{\overline Q_j}_{\rm F}
\leq
K_vK_W^2,
\qquad
|\tr(\overline Q_j)|
\leq
K_vK_W^2.
\]
Since \(Q_j=\overline Q_j-\E(\overline Q_j)\), centering increases each bound by at most a factor of two. Substitution into the preceding formula for \(\|Q\|_{\mathcal G}^2\), followed by summation over the fixed number \(d_v\) of position blocks, gives \(\|Q\|_{\mathcal G}\leq K_Q\) for a constant \(K_Q\) independent of \(d\).

The next lemma shows that the random projection does not replace the covariance-based estimation complexity by the nominal number \(d_vR\) of projected coordinates.

\begin{lemma}[Effective dimension after squared projection]
\label{lem:supp_raw_dimension}
Let \(\Gamma_{d,R}\) be the covariance operator of \((U,\mathcal R_RQ)\). Then, for every \(\lambda>0\),
\begin{equation*}
\E_{\mathcal R}\mathcal N_{\Gamma_{d,R}}(\lambda)
\leq
\mathcal N_{\Gamma_d}(\lambda/2)+d_v.
\end{equation*}
\end{lemma}

\begin{proof}
Define the random linear map \(\mathcal T_R(U,Q)=(U,\mathcal R_RQ)\). Since \(\Gamma_{d,R}=\mathcal T_R\Gamma_d\mathcal T_R^*\), the nonzero eigenvalues of \(\Gamma_{d,R}\) are the same as those of
\[
B_R
=
\Gamma_d^{1/2}\mathcal T_R^*\mathcal T_R\Gamma_d^{1/2}.
\]
To verify this statement, set \(A=\mathcal T_R\Gamma_d^{1/2}\). Then \(AA^*=\Gamma_{d,R}\) and \(A^*A=B_R\). The nonzero eigenvalues of both operators are the squared nonzero singular values of \(A\), with the same multiplicities. Therefore,
\[
\mathcal N_{\Gamma_{d,R}}(\lambda)
=
\mathcal N_{B_R}(\lambda).
\]

For positive operators,
\[
\tr\{B(B+\lambda I)^{-1}\}
=
\tr(I)-\lambda\tr\{(B+\lambda I)^{-1}\}
\]
on the finite-dimensional range under consideration. The inverse map \(B\mapsto(B+\lambda I)^{-1}\) is operator convex, so the negative of its trace is concave. Hence \(B\mapsto\mathcal N_B(\lambda)\) is concave, and Jensen's inequality gives
\begin{equation}
\E_{\mathcal R}\mathcal N_{\Gamma_{d,R}}(\lambda)
\leq
\mathcal N_{\E_{\mathcal R}(B_R)}(\lambda).
\label{eq:supp_dimension_jensen}
\end{equation}
It remains to calculate \(\E(\mathcal T_R^*\mathcal T_R)\). The first-order block is unchanged. For two quadratic features \(Q,C\in\mathbb Q_d\), \eqref{eq:supp_two_quadratic_forms} gives
\begin{align*}
\E_{\mathcal R}
\ip{\mathcal R_R(Q)}{\mathcal R_R(C)}
&=
\sum_{j=1}^{d_v}
\E\{(\boldsymbol a^\prime Q_j\boldsymbol a)
(\boldsymbol a^\prime C_j\boldsymbol a)\}\\
&=
\sum_{j=1}^{d_v}\ip{Q_j}{\mathcal G_d(C_j)}_{\rm F}.
\end{align*}
Hence
\[
\E_{\mathcal R}(\mathcal T_R^*\mathcal T_R)
=
\mathcal D_d
:=
I_U\oplus\mathcal G_d^{\oplus d_v},
\]
and \(\E(B_R)=\Gamma_d^{1/2}\mathcal D_d\Gamma_d^{1/2}\).

Equation~\eqref{eq:supp_spherical_operator} decomposes \(\mathcal D_d\) as
\[
\mathcal D_d
=
(I_U\oplus2a_dI_Q)+a_d\mathcal P,
\]
where \(\mathcal P\) maps each matrix block to \(\tr(C)I_d\). The operator \(\mathcal P\) is positive and has rank \(d_v\), one trace direction for each target-position block. Therefore,
\[
\Gamma_d^{1/2}\mathcal D_d\Gamma_d^{1/2}
=
A_d+P_d,
\]
where \(0\preceq A_d\preceq2\Gamma_d\) and \(P_d\) is positive with rank at most \(d_v\). We spell out the rank argument. Let \(a_1\geq a_2\geq\cdots\) be the eigenvalues of \(A_d\), and let \(\zeta_1\geq\zeta_2\geq\cdots\) be those of \(A_d+P_d\). Rank-\(d_v\) eigenvalue interlacing gives
\[
\zeta_{k+d_v}\leq a_k,
\qquad k\geq1.
\]
Since \(x/(x+\lambda)\leq1\) and is increasing in \(x\),
\begin{align*}
\mathcal N_{A_d+P_d}(\lambda)
&=
\sum_{k=1}^{d_v}\frac{\zeta_k}{\zeta_k+\lambda}
+
\sum_{k>d_v}\frac{\zeta_k}{\zeta_k+\lambda}\\
&\leq
d_v
+
\sum_{k\geq1}\frac{a_k}{a_k+\lambda}
=
d_v+\mathcal N_{A_d}(\lambda).
\end{align*}
Furthermore, \(A_d\preceq2\Gamma_d\), so monotonicity of the effective dimension gives \(\mathcal N_{A_d}(\lambda)\leq\mathcal N_{2\Gamma_d}(\lambda)\). Thus,
\begin{align*}
\mathcal N_{\Gamma_d^{1/2}\mathcal D_d\Gamma_d^{1/2}}(\lambda)
&\leq
\mathcal N_{A_d}(\lambda)+d_v\\
&\leq
\mathcal N_{2\Gamma_d}(\lambda)+d_v\\
&=
\mathcal N_{\Gamma_d}(\lambda/2)+d_v.
\end{align*}
Combining this display with \eqref{eq:supp_dimension_jensen} proves the lemma.
\end{proof}

We also require an oracle inequality in which the response law is determined by a
complete feature but the estimator uses only a reduced feature.

\begin{lemma}[Ridge oracle bound under feature reduction]
\label{lem:supp_approximate_ridge}
Let \((Y,X)\) satisfy
\[
\Pp(Y=1\mid X)=\sigma\{\eta_0(X)\},
\]
and let \(Z=T(X)\) be a centered bounded Hilbert-space feature with covariance
operator \(\Gamma_Z\). Consider scores
\(f_{\alpha,\theta}(Z)=\alpha+\ip{\theta}{Z}\) in a fixed coefficient ball.
Assume that \(\eta_0(X)\) and all candidate scores are uniformly bounded. Let
\(\widehat f_\lambda=f_{\widehat\alpha,\widehat\theta}\) minimize
\[
P_n\ell_Y\{f_{\alpha,\theta}(Z)\}
+\frac{\lambda}{2}\norm{\theta}^2
\]
over this ball. Then, for every comparator
\(f^*(Z)=\alpha^*+\ip{\theta^*}{Z}\) in the same ball,
\begin{equation}
\E_{\mathcal D}
\norm{\widehat f_\lambda(Z)-\eta_0(X)}_{P_X}^2
\leq
C\left\{
\norm{f^*(Z)-\eta_0(X)}_{P_X}^2
+\lambda\norm{\theta^*}^2
+\frac{1+\mathcal N_{\Gamma_Z}(c\lambda)}{n}
\right\}.
\label{eq:supp_approximate_ridge}
\end{equation}
Here
\[
\norm{f(Z)-\eta_0(X)}_{P_X}^2
=
\E_X\left[\{f(T(X))-\eta_0(X)\}^2\right],
\]
and \(\E_{\mathcal D}\) is over the labeled training sample. The constants
\(C,c>0\) depend only on the fixed score and coefficient bounds.
\end{lemma}

\begin{proof}
Let \(B(s)=\log(1+e^s)\), define
\(L(f)=\E[\ell_Y\{f(Z)\}]\), and put
\(L(\eta_0)=\E[\ell_Y\{\eta_0(X)\}]\). For any candidate score \(f(Z)\),
conditioning on the complete feature \(X\) gives
\begin{align*}
&\E\left[
\ell_Y\{f(Z)\}-\ell_Y\{\eta_0(X)\}
\,\middle|\,X
\right]\\
&\qquad=
B\{f(Z)\}-B\{\eta_0(X)\}
-B'\{\eta_0(X)\}\{f(Z)-\eta_0(X)\},
\end{align*}
because \(Z=T(X)\) and
\(B'\{\eta_0(X)\}=\sigma\{\eta_0(X)\}=\E(Y\mid X)\). Uniform boundedness
places \(f(Z)\), \(\eta_0(X)\), and every point between them in a fixed compact
interval. Taylor's theorem and
\(B''(s)=\sigma(s)\{1-\sigma(s)\}\) therefore give fixed constants
\(0<c_0<C_0<\infty\) such that
\begin{equation}
c_0\norm{f(Z)-\eta_0(X)}_{P_X}^2
\leq
L(f)-L(\eta_0)
\leq
C_0\norm{f(Z)-\eta_0(X)}_{P_X}^2.
\label{eq:supp_reduced_curvature}
\end{equation}
This argument conditions on \(X\), so \(\eta_0(X)\) need not be measurable with
respect to \(Z\).

Write
\[
a^2=\norm{f^*(Z)-\eta_0(X)}_{P_X}^2,
\qquad
\widehat\delta=\widehat\alpha-\alpha^*,
\qquad
\widehat h=\widehat\theta-\theta^*.
\]
The lower and upper bounds in \eqref{eq:supp_reduced_curvature}, together with
\[
\norm{\widehat f_\lambda(Z)-\eta_0(X)}_{P_X}^2
\geq
\frac12\norm{\widehat f_\lambda-f^*}_{P_Z}^2-a^2,
\]
imply
\begin{align}
L(\widehat f_\lambda)-L(f^*)
&=
\{L(\widehat f_\lambda)-L(\eta_0)\}
-\{L(f^*)-L(\eta_0)\}\notag\\
&\geq
\frac{c_0}{2}\norm{\widehat f_\lambda-f^*}_{P_Z}^2-C_1a^2
\label{eq:supp_reduced_risk_difference}
\end{align}
for a fixed constant \(C_1\). The norm inequality follows by writing
\(\widehat f_\lambda-\eta_0=(\widehat f_\lambda-f^*)+(f^*-\eta_0)\) in
\(L_2(P_X)\) and using
\(2\langle u,v\rangle\geq-\norm{u}^2/2-2\norm{v}^2\).

The defining inequality for the ridge estimator, with \(f^*\) as comparator, is
\begin{equation}
P_n\{\ell_Y(\widehat f_\lambda)-\ell_Y(f^*)\}
+\frac\lambda2
\left(\norm{\widehat\theta}^2-\norm{\theta^*}^2\right)
\leq0.
\label{eq:supp_oracle_basic}
\end{equation}
Moreover,
\begin{align*}
\norm{\widehat\theta}^2-\norm{\theta^*}^2
&=
\norm{\widehat h}^2+2\ip{\theta^*}{\widehat h}\\
&\geq
\frac12\norm{\widehat h}^2-2\norm{\theta^*}^2.
\end{align*}
Add and subtract the population loss difference in \eqref{eq:supp_oracle_basic} and
use \eqref{eq:supp_reduced_risk_difference} and the penalty bound. Because both
\(\widehat f_\lambda\) and \(f^*\) are functions of \(Z\), centering of \(Z\) gives
\[
\norm{\widehat f_\lambda-f^*}_{P_Z}^2
=
\widehat\delta^2+\norm{\widehat h}_{\Gamma_Z}^2.
\]
Consequently,
\begin{align}
\frac{c_0}{2}
\left(\widehat\delta^2+\norm{\widehat h}_{\Gamma_Z}^2\right)
+\frac\lambda4\norm{\widehat h}^2
&\leq
(P-P_n)\{\ell_Y(\widehat f_\lambda)-\ell_Y(f^*)\}
+C_1a^2+\lambda\norm{\theta^*}^2.
\label{eq:supp_oracle_local}
\end{align}

Put \(A^*=\widehat\delta^2+\|\widehat h\|_{\Gamma_Z}^2\). Moving half of the two quadratic terms in \eqref{eq:supp_oracle_local} to the right gives
\begin{align*}
\frac{c_0}{4}A^*
+
\frac{\lambda}{8}\norm{\widehat h}^2
&\leq
\Big[
(P-P_n)\{\ell_Y(\widehat f_\lambda)-\ell_Y(f^*)\}
-\frac{c_0}{4}A^*
-\frac{\lambda}{8}\norm{\widehat h}^2
\Big]\\
&\qquad
+C_1a^2+\lambda\norm{\theta^*}^2.
\end{align*}
The bracketed expression is bounded by the localized supremum in
Lemma~\ref{lem:supp_process}, with base score \(f^*\), \(a=c_0/4\), and
\(b=\lambda/8\). This application uses only the one-Lipschitz property of the
logistic loss and does not require \(f^*\) to equal the true score. Taking
expectations over the training sample, dropping the nonnegative coefficient-norm
term, and absorbing fixed constants yields
\[
\E_{\mathcal D}\norm{\widehat f_\lambda-f^*}_{P_Z}^2
\leq
C\left\{
a^2+\lambda\norm{\theta^*}^2
+\frac{1+\mathcal N_{\Gamma_Z}(c\lambda)}{n}
\right\}.
\]
Finally,
\[
\norm{\widehat f_\lambda(Z)-\eta_0(X)}_{P_X}^2
\leq
2\norm{\widehat f_\lambda-f^*}_{P_Z}^2+2a^2.
\]
Taking expectations over the training sample and combining the last two displays
proves \eqref{eq:supp_approximate_ridge}.
\end{proof}

\begin{proof}
We now combine the preceding lemmas to prove Theorem~3.
Fix \(M_1\geq\sqrt{3}M\), as in the theorem, and define the projected ridge
estimator over \(|\alpha|\leq M_1\) and
\(\norm{\boldsymbol\beta}_2\leq M_1\).

\emph{Step 1: construct a comparator in the projected model.} The projected population score can be written as
\begin{equation*}
\eta_{0,d}(U,Q)
=
\alpha_0+\ip{b_d}{U}+\ip{\Theta_d}{Q},
\end{equation*}
Set \(\epsilon_d^2=\norm{\eta_{0,d}-\eta_0}_{P_\Phi}^2\). The coefficient norm satisfies \(\norm{b_d}^2+\norm{\Theta_d}^2\leq M^2\), because orthogonal projection cannot increase it. Define
\[
\eta_R^*(U,Q)
=
\alpha_0+\ip{b_d}{U}
+\ip{\boldsymbol\beta_R(\Theta_d)}{\mathcal R_R(Q)}.
\]
For each fixed \(Q\), Lemma~\ref{lem:supp_raw_approximation} bounds the squared approximation error. Integrating first over the text distribution and then over the projection directions gives
\begin{align*}
\E_{\mathcal R}
\norm{\eta_R^*-\eta_{0,d}}_{P_\Phi}^2
&=
\E_Q\E_{\mathcal R}
\left|
\ip{\boldsymbol\beta_R(\Theta_d)}{\mathcal R_R(Q)}
-\ip{\Theta_d}{Q}
\right|^2\\
&\leq
\frac{C}{R}\norm{\Theta_d}^2
\E\norm{Q}_{\mathcal G}^2
\leq
\frac{CM^2}{R}.
\end{align*}
The equality is Fubini's theorem: the squared prediction norm integrates over a new text feature \(Q\), while the outer expectation integrates over the independently generated projection directions.

The random coefficient can be placed in a fixed coefficient ball without changing this order. To see this, write
\[
S_R=\norm{\boldsymbol\beta_R(\Theta_d)}^2
=
\frac1R\sum_{r=1}^R
\sum_{j=1}^{d_v}
(\boldsymbol a_r^\prime A_j\boldsymbol a_r)^2,
\qquad
A_j=\mathcal G_d^{-1}(\Theta_{d,j}).
\]
Equation~\eqref{eq:supp_raw_coefficient_norm} gives \(\E S_R\leq3\norm{\Theta_d}^2/2\). To bound its variance, put \(Y_r=\sum_j(\boldsymbol a_r^\prime A_j\boldsymbol a_r)^2\), so that \(S_R=R^{-1}\sum_rY_r\). The variables \(Y_r\) are independent. The spherical eighth-moment bound used in Lemma~\ref{lem:supp_raw_approximation} gives
\begin{align*}
\operatorname{var}(S_R)
&=
\frac1R\operatorname{var}(Y_1)
\leq
\frac1R\E(Y_1^2)\\
&\leq
\frac{C}{R}
\left\{
\sum_j\E(\boldsymbol a^\prime A_j\boldsymbol a)^2
\right\}^2\\
&=
\frac{C}{R}
\left\{
\sum_j\ip{\Theta_{d,j}}{\mathcal G_d^{-1}(\Theta_{d,j})}_{\rm F}
\right\}^2
\leq
\frac{CM^4}{R}.
\end{align*}
Moreover,
\[
\norm{b_d}^2+\E S_R
\leq
\norm{b_d}^2+\frac32\norm{\Theta_d}^2
\leq \frac32M^2.
\]
The event \(\norm{b_d}^2+S_R>3M^2\) therefore implies \(S_R-\E S_R>3M^2/2\). Chebyshev's inequality gives
\begin{equation}
\Pp_{\mathcal R}
\left\{
\norm{b_d}^2+S_R>3M^2
\right\}
\leq
\frac{C}{R}.
\label{eq:supp_comparator_radius}
\end{equation}
On the event in which the inequality inside braces does not occur, use \(\eta_R^*\) as the comparison score. On its complement, use the score \(\alpha_0+\langle b_d,U\rangle\), which omits the quadratic term. Both scores lie in a fixed coefficient ball of radius \(\sqrt{3}M\). On the exceptional event, the difference between the latter score and \(\eta_{0,d}\) is \(-\langle\Theta_d,Q\rangle\). The feature and coefficient bounds give
\[
|\ip{\Theta_d}{Q}|
\leq
\norm{\Theta_d}\norm{Q}
\leq
C,
\]
uniformly in \(d\). The last inequality follows from \(\|\Theta_d\|\leq M\) and \(\|Q\|^2\leq3\|Q\|_{\mathcal G}^2/2\leq3K_Q^2/2\), because \(a_d\geq1/3\). Its squared contribution, multiplied by the probability in \eqref{eq:supp_comparator_radius}, is therefore \(O(R^{-1})\). On the complementary event, the preceding mean-square calculation for \(\eta_R^*\) already gives an \(O(R^{-1})\) contribution. Denote the resulting comparison score by \(\widetilde{\eta}_R\). Then
\begin{equation}
\E_{\mathcal R}
\norm{\widetilde{\eta}_R-\eta_{0,d}}_{P_\Phi}^2
\leq
\frac{C}{R},
\qquad
\norm{\widetilde{\boldsymbol\beta}_R}_2^2\leq3M^2.
\label{eq:supp_truncated_comparator}
\end{equation}
Because \(|\alpha_0|\leq M\), the comparison score belongs to the projected
parameter class with radius \(M_1\). The fixed enlargement of the radius changes only
the constants in the theorem.

\emph{Step 2: apply the reduced-feature oracle bound conditionally on the
projections.} The projection directions are independent of the labeled sample.
Conditional on \(\boldsymbol a_1,\ldots,\boldsymbol a_R\), set
\[
X=\Phi,
\qquad
Z_R=(U,\mathcal R_RQ).
\]
The reduced feature \(Z_R\) is then a fixed centered function of the complete feature
\(\Phi\), with covariance operator \(\Gamma_{d,R}\), while the response law remains
\[
\Pp(Y=1\mid X)=\sigma\{\eta_0(\Phi)\}.
\]
The comparison score \(\widetilde\eta_R\) constructed in Step 1 is also fixed after
conditioning on the directions and belongs to the projected parameter class.
Assumption~1 and \(\norm{\vartheta_0}\leq M\) bound \(\eta_0(\Phi)\), while the
bounded-predictor condition in Theorem~3 bounds every score in the projected class
uniformly in \(d\) and \(R\). The curvature constants in the following application are
therefore uniform in both resolutions.
Lemma~\ref{lem:supp_approximate_ridge}, applied with complete feature \(X=\Phi\),
reduced feature \(Z_R\), and comparator \(f^*=\widetilde\eta_R\), gives
\begin{align*}
\E_{\mathcal D}\left[
\left.
\norm{\widehat{\eta}_{d,R,\lambda}-\eta_0}_{P_\Phi}^2
\,\right|\,
\boldsymbol a_1,\ldots,\boldsymbol a_R
\right]
&\leq
C\left[
\norm{\widetilde{\eta}_R-\eta_0}_{P_\Phi}^2
+\lambda\norm{\widetilde{\boldsymbol\beta}_R}_2^2
+\frac{1+\mathcal N_{\Gamma_{d,R}}(c\lambda)}{n}
\right].
\end{align*}
The target norm is \(P_\Phi\) because the true log odds are defined on the complete
feature; the effective dimension is based on \(\Gamma_{d,R}\) because the fitted score
uses only \(Z_R\). No measurability of \(\eta_0(\Phi)\) with respect to \(Z_R\) is
required.

The approximation term satisfies
\[
\norm{\widetilde{\eta}_R-\eta_0}_{P_\Phi}^2
\leq
2\norm{\widetilde{\eta}_R-\eta_{0,d}}_{P_\Phi}^2
+2\epsilon_d^2.
\]
Average the conditional bound over the projection directions. Using
\eqref{eq:supp_truncated_comparator} and the preceding decomposition of the
approximation term gives
\begin{align*}
\E_{\mathcal D,\mathcal R}
\norm{\widehat{\eta}_{d,R,\lambda}-\eta_0}_{P_\Phi}^2
&\leq
C\left[
2\E_{\mathcal R}\norm{\widetilde{\eta}_R-\eta_{0,d}}_{P_\Phi}^2
+2\epsilon_d^2
+3M^2\lambda\right.\\
&\hspace{34mm}\left.
+\frac{1+\E_{\mathcal R}\mathcal N_{\Gamma_{d,R}}(c\lambda)}{n}
\right].
\end{align*}
Lemma~\ref{lem:supp_raw_dimension} gives
\[
\E_{\mathcal R}\mathcal N_{\Gamma_{d,R}}(c\lambda)
\leq
\mathcal N_{\Gamma_d}(c\lambda/2)+d_v.
\]
Because \(d_v\geq1\), the scalar \(n^{-1}\) term can be absorbed into \(d_v/n\). Renaming the fixed constants \(C\) and \(c\), we obtain
\begin{equation}
\E_{\mathcal D,\mathcal R}
\norm{\widehat{\eta}_{d,R,\lambda}-\eta_0}_{P_\Phi}^2
\leq
C\left[
\epsilon_d^2
+\lambda
+\frac{\mathcal N_{\Gamma_d}(c\lambda)+d_v}{n}
+\frac1R
\right],
\label{eq:supp_finite_projection_bound}
\end{equation}
after changing the constants \(C\) and \(c\). Under the covariance decay assumption,
\(\mathcal N_{\Gamma_d}(c\lambda)\leq
C\lambda^{-1/\nu}\log(e/\lambda)\). Since \(d_v\) is fixed, its parametric
contribution is absorbed into the same estimation bound. This proves the inequality
asserted in Theorem~3.

\emph{Step 3: bound the path-sieve approximation error.} Suppose that \(\Pi_d\) projects onto \(e_1,\ldots,e_d\). Write the population coefficient as \(\vartheta_0=(b_0,\Theta_0)\), corresponding to the first- and second-order feature blocks. Let \(b_\perp\) be the component of \(b_0\) outside \(\R^{d_v}\otimes\operatorname{span}(e_1,\ldots,e_d)\), and let \(\Theta_\perp\) be the component of \(\Theta_0\) outside the associated symmetric tensor space. Orthogonal projection gives
\[
\norm{b_\perp}^2+\norm{\Theta_\perp}^2
\leq
\norm{\vartheta_0}^2
\leq M^2.
\]
The score omitted by the finite sieve is
\[
\eta_0(\Phi)-\eta_{0,d}(\Phi)
=
\ip{b_\perp}{\Phi^{(1)}}
+
\ip{\Theta_\perp}{\Phi^{(2)}}.
\]
Every coordinate of \(b_\perp\) has a path index larger than \(d\). Since the eigenvalues are nonincreasing, the covariance envelope in \eqref{eq:supp_text_capacity} gives
\[
\E\ip{b_\perp}{\Phi^{(1)}}^2
\leq
C\kappa_{d+1}\norm{b_\perp}^2.
\]
Every coordinate of \(\Theta_\perp\) has at least one path index larger than \(d\). The largest product eigenvalue among such coordinates is \(\kappa_1\kappa_{d+1}\), and hence
\[
\E\ip{\Theta_\perp}{\Phi^{(2)}}^2
\leq
C\kappa_1\kappa_{d+1}\norm{\Theta_\perp}^2.
\]
Using \((a+b)^2\leq2a^2+2b^2\) and the coefficient-norm bound,
\begin{align*}
\epsilon_d^2
&\leq
CM^2\{\kappa_{d+1}+\kappa_1\kappa_{d+1}\}\\
&\leq
CM^2\kappa_{d+1}
\leq
CM^2d^{-\nu}.
\end{align*}
This proves the approximation statement following Theorem~3.

\emph{Step 4: verify the stated resolutions.} Since \(\Gamma_d\) is the covariance of an orthogonal finite-dimensional projection of the complete feature, its ordered eigenvalues are no larger than those of \(\Gamma\). The definition of effective dimension and Lemma~\ref{lem:supp_effective_dimension} therefore give
\[
\mathcal N_{\Gamma_d}(c\lambda)
\leq
C\lambda^{-1/\nu}\log(e/\lambda).
\]
Set
\[
r_n
=
\left\{\frac{\log(en)}{n}\right\}^{\nu/(\nu+1)}.
\]
Write \(L_n=\log(en)\). If \(d\gtrsim(n/L_n)^{1/(\nu+1)}\), then raising both sides to power \(-\nu\) gives
\[
d^{-\nu}
\lesssim
\left(\frac{n}{L_n}\right)^{-\nu/(\nu+1)}
=r_n.
\]
If \(R\gtrsim(n/L_n)^{\nu/(\nu+1)}\), then
\[
R^{-1}
\lesssim
\left(\frac{n}{L_n}\right)^{-\nu/(\nu+1)}
=r_n.
\]
Finally, if \(\lambda\asymp r_n\), then \(\lambda^{-1/\nu}\asymp(n/L_n)^{1/(\nu+1)}\) and \(\log(e/\lambda)\asymp L_n\). Therefore,
\begin{align*}
\frac{\lambda^{-1/\nu}\log(e/\lambda)}{n}
&\asymp
\frac1n
\left(\frac{n}{L_n}\right)^{1/(\nu+1)}L_n\\
&=
n^{-\nu/(\nu+1)}L_n^{\nu/(\nu+1)}
=r_n.
\end{align*}
The regularization term itself is \(\lambda\asymp r_n\), while the fixed position term satisfies \(d_v/n\lesssim r_n\) because \(\nu/(\nu+1)<1\). Substitution into \eqref{eq:supp_finite_projection_bound} shows that every term is of order at most \(r_n\). Hence the finite path sieve and squared-projection estimator attain the minimax rate established by Theorems~1 and~2.
\end{proof}

\subsection{Proof of Proposition 1}
\label{sec:supp_auc}

\begin{proof}
Let \(q(\Phi)=\Pp(Y=1\mid\Phi)\) and \(\widehat q(\Phi)=\sigma\{\widehat{\eta}(\Phi)\}\). We first condition on the labeled sample used to construct \(\widehat{\eta}\), so that \(\widehat{\eta}\) and \(\widehat q\) are fixed functions. Let \(\Phi'\) be an independent copy of \(\Phi\), and write \(q'=q(\Phi')\) and \(\widehat q'=\widehat q(\Phi')\).

Assumption~3 states that \(\Pp(Y=y)\geq c_Y\) for \(y\in\{0,1\}\) and some \(c_Y>0\), and that, for some \(C_m<\infty\) and \(\tau\geq0\),
\[
\Pp\{0<|q-q'|\leq t\}
\leq C_m t^\tau,
\qquad t>0.
\]

We begin with the pairwise representation of AUC. For any score \(s\), define
\[
a_s(\Phi,\Phi')
=
\mathbf 1\{s(\Phi)>s(\Phi')\}
+\frac12\mathbf 1\{s(\Phi)=s(\Phi')\}.
\]
By conditioning on \((\Phi,\Phi')\),
\[
\Pp(Y=1)\Pp(Y=0)\AUC(s)
=
\E\{q(1-q')a_s(\Phi,\Phi')\}.
\]
Interchanging \(\Phi\) and \(\Phi'\), adding the two equal expressions, and using \(a_s(\Phi',\Phi)=1-a_s(\Phi,\Phi')\), we obtain
\begin{align*}
2\Pp(Y=1)\Pp(Y=0)\AUC(s)
&=
\E\left[
q(1-q')a_s(\Phi,\Phi')
+q'(1-q)a_s(\Phi',\Phi)
\right]\\
&=
\E\left[
q'(1-q)+(q-q')a_s(\Phi,\Phi')
\right].
\end{align*}
The final expression is maximized pointwise by ordering observations according to \(q\). Since \(\eta_0=\log\{q/(1-q)\}\) is strictly increasing in \(q\), this ordering is also produced by \(\eta_0\).

Let
\[
\mathcal E
=
\left\{
(\widehat q-\widehat q')(q-q')\leq0,
\quad q\neq q'
\right\}
\]
be the event on which the estimated and optimal rankings disagree, including an estimated tie. Subtracting the pairwise expression for \(\widehat{\eta}\) from that for \(\eta_0\) gives
\begin{equation}
\AUC(\eta_0)-\AUC(\widehat{\eta})
\leq
\frac{1}{2\Pp(Y=1)\Pp(Y=0)}
\E\left\{
|q-q'|\mathbf 1(\mathcal E)
\right\}.
\label{eq:supp_auc_pairwise_bound}
\end{equation}
Here is the pointwise step behind this inequality. The difference of the two pairwise formulas is
\[
\frac{1}{2\Pp(Y=1)\Pp(Y=0)}
\E\left[
(q-q')
\{a_{\eta_0}(\Phi,\Phi')-a_{\widehat{\eta}}(\Phi,\Phi')\}
\right].
\]
If \(q>q'\), then \(a_{\eta_0}=1\), and the expression in braces is positive only when the fitted score fails to place \(\Phi\) strictly above \(\Phi'\). If \(q<q'\), the same statement holds with the observations interchanged. In either case, the integrand is at most \(|q-q'|\mathbf1(\mathcal E)\), which proves \eqref{eq:supp_auc_pairwise_bound}. 
The class-probability condition in Assumption~3 bounds the denominator away from zero.

Set
\[
e=|\widehat q-q|,
\qquad
e'=|\widehat q'-q'|,
\qquad
\Delta=|q-q'|.
\]
On \(\mathcal E\), the estimated ordering has the opposite sign from the optimal ordering. The triangle inequality then gives
\begin{equation*}
\Delta\mathbf 1(\mathcal E)
\leq
(e+e')\mathbf 1(\mathcal E).
\end{equation*}
For example, if \(q>q'\) but \(\widehat q\leq\widehat q'\), then
\[
q-q'
\leq
(q-\widehat q)+(\widehat q'-q')
\leq e+e'.
\]
The other ordering is identical after interchanging the two observations.

Let \(\varepsilon^2=\E(e^2)\), where the expectation is over a new \(\Phi\), conditional on the training sample. For any \(0<t\leq1\), split the expectation in \eqref{eq:supp_auc_pairwise_bound} according to whether \(\Delta\leq t\). Assumption~3 gives
\begin{align*}
\E\{\Delta\mathbf 1(\mathcal E)\}
&\leq
\E\{\Delta\mathbf 1(0<\Delta\leq t)\}
+\E\{\Delta\mathbf 1(\mathcal E,\Delta>t)\}\\
&\leq
t\Pp(0<\Delta\leq t)
+\E\{(e+e')\mathbf 1(e+e'>t)\}\\
&\leq
C_mt^{\tau+1}
+\frac{\E(e+e')^2}{t}\\
&\leq
C_mt^{\tau+1}+\frac{4\varepsilon^2}{t}.
\end{align*}
On \(\mathcal E\cap\{\Delta>t\}\), the inequality \(\Delta\leq e+e'\) implies \(e+e'>t\); this explains the second term in the second line. The third line uses \(x\mathbf 1(x>t)\leq x^2/t\), and the last line uses \((e+e')^2\leq2e^2+2e'^2\) together with the identical distributions of \(e\) and \(e'\).

If \(\varepsilon=0\), the desired result is immediate. Otherwise, since \(0\leq e\leq1\), we have \(\varepsilon^2\leq1\) and may choose \(t=(\varepsilon^2)^{1/(\tau+2)}\). The two terms then have the same order:
\[
t^{\tau+1}
=
\frac{\varepsilon^2}{t}
=
(\varepsilon^2)^{(\tau+1)/(\tau+2)}.
\]
Substitution into \eqref{eq:supp_auc_pairwise_bound} yields, conditionally on the training sample,
\begin{equation}
\AUC(\eta_0)-\AUC(\widehat{\eta})
\leq
C(\varepsilon^2)^{(\tau+1)/(\tau+2)}.
\label{eq:supp_auc_probability_error}
\end{equation}

The logistic function is \(1/4\)-Lipschitz because \(\sup_x\sigma(x)\{1-\sigma(x)\}=1/4\). Therefore,
\[
\varepsilon^2
\leq
\frac1{16}\norm{\widehat{\eta}-\eta_0}_{P_\Phi}^2.
\]
Finally, take expectation over the labeled sample in \eqref{eq:supp_auc_probability_error}. Since \((\tau+1)/(\tau+2)\in(0,1)\), Jensen's inequality gives
\[
\AUC(\eta_0)-\E\{\AUC(\widehat{\eta})\}
\leq
C
\left[
\E\norm{\widehat{\eta}-\eta_0}_{P_\Phi}^2
\right]^{(\tau+1)/(\tau+2)}.
\]
More explicitly, conditional on the training sample, \(\varepsilon^2\leq\|\widehat{\eta}-\eta_0\|_{P_\Phi}^2/16\). The concavity of \(x^{(\tau+1)/(\tau+2)}\) gives
\begin{align*}
\E_{\rm train}\left\{(\varepsilon^2)^{(\tau+1)/(\tau+2)}\right\}
&\leq
\left\{\E_{\rm train}(\varepsilon^2)\right\}^{(\tau+1)/(\tau+2)}
\\
&\leq
16^{-(\tau+1)/(\tau+2)}
\left\{
\E_{\rm train}\norm{\widehat{\eta}-\eta_0}_{P_\Phi}^2
\right\}^{(\tau+1)/(\tau+2)}.
\end{align*}
This is the assertion of Proposition~1.
\end{proof}

\subsection{Additional theoretical remarks}
\label{sec:supp_scope}

The localized empirical-process argument is a standard part of regularized regression. The problem-specific steps are the conditional-averaging argument for dependent token paths, the ordering of the quadratic product spectrum, the matching path-level lower submodel, and the analysis of the squared projections used by the finite estimator. The Rademacher construction is used only to establish a least-favorable submodel; it is not an assumption on the observed likelihood arrays. Arbitrary dependence among paths within a text remains allowed because the covariance bound is transferred by conditional expectation.

The bound \(\epsilon_d^2\leq Cd^{-\nu}\) uses a spectral sieve. For the spline and random channel basis used numerically, Theorem~3 remains valid with this approximation error left explicit. The bounded-predictor condition in Theorem~3 is uniform in \(d\) and \(R\); it is stated separately because the normalization \(\|\boldsymbol a_r\|_2=\sqrt d\) does not make pointwise boundedness of projected predictors follow from Assumption~1 alone. This condition provides uniform logistic curvature. It does not imply that the numerical choice \(\lambda=0.01\) is asymptotically optimal.

\subsection{Equivalent empirical-measure formulation}

The finite-sum definitions of LAR-1 and LAR-2 admit an equivalent formulation through normalized empirical measures. Using the notation of Section~3 of the main paper, let \(\delta_x\) denote a unit point mass at \(x\). For target token \(t\) in text \(i\), define the path-level measure
\[
\mu_{it}
=
\frac{1}{G}
\sum_{g=1}^{G}
\delta_{\left(u_g,\boldsymbol h_{itg}\right)},
\]
and define the document-level measure
\[
\mu_i
=
\frac{1}{(T_i-1)G}
\sum_{t=2}^{T_i}
\sum_{g=1}^{G}
\delta_{\left(u_g,v_{it},\boldsymbol h_{itg}\right)}.
\]
Both are probability measures, with equal mass across aligned scales and equal total mass across target tokens. The LAR-1 log odds can then be written as
\[
\eta_i^{(1)}
=
\alpha
+
\int
b(u,v,\boldsymbol h)
\,d\mu_i(u,v,\boldsymbol h).
\]
Similarly, the second-order contribution of target token \(t\) is
\[
Q_{it}
=
\iint
\kappa_{v_{it}}
\left(
(u,\boldsymbol h),
(u',\widetilde{\boldsymbol h})
\right)
\,d\mu_{it}(u,\boldsymbol h)
\,d\mu_{it}(u',\widetilde{\boldsymbol h}),
\]
and \(\eta_i^{(2)}=\eta_i^{(1)}+(T_i-1)^{-1}\sum_{t=2}^{T_i}Q_{it}\). Thus, the empirical measures provide compact notation for the same finite aligned arrays: \(\mu_i\) retains context scale, target position, and channel value for the cellwise term, while the path-specific measures \(\mu_{it}\) retain the target-token grouping used by the second-order term.

The same notation also describes conventional scores computed from the full-context cells. For channel $c\in\{\ell,z\}$, define its full-context empirical distribution by
\[
\mu_{i,c}^{\mathrm{full}}
=
\frac{1}{T_i-1}
\sum_{t=2}^{T_i}
\delta_{h_{it,c}(1)},
\]
where $h_{it,\ell}(1)=\ell(1,t)$ and $h_{it,z}(1)=z(1,t)$. If $q_{i,c}(\kappa)$ is the empirical $\kappa$-quantile under $\mu_{i,c}^{\mathrm{full}}$, then average loss and a lower-tail score can be written as
\[
S_{{\rm loss},i}
=
-\int x\,d\mu_{i,\ell}^{\mathrm{full}}(x),
\qquad
S_{\kappa,c,i}
=
\frac{\int x\,\mathbf 1\{x\leq q_{i,c}(\kappa)\}\,d\mu_{i,c}^{\mathrm{full}}(x)}
{\int \mathbf 1\{x\leq q_{i,c}(\kappa)\}\,d\mu_{i,c}^{\mathrm{full}}(x)}.
\]
The choices $c=\ell$ and $c=z$ correspond to Min-K\% and Min-K\%++. Their document-specific quantiles define functionals of the full-context empirical distribution rather than special cases of LAR-1's shared cellwise function.

\clearpage
\section{Additional numerical results}
\label{sec:supp_numerical}

\subsection{Operating characteristics}

The main paper uses AUC to summarize ranking over the complete false-positive range. Table~\ref{tab:supp_operating_characteristics} reports the normalized partial AUC over \([0,0.10]\) and the empirical true-positive rate at three false-positive rates. These descriptive operating characteristics are computed from the five-fold out-of-fold LAR-2 scores under the primary representation and fitting specification.

\begin{table}[H]
\centering
\caption{Operating characteristics of LAR-2 in the low-false-positive region. The partial AUC is normalized by the width of the interval \([0,0.10]\).}
\label{tab:supp_operating_characteristics}
\small
\begin{tabular}{llrrrr}
\toprule
Task & Scoring model & pAUC$_{0.10}$ & TPR$_{0.01}$ & TPR$_{0.05}$ & TPR$_{0.10}$ \\
\midrule
WikiMIA & Pythia-1B & 0.703 & 0.447 & 0.749 & 0.862 \\
 & Pythia-6.9B & 0.751 & 0.533 & 0.805 & 0.900 \\
 & Pythia-12B & 0.733 & 0.426 & 0.789 & 0.887 \\
 & LLaMA-2-7B & 0.837 & 0.611 & 0.898 & 0.937 \\
 & Falcon-7B & 0.869 & 0.674 & 0.931 & 0.984 \\
\addlinespace
AI detection & DeepSeek-7B & 0.934 & 0.833 & 0.953 & 0.977 \\
 & OLMo-7B & 0.971 & 0.963 & 0.977 & 0.987 \\
 & Qwen3-4B & 0.971 & 0.948 & 0.970 & 0.990 \\
 & Qwen3-8B & 0.973 & 0.958 & 0.978 & 0.985 \\
 & Qwen3-14B & 0.976 & 0.963 & 0.992 & 0.993 \\
\bottomrule
\end{tabular}
\end{table}

\subsection{Comparison under common supervision}

The main tables compare LAR with both fixed likelihood scores and supervised detectors. To separate the contribution of supervised fitting from that of shorter-context information, Tables~\ref{tab:supp_supervised_wiki} and~\ref{tab:supp_supervised_ai} compare models trained with the same labels and fold assignments. The supervised-summary model fits ridge logistic regression to 12 full-context features: the mean, standard deviation, empirical 10th percentile, and lower-tail mean of each of the raw log probability, vocabulary-standardized log probability, and BOS-relative contrast. The full-context LAR fits use the same basis and regression specification as the complete-array models but retain only cells with $u=1$.

\begin{table}[H]
\centering
\caption{WikiMIA AUC under common supervision. All rows use the same labels and five-fold assignments.}
\label{tab:supp_supervised_wiki}
\footnotesize
\setlength{\tabcolsep}{2pt}
\begin{tabular}{lccccc}
\toprule
Method & Pythia-1B & Pythia-6.9B & Pythia-12B & LLaMA-2-7B & Falcon-7B \\
\midrule
Supervised full-context summaries & 0.653 & 0.702 & 0.734 & 0.612 & 0.719 \\
Full-context LAR-1 & 0.802 & 0.825 & 0.838 & 0.844 & 0.883 \\
Complete-array LAR-1 & 0.908 & 0.923 & 0.931 & 0.948 & 0.963 \\
Full-context LAR-2 & 0.893 & 0.908 & 0.916 & 0.934 & 0.942 \\
Complete-array LAR-2 & 0.952 & 0.960 & 0.961 & 0.976 & 0.984 \\
\bottomrule
\end{tabular}
\end{table}

\begin{table}[H]
\centering
\caption{AI-generated text detection AUC under common supervision. All rows use the same labels and pair-grouped five-fold assignments.}
\label{tab:supp_supervised_ai}
\footnotesize
\setlength{\tabcolsep}{2pt}
\begin{tabular}{lccccc}
\toprule
Method & DeepSeek-7B & OLMo-7B & Qwen3-4B & Qwen3-8B & Qwen3-14B \\
\midrule
Supervised full-context summaries & 0.792 & 0.818 & 0.869 & 0.875 & 0.851 \\
Full-context LAR-1 & 0.937 & 0.983 & 0.977 & 0.990 & 0.985 \\
Complete-array LAR-1 & 0.988 & 0.995 & 0.996 & 0.995 & 0.996 \\
Full-context LAR-2 & 0.945 & 0.987 & 0.979 & 0.988 & 0.983 \\
Complete-array LAR-2 & 0.988 & 0.994 & 0.994 & 0.993 & 0.994 \\
\bottomrule
\end{tabular}
\end{table}

The supervised summaries improve on many fixed scores in the main paper, confirming that label information itself is useful. The full-context LAR fits improve further, reflecting flexible modeling of the full-context channel distribution. Across all ten scoring models, the complete-array LAR-1 exceeds its full-context counterpart; the same pattern holds for LAR-2 in WikiMIA and in four of the five generated-text settings. These comparisons isolate the gain from shorter-context likelihoods after holding supervision and the regression architecture fixed.

\subsection{Additional text generators}

The generated passages in the main analysis are produced by GPT-4o. We repeat the analysis with Claude-3.5-Haiku and Gemini-2.5-Flash to examine whether the results depend on that generator. For each generator, the data contain the same four domains as the main analysis, with 150 human passages and 150 generated counterparts per domain. We use 20 domain-stratified random splits, assigning each human passage and its generated counterpart to the same split. Each split uses 80\% of the 600 source pairs for training and 20\% for testing. The fixed detection scores and RADAR are evaluated on the corresponding test sets, and the AdaDetectGPT witness for each domain is estimated from the other three domains. StatDetectLLM is not repeated in this expanded analysis because of its substantially greater computational cost.

\begin{table}[H]
\centering
\caption{AI-generated text detection for Claude-3.5-Haiku passages. Entries are mean test AUCs over 20 domain-stratified, source-pair-grouped splits, with standard deviations across splits in parentheses. The largest AUC in each scoring-model column is shown in bold; values tied after rounding are both highlighted. RADAR reports one AUC spanning the five scoring-model columns.}
\label{tab:supp_claude_generator}
\footnotesize
\setlength{\tabcolsep}{3pt}
\renewcommand{\arraystretch}{0.8}
\def\dashunit{\rule[0.55ex]{4pt}{0.3pt}\hspace{2pt}}
\def\dashsegment{\hbox{\dashunit\dashunit\dashunit\dashunit\dashunit\dashunit\dashunit\dashunit\dashunit\dashunit\dashunit\dashunit\dashunit}}
\begin{tabular}{lccccc}
\toprule
Method & DeepSeek-7B & OLMo-7B & Qwen3-4B & Qwen3-8B & Qwen3-14B \\
\midrule
\multicolumn{6}{l}{\textit{Scoring-model detection scores}} \\
Likelihood & 0.700 (0.018) & 0.574 (0.015) & 0.723 (0.019) & 0.732 (0.020) & 0.732 (0.020) \\
Entropy & 0.609 (0.019) & 0.488 (0.012) & 0.597 (0.015) & 0.636 (0.016) & 0.628 (0.015) \\
Rank & 0.642 (0.014) & 0.587 (0.017) & 0.667 (0.013) & 0.667 (0.014) & 0.671 (0.020) \\
Log-rank & 0.692 (0.018) & 0.578 (0.015) & 0.708 (0.017) & 0.724 (0.018) & 0.724 (0.019) \\
LRR & 0.619 (0.014) & 0.584 (0.020) & 0.581 (0.018) & 0.640 (0.016) & 0.641 (0.016) \\
Fast-DetectGPT & 0.677 (0.013) & 0.680 (0.022) & 0.630 (0.017) & 0.599 (0.023) & 0.634 (0.015) \\
\midrule
\multicolumn{6}{l}{\textit{Supervised detectors}} \\
AdaDetectGPT & 0.665 (0.013) & 0.677 (0.020) & 0.634 (0.018) & 0.601 (0.023) & 0.638 (0.016) \\
RADAR & \multicolumn{2}{c}{\strut\dashsegment} & 0.803 (0.013) & \multicolumn{2}{c}{\strut\dashsegment} \\
\midrule
\multicolumn{6}{l}{\textit{Likelihood-array regression}} \\
LAR-1 & 0.978 (0.007) & 0.994 (0.004) & \textbf{0.987} (0.004) & \textbf{0.992} (0.005) & \textbf{0.989} (0.007) \\
LAR-2 & \textbf{0.981} (0.006) & \textbf{0.996} (0.003) & \textbf{0.987} (0.006) & 0.991 (0.005) & \textbf{0.989} (0.008) \\
\bottomrule
\end{tabular}
\end{table}

\begin{table}[H]
\centering
\caption{AI-generated text detection for Gemini-2.5-Flash passages. Entries are mean test AUCs over 20 domain-stratified, source-pair-grouped splits, with standard deviations across splits in parentheses. The largest AUC in each scoring-model column is shown in bold. RADAR reports one AUC spanning the five scoring-model columns.}
\label{tab:supp_gemini_generator}
\footnotesize
\setlength{\tabcolsep}{3pt}
\renewcommand{\arraystretch}{0.8}
\def\dashunit{\rule[0.55ex]{4pt}{0.3pt}\hspace{2pt}}
\def\dashsegment{\hbox{\dashunit\dashunit\dashunit\dashunit\dashunit\dashunit\dashunit\dashunit\dashunit\dashunit\dashunit\dashunit\dashunit}}
\begin{tabular}{lccccc}
\toprule
Method & DeepSeek-7B & OLMo-7B & Qwen3-4B & Qwen3-8B & Qwen3-14B \\
\midrule
\multicolumn{6}{l}{\textit{Scoring-model detection scores}} \\
Likelihood & 0.623 (0.019) & 0.535 (0.020) & 0.657 (0.018) & 0.678 (0.020) & 0.695 (0.021) \\
Entropy & 0.602 (0.021) & 0.487 (0.019) & 0.599 (0.016) & 0.622 (0.017) & 0.630 (0.017) \\
Rank & 0.577 (0.008) & 0.535 (0.011) & 0.579 (0.011) & 0.584 (0.011) & 0.580 (0.011) \\
Log-rank & 0.618 (0.019) & 0.532 (0.019) & 0.656 (0.017) & 0.678 (0.019) & 0.690 (0.020) \\
LRR & 0.565 (0.013) & 0.516 (0.018) & 0.575 (0.012) & 0.608 (0.013) & 0.607 (0.017) \\
Fast-DetectGPT & 0.568 (0.021) & 0.617 (0.019) & 0.575 (0.010) & 0.583 (0.013) & 0.598 (0.015) \\
\midrule
\multicolumn{6}{l}{\textit{Supervised detectors}} \\
AdaDetectGPT & 0.546 (0.020) & 0.620 (0.023) & 0.574 (0.010) & 0.579 (0.015) & 0.597 (0.016) \\
RADAR & \multicolumn{2}{c}{\strut\dashsegment} & 0.612 (0.015) & \multicolumn{2}{c}{\strut\dashsegment} \\
\midrule
\multicolumn{6}{l}{\textit{Likelihood-array regression}} \\
LAR-1 & 0.979 (0.007) & 0.990 (0.005) & 0.979 (0.007) & 0.989 (0.004) & \textbf{0.990} (0.004) \\
LAR-2 & \textbf{0.983} (0.006) & \textbf{0.992} (0.003) & \textbf{0.983} (0.006) & \textbf{0.990} (0.004) & 0.987 (0.009) \\
\bottomrule
\end{tabular}
\end{table}

Across both generators, LAR-1 and LAR-2 record higher mean test AUCs than the comparison procedures for every scoring model. Their AUCs range from \(0.978\) to \(0.996\) for Claude-3.5-Haiku and from \(0.979\) to \(0.992\) for Gemini-2.5-Flash. The difference between the two LAR variants remains small and changes direction across scoring models. These results extend the generated-text findings beyond GPT-4o and reinforce the conclusion that first-order likelihood-array information accounts for most of the separation in this application.

\subsection{Values underlying the restricted and sample-size analyses}

Table~\ref{tab:supp_mechanism_checks} gives the numerical values underlying the restricted analyses in the main paper. The channel restrictions are cumulative, and the array-support restrictions are reported separately for LAR-1 and LAR-2.

\begin{figure}[H]
\centering
\includegraphics[width=0.98\textwidth]{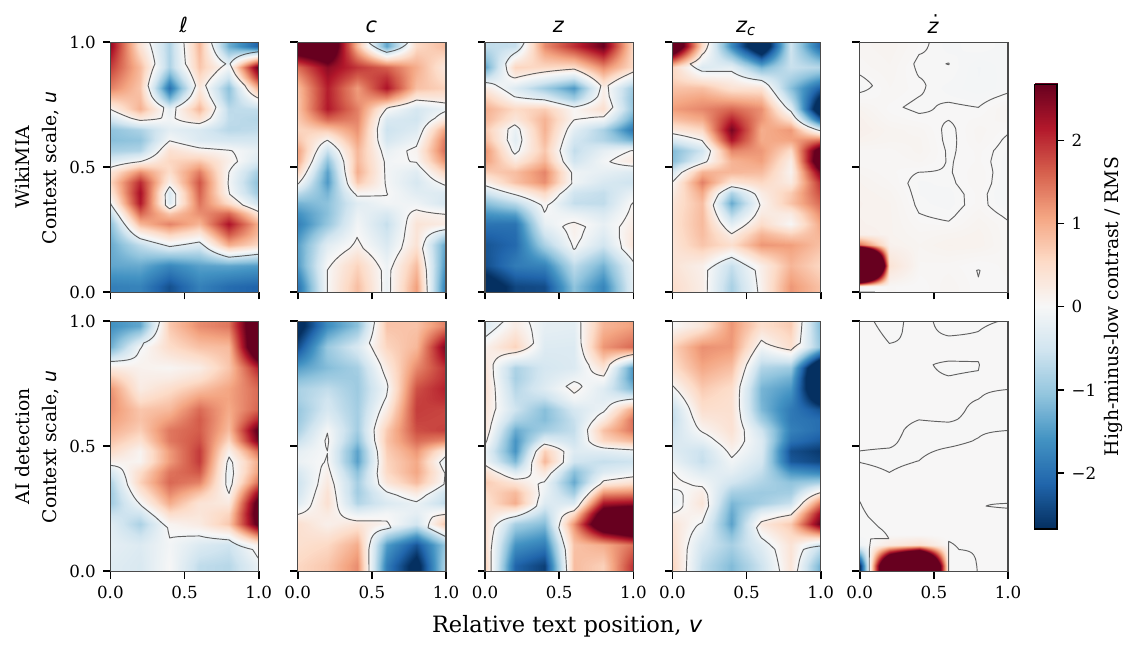}
\caption{Estimated coefficient contrasts for all five likelihood-derived channels used in the paper. Each panel compares the fitted first-order coefficient function at the 80th and 20th percentiles of the indicated channel, holding the remaining channels at their medians.  The fold-averaged surface is divided by its root mean square, so the panels show the spatial pattern and sign of each contrast rather than their relative magnitudes.  The rows correspond to WikiMIA with Pythia-1B and AI-generated text detection with DeepSeek-7B; gray contours mark zero.}
\label{fig:supp_functional_contrasts}
\end{figure}

\begin{table}[H]
\centering
\caption{Restricted analyses in two representative settings. Entries are out-of-fold AUCs. The likelihood-channel rows report LAR-2, and the array-support rows report both LAR-1 and LAR-2.}
\label{tab:supp_mechanism_checks}
\small
\begin{tabular}{llcc}
\toprule
Component & Restriction & WikiMIA & AI detection \\
\midrule
Likelihood-derived channels & Raw log probability & 0.830 & 0.975 \\
 & + BOS-relative probability & 0.906 & 0.984 \\
 & + Vocabulary-standardized channels & 0.953 & 0.989 \\
 & + Local context response & 0.952 & 0.988 \\
\addlinespace
Array support: LAR-1 & Full-context cells & 0.802 & 0.937 \\
 & Shorter-context cells & 0.907 & 0.988 \\
 & Complete array & 0.908 & 0.988 \\
\addlinespace
Array support: LAR-2 & Full-context cells & 0.893 & 0.945 \\
 & Shorter-context cells & 0.952 & 0.989 \\
 & Complete array & 0.952 & 0.988 \\
\bottomrule
\end{tabular}
\end{table}

Table~\ref{tab:supp_sample_size} reports the summaries underlying the labeled-sample-size curves.  The full-sample value is obtained from one fit; the other columns summarize 20 repeated training subsamples evaluated on a common test set.

\begin{table}[H]
\centering
\caption{Median test AUC over 20 repeated labeled subsamples. Parentheses give the 10th and 90th percentiles; the Full column reports the single full-sample fit.}
\label{tab:supp_sample_size}
\footnotesize
\setlength{\tabcolsep}{3pt}
\resizebox{\textwidth}{!}{%
\begin{tabular}{llccccc}
\toprule
Task & Method & 100 & 200 & 400 & 800 & Full \\
\midrule
WikiMIA & Full-context only & 0.638 (0.619, 0.664) & 0.682 (0.661, 0.725) & 0.708 (0.689, 0.731) & 0.749 (0.735, 0.760) & 0.777 \\
 & LAR-1 & 0.684 (0.658, 0.719) & 0.753 (0.725, 0.785) & 0.812 (0.788, 0.841) & 0.875 (0.863, 0.884) & 0.913 \\
 & LAR-2 & 0.714 (0.674, 0.743) & 0.796 (0.760, 0.813) & 0.857 (0.842, 0.881) & 0.914 (0.909, 0.923) & 0.943 \\
\addlinespace
AI detection & Full-context only & 0.860 (0.846, 0.883) & 0.890 (0.874, 0.907) & 0.897 (0.884, 0.912) & 0.902 (0.898, 0.911) & 0.908 \\
 & LAR-1 & 0.956 (0.945, 0.967) & 0.973 (0.967, 0.979) & 0.978 (0.976, 0.983) & 0.982 (0.979, 0.983) & 0.982 \\
 & LAR-2 & 0.955 (0.945, 0.963) & 0.969 (0.964, 0.975) & 0.977 (0.973, 0.982) & 0.980 (0.977, 0.982) & 0.981 \\
\bottomrule
\end{tabular}}
\end{table}

\subsection{Sensitivity analysis}
\label{sec:supp_numerical_sensitivity}

We examine the sensitivity of LAR-2 to its numerical specification. Panels A--E of Figure~\ref{fig:supp_sieve_stability} show the results of varying one component at a time for WikiMIA with Pythia-1B, while Panel F examines variation across random-feature seeds in both applications. The vertical dashed lines indicate the specification used in the main comparisons.

\begin{figure}[H]
\centering
\begingroup
\small
\renewcommand{\baselinestretch}{1}\selectfont
\includegraphics[width=0.90\textwidth]{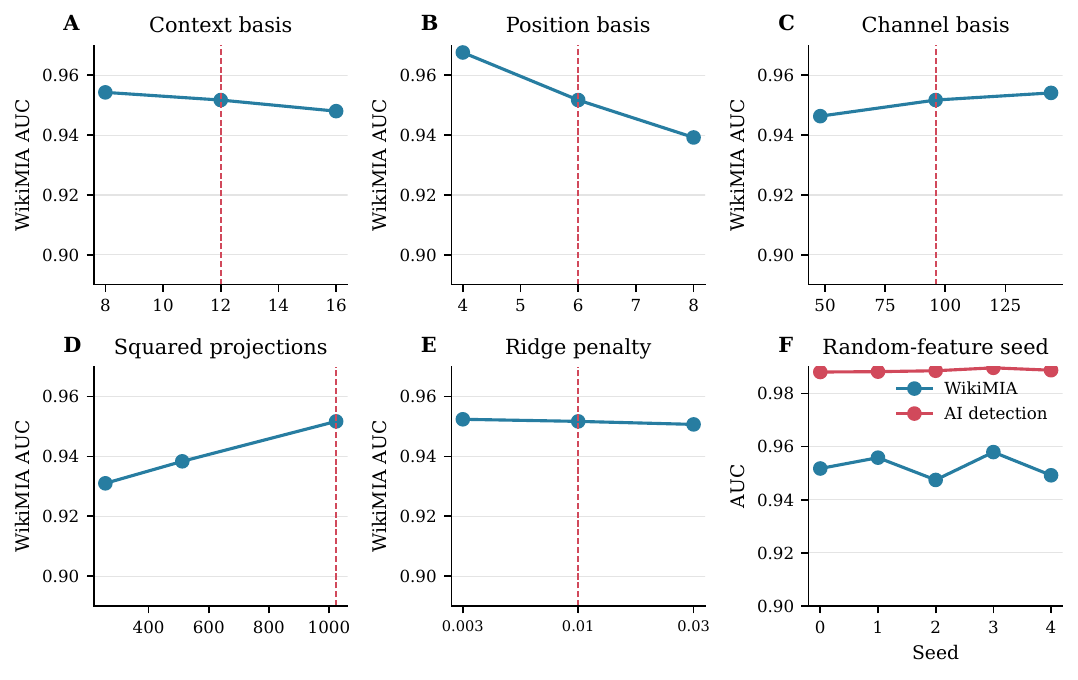}
\caption{Sensitivity of LAR-2 to its numerical specification. Panels A--E report WikiMIA AUC as the context-basis dimension, position-basis dimension, channel-basis dimension, number of squared projections, and ridge penalty are varied separately. Dashed lines indicate the specification used in the main comparisons. Panel F reports AUC across random-feature seeds for WikiMIA and AI-generated text detection.}
\label{fig:supp_sieve_stability}
\endgroup
\end{figure}

AUC changes little over the displayed ranges of the ridge penalty and channel-basis dimension. The position-basis dimension and number of squared projections produce somewhat greater variation in WikiMIA AUC, but LAR-2 remains above the strongest comparison procedure over all displayed settings. Across five random-feature seeds, AUC ranges from \(0.947\) to \(0.958\) for WikiMIA and from \(0.988\) to \(0.990\) for AI-generated text detection.

\subsection{Simulation}

The simulation isolates the first- and second-order information represented by the likelihood-array model. Each observation is a \(16\times24\) aligned grid with five Gaussian channels. Under the baseline model, each channel follows an autoregressive process over context scale with correlation 0.45. Class differences are introduced over the region \(u\in[0.25,0.75]\) and \(v\in[0.20,0.80]\). The marginal-mean and marginal-scale settings shift the mean or variance of the first channel. The dependence-only setting changes the sign and magnitude of the context-scale correlation without changing the one-point Gaussian margins. The variable-length setting combines a marginal mean shift with a random number of retained target positions. Each of 30 repetitions contains 400 training and 200 test observations. The oracle score uses the known data-generating feature.

\begin{figure}[t]
\centering 
\includegraphics[width=0.90\textwidth]{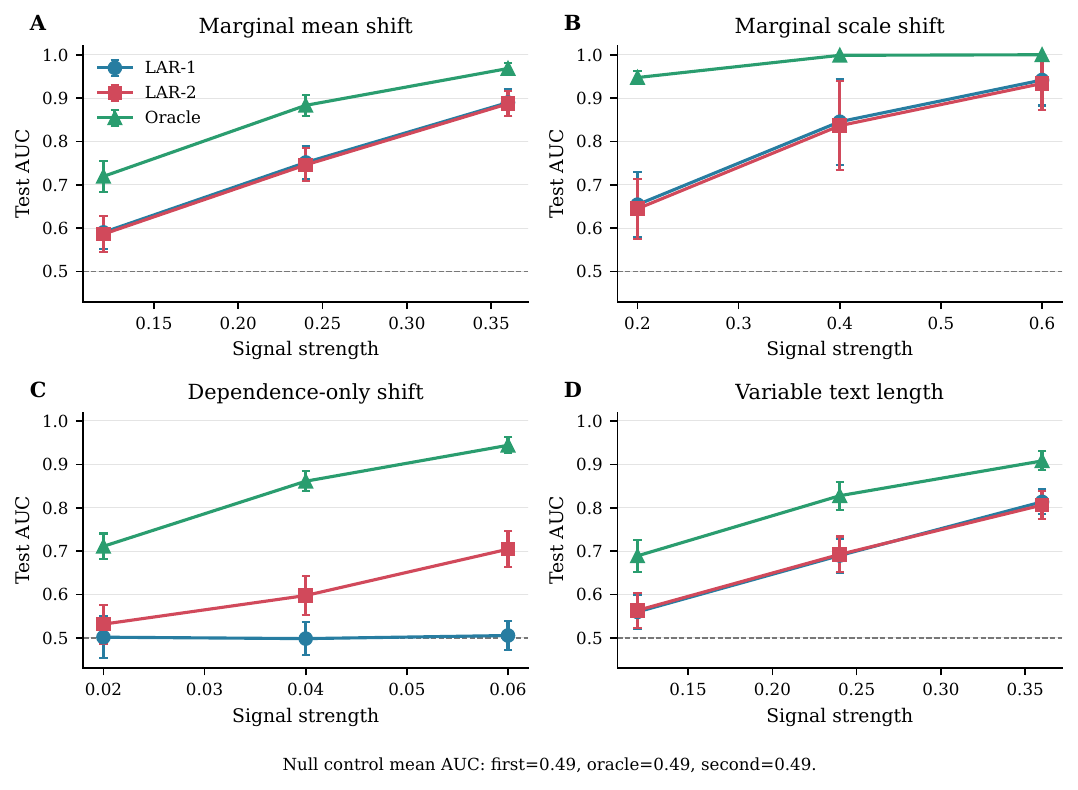}
\caption{Recovery of distinct likelihood-array signals. Points show mean test AUC and bars show one Monte Carlo standard deviation over 30 repetitions for marginal mean shifts (A), marginal scale shifts (B), dependence-only shifts (C), and variable text lengths (D). Under the dependence-only alternative, the classes have the same first-order marginal distributions; LAR-1 remains near chance while LAR-2 gains ranking accuracy as the association difference increases.}
\label{fig:supp_simulation}
\end{figure}

Panels A, B, and D show that LAR-1 recovers changes in the distribution of the aligned channel values and remains effective when the number of target positions varies. Panel C isolates the role of the quadratic term: LAR-1 remains near 0.5, whereas the LAR-2 AUC increases from approximately 0.53 to 0.70 over the displayed dependence strengths. The null control has mean AUC 0.49 for both fitted models.

\subsection{Transfer across length groups and text domains}

The transfer analyses keep the representation and regression specification fixed while separating training and test groups.  For WikiMIA, the across-length fits train on one pair of benchmark length groups and test on the complementary pair.  For generated-text detection, each fit leaves one domain out of training.  Figure~\ref{fig:supp_transfer} summarizes the comparisons, and Table~\ref{tab:supp_transfer} gives the numerical values.

\begin{figure}[t]
\centering
\includegraphics[width=0.98\textwidth]{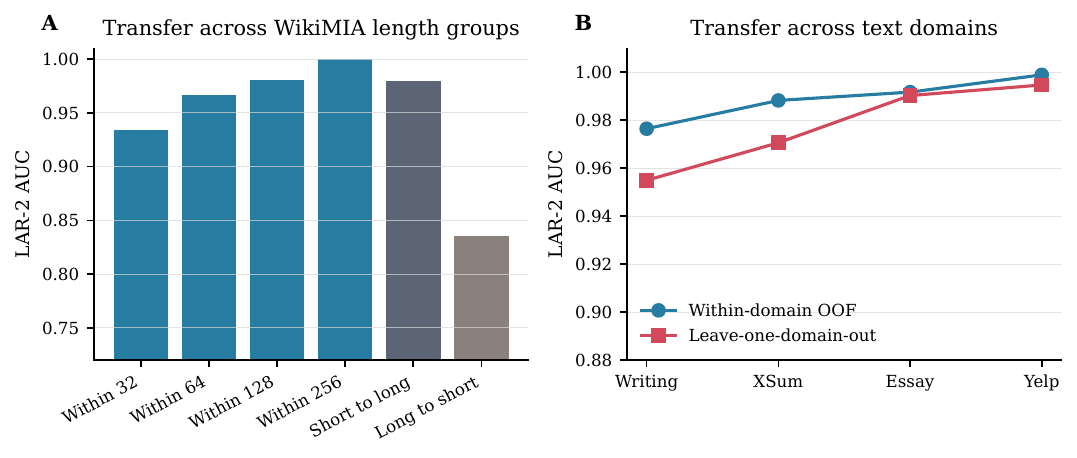}
\caption{Transfer across text length and domain. Panel A reports LAR-2 AUC within WikiMIA length groups and when training and testing on disjoint pairs of length groups. Panel B compares pair-grouped within-domain out-of-fold AUC with leave-one-domain-out AUC for AI-generated text detection.}
\label{fig:supp_transfer}
\end{figure}

\begin{table}[H]
\centering
\caption{Within-group and out-of-group AUCs for LAR-2. The WikiMIA cross-length rows train on the complementary pair of length groups.}
\label{tab:supp_transfer}
\small
\begin{tabular}{lllrr}
\toprule
Task & Analysis & Test group & Test size & AUC \\
\midrule
WikiMIA & Within length group & 32 & 776 & 0.934 \\
WikiMIA & Within length group & 64 & 542 & 0.967 \\
WikiMIA & Within length group & 128 & 250 & 0.981 \\
WikiMIA & Within length group & 256 & 82 & 1.000 \\
WikiMIA & Across length groups & 128+256 & 332 & 0.979 \\
WikiMIA & Across length groups & 32+64 & 1318 & 0.836 \\
AI detection & Within domain & Writing & 300 & 0.976 \\
AI detection & Within domain & XSum & 300 & 0.988 \\
AI detection & Within domain & Essay & 300 & 0.992 \\
AI detection & Within domain & Yelp & 300 & 0.999 \\
AI detection & Leave one domain out & Writing & 300 & 0.955 \\
AI detection & Leave one domain out & XSum & 300 & 0.971 \\
AI detection & Leave one domain out & Essay & 300 & 0.990 \\
AI detection & Leave one domain out & Yelp & 300 & 0.995 \\
\bottomrule
\end{tabular}
\end{table}

Leave-one-domain-out AUC differs from the corresponding within-domain value by 0.002--0.021. Cross-length AUC is 0.979 when the shorter groups are used for training and the longer groups for testing, and 0.836 in the reverse direction. The within-group AUC of 1.000 for the 256-token group is descriptive because that group contains only 82 texts.

\end{document}